\documentclass{article}
\usepackage{hyphenat}
\PassOptionsToPackage{numbers, compress}{natbib}
\usepackage[preprint]{neurips_2026}

\usepackage[utf8]{inputenc} 
\usepackage[T1]{fontenc}    
\usepackage{hyperref}       
\usepackage{url}            
\usepackage{booktabs}       
\usepackage{amsfonts}       
\usepackage{nicefrac}       
\usepackage{microtype}      
\usepackage{xcolor}         
\usepackage{multirow}
\usepackage{caption}
\usepackage{subcaption}
\usepackage{graphicx} 
\usepackage{multirow}
\usepackage{amsmath}
\usepackage{amsfonts}
\usepackage{amssymb}
\usepackage{amsthm}
\usepackage{bm}
\usepackage{booktabs}
\usepackage{longtable}
\usepackage{pdflscape}
\usepackage{array}
\usepackage{tikz}
\usetikzlibrary{positioning,fit}

\newtheorem{theorem}{Theorem}[section]
\newtheorem{proposition}[theorem]{Proposition}
\newtheorem{lemma}[theorem]{Lemma}

\theoremstyle{definition}
\newtheorem{definition}[theorem]{Definition}
\newtheorem{assumption}[theorem]{Assumption}

\theoremstyle{remark}

\title{Multiscale Euclidean Network Trajectories: \\Second-Moment Geometry, Attribution, and Change Points}

\author{%
  Haruka Ezoe\\
  Graduate School of Information Science and Technology\\
  The University of Tokyo\\
  \and
  Ryohei Hisano \\
  Graduate School of Information Science and Technology\\
  The University of Tokyo\\
  The Canon Institute for Global Studies
}

\begin{document}

\maketitle

\begin{abstract}
Understanding a dynamic network requires more than locating when it changes: we also want to know the magnitude and structural direction of change and which nodes contribute. Unfolded spectral embedding places network snapshots in a joint representation, but the underlying rectangular latent factorization is invariant under general linear transformations. Such transformations preserve edge probabilities while potentially distorting Euclidean second moments. We introduce \emph{Multiscale Euclidean Network Trajectories} (MENT), whose central estimands are pairwise second-moment operators of the dynamic latent positions. Under an isotropic normalization of the shared anchor latent positions, these operators are identifiable up to simultaneous orthogonal conjugation. They yield a global trace variation distance and, through an aggregated operator, shared orthogonal directions for mode-wise distances. The squared global distance decomposes exactly across these modes. Classical multidimensional scaling turns the distances into global and mode-wise trajectories of time. A modified unfolded spectral estimator consistently recovers the operators and induced trajectories, and we establish bounds connecting node attributions to trajectory displacements and trajectory estimation to change point localization. Synthetic experiments validate geometric recovery and show improved detection of mode-specific changes against baselines. Analyses of two real networks illustrate interpretable mode- and node-level descriptions of temporal evolution.
\end{abstract}

\section{Introduction}

Dynamic networks arise in social, communication, biological, and legal systems, where interactions evolve over time. A central goal in their analysis is not merely to detect that a network has changed, but to determine \emph{how} it changed and \emph{which nodes} contributed to the change. Addressing these questions requires a geometry on time: given a sequence of network snapshots, we seek distances between time points that reflect latent structural variation, admit interpretable decompositions, and are statistically identifiable from the observed graphs.

Spectral embeddings offer a natural starting point because they represent network structure through node-level latent positions. Adjacency spectral embedding represents individual snapshots, whereas unfolded adjacency spectral embedding (UASE) constructs a joint representation with coordinates shared across time~\citep{gallagher2021spectral}. Even in a joint representation, however, the geometry of the underlying rectangular latent factors is not automatically identifiable. In multilayer and unfolded models, the rectangular factorization of the population unfolded matrix is invariant under the transformation
\[
(\mathbf X,\mathbf Y)\mapsto(\mathbf X\mathbf G^{-\top},\mathbf Y\mathbf G),
\qquad \mathbf G\in\mathrm{GL}(d),
\]
so the latent factors are generally identifiable only up to a nonsingular linear transformation~\citep{jones2020multilayer}. This ambiguity is harmless for edge probabilities, which depend on bilinear products, but it is consequential for geometry: a general linear transformation can distort lengths, directions, and second moments. Thus, geometric comparisons made directly from an arbitrary unfolded factorization need not describe an identifiable property of the dynamic network.


A complementary literature constructs trajectories directly from dissimilarities between entire network snapshots. Classical multidimensional scaling (CMDS) converts such dissimilarities into Euclidean representations of time~\citep{borg2005modern,athreya2025euclidean,chen2024euclidean,zheng2024dynamic}, while methods such as DeltaCon and NetSimile provide global measures of network similarity~\citep{koutra2016deltacon,berlingerio2012netsimile}. These approaches summarize \emph{when} snapshots differ, but global dissimilarities alone reveal less about the structure responsible for those differences. In particular, Euclidean Mirror uses maximum directional variation, selecting a dominant direction separately for each time pair, so it does not provide persistent structural modes for the full sequence. Moreover, the resulting maximum directional variation distance matrix over a finite set of time points need not be exactly realizable in any Euclidean space, so CMDS may represent it only approximately. Appendix~\ref{app:mv-non-euclidean} gives a counterexample. Likewise, change point methods for dynamic graphs are primarily designed to detect or localize changes~\citep{huang2020laplacian,bhattacharjee2020change,zhao2019change,enikeeva2025change,wang2026change}. They do not generally provide a common geometry that simultaneously decomposes variation into persistent modes and attributes it to individual nodes.


We introduce \emph{Multiscale Euclidean Network Trajectories} (MENT) to bridge these perspectives. The key distinction from existing UASE theory lies at the population level. Convergence of a sample embedding to a population UASE target after a shared rotation does not remove the \(\mathrm{GL}(d)\) ambiguity of the rectangular latent factorization and therefore does not by itself identify Euclidean second moments. MENT resolves this problem by imposing an isotropic normalization on the shared anchor latent positions and taking as its central estimand the collection of pairwise second-moment operators
\[
\mathbf M(t,s) =
\mathbb E\!\left[(\phi(t)-\phi(s))(\phi(t)-\phi(s))^\top\right].
\]
The expectation is over the paired latent positions of the same node at times \(t\) and \(s\). Thus, \(\mathbf M(t,s)\) depends on their joint temporal behavior, not only on the two snapshot-wise marginal distributions. Among equivalent factorizations, this canonicalization reduces the admissible ambiguity from \(\mathrm{GL}(d)\) to a single shared orthogonal transformation. The operators \(\mathbf M(t,s)\) are consequently identified up to simultaneous orthogonal conjugation, which preserves their trace, spectrum, and induced geometry. This yields a trace variation distance measuring total structural change and mode-wise variation distances defined using the eigenvectors of an aggregated second-moment operator. The aggregated operator provides one common orthogonal basis for the dynamic sequence, ordering modes by their cumulative contribution over a chosen collection of time pairs. Moreover, the squared trace distance decomposes exactly into the squared mode-wise distances.

Applying CMDS to these distances produces Euclidean trajectories of time at both global and mode-wise resolutions. This construction connects three levels of description: second-moment operators encode latent structural differences, their trace and directional quadratic forms define temporal distances, and CMDS represents those distances as trajectories. Because the construction is based on a shared canonical geometry, the same objects also support node-level attribution and mode-specific change point analysis. The trace trajectory summarizes global structural evolution, while the mode-wise trajectories distinguish changes occurring along different persistent directions. Node contributions connect local changes in the estimated latent positions to the corresponding trajectory displacement.

For estimation, we use a modified form of UASE designed for the canonical geometry. Rather than treating this embedding as the final estimand, we use it to recover the pairwise second-moment operators up to one shared orthogonal transformation. We then propagate this recovery through the induced squared distances and CMDS construction. This operator-to-distance-to-trajectory argument establishes consistency of the global and mode-wise temporal trajectories. Separate results relate estimated node attributions to trajectory displacements and control change point localization under additional assumptions.

Our contributions are threefold. First, we isolate the \(\mathrm{GL}(d)\) ambiguity as the obstruction to identifiable temporal geometry in unfolded spectral models and show how anchor isotropy yields a canonical second-moment geometry. Second, we introduce trace and shared mode-wise temporal geometries and show how they generate multiscale Euclidean trajectories with interpretable decompositions. Third, we construct an estimator of the second-moment operators from observed networks, establish recovery of the induced trajectories, and connect this theory to node attribution and change point localization. Experiments on synthetic and real dynamic networks illustrate the resulting global, mode-wise, and node-level descriptions of temporal change.

\section{Identifiable Second-Moment Geometry}
\label{sec:geometry}

\subsection{Setup and generative model}

We observe a dynamic network over time \(\mathcal T=\{1,\dots,T\}\) with adjacency matrices \(\mathbf A(t)\in\mathbb R^{n\times n}\) at each $t\in \mathcal T$, where \(n\) denotes the number of nodes. Assuming that an underlying structure governs the dynamic network, our goal is to (i) construct a representation of the underlying temporal evolution and (ii) estimate it from the observed snapshots. We consider a model in line with multilayer random dot product graph models~\citep{jones2020multilayer,gallagher2021spectral}, where edge probabilities are determined by inner products of node-specific latent positions.

\textbf{Generative model:} 
At the population level, the underlying structure is characterized by \(L^2\) random vectors \(\chi:\Omega\to\mathcal X\) and \(\phi(t):\Omega\to\mathcal Y(t)\), \(t\in\mathcal T\), where \(\chi\) captures the time-invariant structure, while \(\phi(t)\) captures the time-varying structure at time \(t\). Here, \(\mathcal X,\mathcal Y(t)\subseteq\mathbb R^d\) are bounded sets such that \(\|x\|_2,\|y(t)\|_2\le\beta\) and \(x^\top y(t)\in[0,1]\) for all \(x\in\mathcal X\) and \(y(t)\in\mathcal Y(t)\), \(t\in\mathcal T\), where \(\beta>0\) is a constant and \(d\) denotes the latent space dimension. We further assume that the second-moment matrices \(\mathbb E[\chi\chi^\top]\) and \(\mathbb E[\phi(t)\phi(t)^\top]\) are invertible for all \(t\in\mathcal T\).
At the finite sample level, the anchor and dynamic latent position matrices \(\mathbf X,\mathbf Y(t)\in\mathbb R^{n\times d}\) have rows corresponding to the \(n\) nodes, which are i.i.d. copies of \(\chi\) and \(\phi(t)\), respectively. Here the rows are coupled across time: for each node \(i\),
\((\mathbf X_{i:},\{\mathbf Y_{i:}(t)\}_{t\in\mathcal T})\)
is one independent copy of \((\chi,\{\phi(t)\}_{t\in\mathcal T})\), and the same node-specific draw is used across all snapshots.
We also assume $\phi(t)=\mathbf{B}(t)^\top\chi \quad \text{a.s.}$ for a symmetric matrix $\mathbf{B}(t)=\mathbf{B}(t)^\top$ to model undirected network. Using this, edges are generated independently via
\[
\mathbf P(t)=\mathbf X\mathbf Y(t)^\top \in\mathbb R^{n\times n},
\qquad
\mathbf A_{ij}(t)\sim \mathrm{Bernoulli}\bigl(\mathbf P_{ij}(t)\bigr)\quad (i<j).
\]
At the latent level, the unfolded probability matrix factors are as follows:
\[
\mathbf P=[\mathbf P(1)\mid\cdots\mid\mathbf P(T)]=\mathbf X\mathbf Y^\top\in\mathbb R^{n\times nT},
\]
where \(\mathbf Y\in\mathbb R^{nT\times d}\) stacks the matrices \(\mathbf Y(t)\). This completes the generation process.
The constants $\beta$ and $d$ are independent of $n$, while $T=T_n$ may depend on $n$.

\subsection{Trace and mode-wise distances}
\label{subsec:multiscale-decomposition}

Let $\bm x, \bm y\in L^2(\Omega;\mathbb R^d)$ be random vectors. The population second-moment matrix of the difference between \(\bm{x}\) and \(\bm{y}\) is defined as
\[
\mathbf M(\bm x,\bm y) := \mathbb E\!\left[(\bm x-\bm y)(\bm x-\bm y)^\top\right].
\]

\begin{definition}[Trace and mode-wise variation distances]
	\label{def:tv-k-paper}
    The trace variation distance between $\bm{x}$ and $\bm{y}$ is defined as
	\[
	d_{\mathrm{tr}}(\bm{x},\bm{y}) := \left(\mathrm{tr} \mathbf M(\bm x,\bm y)\right)^{1/2} = \|\bm x-\bm y\|_{L^2(\Omega;\mathbb R^d)}.
	\]
	Given an orthonormal basis \(\{\bm u_k\}_{k=1}^d\), define the $k$th mode-wise variation distance between $\bm x$ and $\bm y$ as
	\[
	d_k(\bm x,\bm y) := \left(\bm u_k^\top \mathbf M(\bm x,\bm y)\bm u_k\right)^{1/2} = \|(\bm x-\bm y)^\top \bm u_k\|_{L^2(\Omega;\mathbb R)}.
	\]
\end{definition}

\textbf{Remark.} By construction, $d_k(\cdot,\cdot)$ defines a Hilbertian pseudometric on $L^2(\Omega;\mathbb R^d)$.
Specifically,
\[
\begin{aligned}
	&d_k(\bm{x},\bm{y}) = 0 \quad\Leftrightarrow \quad (\bm x-\bm y)^\top \bm u_k = 0 \ \text{almost surely}.
\end{aligned}
\]

For notational simplicity, throughout the remainder of the paper, we write
\[
\mathbf{M}(t,s) = \mathbf{M}(\phi(t),\phi(s)),
\qquad
d_\rho(t,s) = d_\rho(\phi(t),\phi(s)),
\qquad
\rho\in\{\mathrm{tr},1,\ldots,d\}.
\]

\subsection{Canonicalization of the geometry}

To obtain a geometrically meaningful and statistically identifiable framework, we study a canonical parameterization determined by the second-moment of the anchor random vector. This is set using the following assumption.

\begin{assumption}[Anchor isotropy]
	\label{assump:isotropy}
	The anchor random vector satisfies $\mathbb E[\chi\chi^\top]=\mathbf I_d$.
\end{assumption}

The key geometric issue in multilayer and unfolded spectral constructions is that the underlying latent factorization is rectangular rather than square. Unlike square random dot product graph factorizations of a single snapshot, which are identifiable up to an orthogonal transformation, the rectangular factorization admits the larger \(\mathrm{GL}(d)\) ambiguity. Specifically, for any \(\mathbf G\in\mathrm{GL}(d)\),
\[
\mathbf P=\mathbf X\mathbf Y^\top = (\mathbf{XG}^{-\top})(\mathbf{YG})^\top.
\]
This general linear transformation leaves the edge probabilities unchanged. Consequently, the population latent object is only identifiable up to a general linear transformation~\citep{jones2020multilayer}. Under the canonical parameterization, this ambiguity is reduced to a shared orthogonal transformation. We summarize this fact in the following proposition.


\begin{proposition}[$\mathrm{GL}(d)$ and $\mathbb{O}(d)$ identifiability]
\label{prop:identifiability}
We say that two population objects
$(\chi,\{\phi(t)\}_{t\in\mathcal T})$
and
$(\tilde\chi,\{\tilde\phi(t)\}_{t\in\mathcal T})$
are equivalent if they induce the same
pairwise inner products between independent population draws. Without
Assumption~\ref{assump:isotropy}, this equivalence holds if and only if there
exists $\mathbf G\in\mathrm{GL}(d)$ such that
$\tilde\chi=\mathbf G^{-1}\chi$ and
$\tilde\phi(t)=\mathbf G^\top\phi(t)$ a.s. for all $t\in\mathcal T$.
Under Assumption~\ref{assump:isotropy} for both $\chi$ and $\tilde\chi$, the
equivalence holds if and only if $\mathbf G\in\mathbb O(d)$.
\end{proposition}

Proposition~\ref{prop:identifiability} characterizes equivalence in both directions (see Appendix~\ref{app:identifiability-invariance} for the formal definition and proof). The sufficient direction shows that, for any \(\mathbf G\in\mathrm{GL}(d)\), the shared transformation \(\tilde\chi=\mathbf G^{-1}\chi\) and \(\tilde\phi(t)=\mathbf G^\top\phi(t)\) produces an equivalent population object. Under this transformation, the population second-moment matrix satisfies \(\mathbf M(t,s)\mapsto\mathbf G^\top\mathbf M(t,s)\mathbf G\). A general linear congruence need not preserve the geometric quantities derived from \(\mathbf M(t,s)\). In particular, distances and temporal trajectories may change. Thus, equivalent latent representations may induce different temporal geometries. The necessary direction shows conversely that every equivalent population object must arise from such a shared transformation. When both anchors satisfy Assumption~\ref{assump:isotropy}, this direction further forces \(\mathbf G\in\mathbb O(d)\), so the remaining ambiguity is exactly a shared orthogonal transformation. For such a transformation, \(\mathbf G^{-1}=\mathbf G^\top\), and the second-moment matrices are related by orthogonal conjugation, which preserves their traces and spectra. The resulting trace-based and shared mode-wise temporal geometries are therefore invariant to the remaining non-identifiability (see Proposition~\ref{prop:geometry-invariance}).

\subsection{Shared modes}

Trace variation measures total temporal change but not the directions responsible for it. The following identity decomposes that change along any orthonormal basis. Choosing the basis once for the entire sequence makes the resulting modes comparable across time pairs.

\begin{theorem}[Orthogonal decomposition of trace variation]
	\label{thm:tv-decomposition-paper}
	For the distances defined in Definition~\ref{def:tv-k-paper}, the following holds.
	\[
	d_{\mathrm{tr}}(t,s)^2=\sum_{k=1}^d d_k(t,s)^2.
	\]
\end{theorem}

We obtain this shared basis from an aggregated second-moment operator. Let \(\mathcal S\subseteq\mathcal T\times\mathcal T\) be a user-specified collection of time pairs, and define the aggregated second-moment matrix
\[
\mathcal M:=\sum_{(t,s)\in\mathcal S}\mathbf M(t,s) \in \mathbb R^{d\times d}.
\]
We take the orthonormal eigenvectors of \(\mathcal M\) as the orthonormal basis \(\{\bm u_k\}_{k=1}^d\), with eigenvalues ordered decreasingly. Let \(k_{\mathrm{max}}\in[d]\) denote the number of modes retained for individual analysis, and we require their eigenvalues to be non-degenerate\footnote{For \(k>k_{\mathrm{max}}\), the eigenvectors need not be uniquely determined in the presence of eigenvalue degeneracy, but this freedom is immaterial since these modes are not used for individual analysis.}.

\begin{assumption}[Non-degenerate retained modes]
\label{assump:separated_eigenvalues}
There exists a sequence \(\delta=\delta_n>0\) such that, for every \(k\in[k_{\mathrm{max}}]\),
\[
\lambda_k(\mathcal M)-\lambda_{k+1}(\mathcal M)
\ge \delta_n,
\]
where we define
\(
\lambda_{d+1}(\mathcal M):=-\infty
\).
\end{assumption}

This yields an orthonormal basis that describes the orthogonal structural directions. The following proposition shows that the aggregated operator orders modes by their cumulative contribution to temporal variation over the selected time pairs. In particular, each mode is selected to maximize this cumulative contribution subject to orthogonality.

\begin{proposition}[Spectral characterization of shared modes]
\label{prop:shared_modes}
For each \(k\in[d]\),
\[
\lambda_k(\mathcal M)
=
\sum_{(t,s)\in\mathcal S}
d_k(t,s)^2,
\qquad
\bm u_k
\in
\arg\max_{\substack{
\|\bm u\|=1\\
\bm u\perp\bm u_1,\ldots,\bm u_{k-1}
}}
\sum_{(t,s)\in\mathcal S}
\bm u^\top
\mathbf M(t,s)
\bm u.
\]
\end{proposition}

The proofs of Theorem~\ref{thm:tv-decomposition-paper} and
Proposition~\ref{prop:shared_modes} are given in
Appendix~\ref{app:trace-variation-shared-modes}.

\subsection{Euclidean temporal trajectories}

For each \(\rho\in\{\mathrm{tr},1,\ldots,d\}\), we apply classical multidimensional scaling (CMDS) to the temporal distances \(d_\rho(t,s)\)~\citep{borg2005modern,athreya2025euclidean}. Specifically, define
\[
\mathbf D_\rho^{(2)}
=
\bigl[d_\rho(\phi(t),\phi(s))^2\bigr]_{t,s\in\mathcal T},
\qquad
\mathbf J
=
\mathbf I_T-\frac1T\mathbf 1\mathbf 1^\top,
\qquad
\mathbf E_\rho
=
-\frac12\mathbf J\mathbf D_\rho^{(2)}\mathbf J
\in\mathbb R^{T\times T}.
\]
CMDS yields a $c_\rho$-dimensional embedding
$\psi_\rho(t)\in\mathbb R^{c_\rho}$ such that
\[
\|\psi_\rho(t)-\psi_\rho(s)\|
\approx
d_\rho(\phi(t),\phi(s)),
\qquad
t,s\in\mathcal T.
\]
Since $d_\rho$ is Hilbertian, $\mathbf E_\rho$ is positive semidefinite and satisfies \(\operatorname{rank}(\mathbf E_\rho)\le T-1\). Therefore, we consider $c_\rho\le \operatorname{rank}(\mathbf E_\rho)$, and the distances are exactly preserved when $c_\rho=\operatorname{rank}(\mathbf E_\rho)$.

\textbf{Practical choices.} The set \(\mathcal S\) controls temporal scale: adjacent pairs emphasize local increments, as in our change-point experiments, while all pairs capture period-wide variation. We inspect modes in decreasing \(\lambda_k(\mathcal M)\), omit negligible modes, and treat small eigengaps cautiously. Mild overestimation of \(d\) yields nearly flat noise modes (Appendix~\ref{app:overestimating}). We choose \(c_\rho\) from the CMDS spectrum and use \(c_\rho=1\) for change-point localization.


\section{Estimation and Inference}
\label{sec:estimation-inference}

We turn the identifiable geometry of Section~\ref{sec:geometry} into estimates. Modified UASE provides node embeddings, their differences yield estimates of the pairwise second-moment matrices, and trace and shared mode-wise distances then yield CMDS trajectories. The embeddings are intermediate, whereas the second-moment geometry is the primary estimand.


\subsection{Modified UASE and trajectory estimation}

Following \citet{athreya2025euclidean}, we construct a consistent estimator based on node embeddings derived from the observed networks.
Because the standard UASE scaling targets a population factor whose Euclidean geometry is not identified under the underlying \(\mathrm{GL}(d)\) ambiguity, we modify the scaling to target the canonical dynamic factor. We define the unfolded adjacency matrix using column concatenation and compute a rank-$d$ truncated singular value decomposition:
\[
\mathbf A = [\mathbf A(1)\mid \mathbf A(2)\mid \cdots \mid \mathbf A(T)] \in \mathbb R^{n\times nT}, \qquad \mathbf A \approx \hat{\mathbf U}\hat{\mathbf \Sigma}\hat{\mathbf V}^\top,
\]
where $\hat{\mathbf V}(t)\in\mathbb R^{n\times d}$ denotes the block corresponding to time $t$. Then, we write the new estimator, modified UASE, as follows:
\[
\hat{\mathbf Y}(t):=\frac{1}{n^{1/2}}\hat{\mathbf V}(t)\hat{\mathbf \Sigma}\in \mathbb R^{n\times d},\qquad t\in\mathcal T.
\]
The scaling in this definition is essential. Multiplying the right singular vectors by the full singular-value matrix and normalizing by \(n^{1/2}\) targets the dynamic factor \(\mathbf Y(t)\), rather than an arbitrary balanced factorization of the unfolded matrix. Under anchor isotropy, the scale is fixed by the anchor side of the factorization, leaving only one orthogonal ambiguity shared by all time points. The modified estimator therefore recovers the canonical dynamic latent positions up to one orthogonal transformation shared across all snapshots. We summarize this in the following Theorem. See Appendix~\ref{app:uase-consistency} for the proof.

\begin{theorem}[Modified UASE under canonical geometry]
	\label{thm:modified-uase-paper}
	Assume Assumption~\ref{assump:isotropy} holds and that $T=O((\log n)^{{1/3}})$ as $n\to\infty$. Then there exists a possibly random orthogonal matrix \(\mathbf W\in\mathbb O(d)\) such that
	\[
	\sup_{t\in\mathcal T}\left\|\hat{\mathbf Y}(t)-\mathbf Y(t)\mathbf W\right\|_2=O(\log n)\qquad \text{a.s.}
	\]
\end{theorem}

Relative to the \(O(n^{1/2})\) spectral scale of the latent matrices, the uniform \(O(\log n)\) error becomes \(O(n^{-1/2}\log n)\) after normalization. Because the same \(\mathbf W\) aligns all snapshots, time differences need no pairwise Procrustes alignment. We estimate the second-moment matrix by
\[
\hat{\mathbf M}(t,s):=\frac{1}{n}\bigl(\hat{\mathbf Y}(t)-\hat{\mathbf Y}(s)\bigr)^\top\bigl(\hat{\mathbf Y}(t)-\hat{\mathbf Y}(s)\bigr).
\]

The factor \(1/n\) makes \(\hat{\mathbf M}(t,s)\) an empirical second moment. Because the node embeddings share one alignment, \(\hat{\mathbf M}(t,s)\) estimates \(\mathbf M(t,s)\) up to orthogonal conjugation. We define the estimated trace variation distance as
\[
\hat d_{\mathrm{tr}}(t,s) := \left(\mathrm{tr} \hat{\mathbf M}(t,s)\right)^{1/2} =  \frac{1}{n^{1/2}}\|\hat{\mathbf Y}(t)-\hat{\mathbf Y}(s)\|_F.
\]
Define the estimated aggregated second-moment matrix
\[
\hat{\mathcal M} := \sum_{(t,s)\in\mathcal S}\hat{\mathbf M}(t,s)\in\mathbb R^{d\times d},
\]
and let the estimated basis \(\{\hat{\bm u}_k\}_{k=1}^d\) be given by the orthonormal eigenvectors of \(\hat{\mathcal M}\), with eigenvalues ordered decreasingly.

The estimated \(k\)th mode-wise variation distance is defined as
\[
\hat d_k(t,s) := \left(\hat{\bm u}_k^\top \hat{\mathbf M}(t,s)\hat{\bm u}_k\right)^{1/2} = \frac{1}{n^{1/2}}\|(\hat{\mathbf Y}(t)-\hat{\mathbf Y}(s))\hat{\bm u}_k\|.
\]

For each \(\rho\in\{\mathrm{tr},1,\ldots,d\}\), we apply the same CMDS construction as at the population level, replacing \(d_\rho(t,s)\) by \(\hat d_\rho(t,s)\). 
Since \(\hat d_\rho\) is also Hilbertian, the resulting centered Gram matrix \(\hat{\mathbf E}_\rho\) is positive semidefinite. This produces the estimated trajectory \(\hat\psi_\rho\in\mathbb R^{c_\rho}\). If \(\operatorname{rank}(\hat{\mathbf E}_\rho)<c_\rho\), the remaining coordinates are set to zero.

The next theorem shows that recovery of the second-moment geometry propagates through the distance construction and CMDS. See Appendix~\ref{app:statistical-recovery} for the proof.

\begin{theorem}
\label{thm:trajectory-consistency-main}
Assume that Assumptions~\ref{assump:isotropy} and
\ref{assump:separated_eigenvalues} hold, with \(T=O((\log n)^{{1/3}})\) and \(|\mathcal S|/\delta=O(1)\) as $n\to\infty$.
For any $\rho\in\{\mathrm{tr},1,\ldots,k_{\mathrm{max}}\}$ and
$c_\rho\le \operatorname{rank}(\mathbf E_\rho)$ satisfying
\(
\lambda_{c_\rho}(\mathbf E_\rho)
-
\lambda_{c_\rho+1}(\mathbf E_\rho)
\ge \gamma
\)
for some constant $\gamma>0$ independent of $n$,
there exists a possibly random orthogonal matrix
$\mathbf R_\rho\in\mathbb O(c_\rho)$ such that
\[
\left(
\sum_{t=1}^T
\bigl\|
\hat{\psi}_\rho(t)-\mathbf R_\rho\psi_\rho(t)
\bigr\|^2
\right)^{1/2}
=
O\!\left(\frac{T^{3/2}\log n}{n^{1/2}}
\right)
\qquad\text{a.s.}
\]
\end{theorem}

This result completes the operator-to-distance-to-trajectory argument. The matrix \(\mathbf R_\rho\) accounts for the intrinsic orthogonal ambiguity of Euclidean coordinates produced by CMDS, while the rate shows that the aggregate squared trajectory error vanishes. Thus, consistency is established for the temporal representation itself, rather than only for the intermediate node embeddings.

\subsection{Node attribution}

Moreover, we define the node attributions:
\[
\begin{aligned}
	\hat a_{\mathrm{tr}}(t,s;i)
	&=
	\frac{1}{n^{1/2}}
	\|\hat{\mathbf Y}_{i:}(t)-\hat{\mathbf Y}_{i:}(s)\|,
	\\
	\hat a_k(t,s;i)
	&=
	\frac{1}{n^{1/2}}
	\left\langle
	\hat{\mathbf Y}_{i:}(t)-\hat{\mathbf Y}_{i:}(s),
	\hat{\bm u}_k
	\right\rangle,
	\qquad k\in\{1,\dots,d\}.
\end{aligned}
\]
Trace attribution gives node \(i\)'s nonnegative total displacement, while signed mode-wise attribution gives its movement along mode \(k\). Squared node attributions sum exactly to the estimated squared distance. The next theorem controls the effect of CMDS truncation.


\begin{theorem}[Node-level to trajectory-level control]
	\label{thm:local-global-main}
	For any $\rho\in\{\mathrm{tr},1,\ldots,d\}$, $c_\rho\le \operatorname{rank}(\mathbf E_\rho)$ and $t,s\in\mathcal T$,
	\[
	\left|
	\|\hat\psi_{\rho}(t)-\hat\psi_{\rho}(s)\|^2
	-
	\sum_{i=1}^n \hat a_{\rho}(t,s;i)^2
	\right|
	\le
	2\left(
	\sum_{i=c_\rho+1}^T
	\lambda_i(\hat{\mathbf E}_{\rho})^2
	\right)^{1/2}.
	\]
\end{theorem}


Node attributions exactly decompose the estimated squared distance. The theorem shows that CMDS truncation is the only source of discrepancy with the squared trajectory displacement. Retaining the full positive spectrum removes it. See Appendix~\ref{app:node-attribution} for details and proof.

\subsection{Change point detection}

These low-dimensional temporal trajectories also lead naturally to inference on temporal events. In particular, once a trace variation or mode-wise trajectory has been estimated, change point detection can be conducted directly on that trajectory. For a 1D estimated trajectory \(\hat\psi_\rho(t)\in\mathbb R\), define the estimator
\[
\hat t_\rho\in\arg\min_{k\in\mathcal K}\min_{\theta}\sum_{t=1}^T\bigl(\hat\psi_\rho(t)-\Psi(t;k,\theta)\bigr)^2,
\]
where \(\Psi(t;k,\theta)\) is a piecewise parametric trajectory with candidate knot \(k\in \mathcal K\subseteq \mathcal{T}\). In particular, piecewise constant models encode \(0\)th order changes and piecewise linear models encode \(1\)st order changes~\citep{chen2024euclidean}. For a 1D population trajectory $\psi_\rho(t)\in \mathbb R$, let
\[
Q_\rho(k):=\min_{\theta}\sum_{t=1}^T\left(\psi_\rho(t)-\Psi(t;k,\theta)\right)^2
\]
and let \(t_\rho^*\in\arg\min_{k\in\mathcal K}Q_\rho(k)\) denote the population change point. Under the uniqueness and separation assumptions stated in Appendix~\ref{app:cp}, the localization error is controlled by both the trajectory estimation error and the approximation error of the piecewise model.

\begin{theorem}[Change point localization]
	\label{thm:cp-localization}
	Under the uniqueness and separation assumptions,
	\[
	\left| \hat t_\rho - t_\rho^* \right|
	\le
	\frac{4}{\alpha D(\theta^*)}
	\left\{
	\left(\sum_{t=1}^{T}\left(\hat\psi_\rho(t)-\psi_\rho(t)\right)^2\right)^{1/2}
	+
	Q_\rho(t_\rho^*)^{1/2}
	\right\}^2,
	\]
	where \(D(\theta^*)>0\) measures the change strength and \(\alpha=\alpha_T>0\) is the separation coefficient specified in Appendix~\ref{app:cp}. The bound holds for each fixed \(T\), and \(\alpha_T\) may decrease as \(T\) grows.
\end{theorem}

The bound separates two sources of localization error. The first term measures error in estimating the population trajectory, while \(Q_\rho(t_\rho^*)\) measures how well the chosen piecewise parametric family represents that population trajectory. The quantities \(\alpha\) and \(D(\theta^*)\) quantify separation between piecewise templates: better-separated changes are more stable under perturbations of the trajectory. Theorem~\ref{thm:cp-localization} concerns the single-knot least-squares estimator above. The multiple-change-point procedure in Section~\ref{sec:exp} instead scores level and slope innovations on the estimated mode-wise trajectories and is evaluated empirically.


\section{Experiments}
\label{sec:exp}

We evaluate MENT on two synthetic datasets with known latent structure and change points, and on two real-world datasets. One is a U.S. court opinion citation network~\citep{freelawproject2026bulk} and the other is Enron e-mail network~\citep{klimt2004enron,cohen2015enron}. The experiments test four claims: (i) unresolved \(\mathrm{GL}(d)\) ambiguity distorts temporal geometry, (ii) our proposed embedding consistently recovers trace and mode-wise trajectories, (iii) trajectory-level displacements admit faithful node-level attribution and (iv) mode-wise trajectories improve localization of heterogeneous change points. Both synthetic datasets are dynamic Stochastic Block Models (SBMs)~\citep{holland1983stochastic} with three equal communities and three orthogonal latent modes. Dataset~1 has one change point per mode, while Dataset~2 has multiple mode-specific change points. Full details and additional validation are given in Appendix~\ref{app:experiments}.

\textbf{Synthetic Dataset 1:} Figure~\ref{fig:mirror_overview_scatter} validates the geometric and attribution claims in a controlled setting. Panel~(a) shows that the trace and mode-wise trajectory estimation errors of the modified UASE decrease as \(n\) increases, in line with Theorems~\ref{thm:modified-uase-paper} and~\ref{thm:trajectory-consistency-main}. By contrast, the original UASE retains appreciable error over the tested sample sizes. This pattern is consistent with its population target being a generally nonorthogonal \(\mathrm{GL}(d)\) transform of the canonical dynamic positions, which can distort the induced second-moment geometry. Panel~(b) shows that the recovered mode-wise trajectories are nearly identical to the population trajectories, and isolate the planted \(0\)th and \(1\)st order changes, whereas the trace trajectory aggregates them into a global summary (Section~\ref{subsec:multiscale-decomposition}). Panel~(c) illustrates the node-to-trajectory bound: trajectory-level displacements closely match the corresponding aggregate node-level contributions for each time pair, and the discrepancy between their squared values is controlled by the CMDS truncation residual (Theorem~\ref{thm:local-global-main}).


\begin{figure}[tbp]
	\centering
	\begin{subfigure}[t]{0.34\linewidth}
		\centering
		\includegraphics[width=\linewidth]{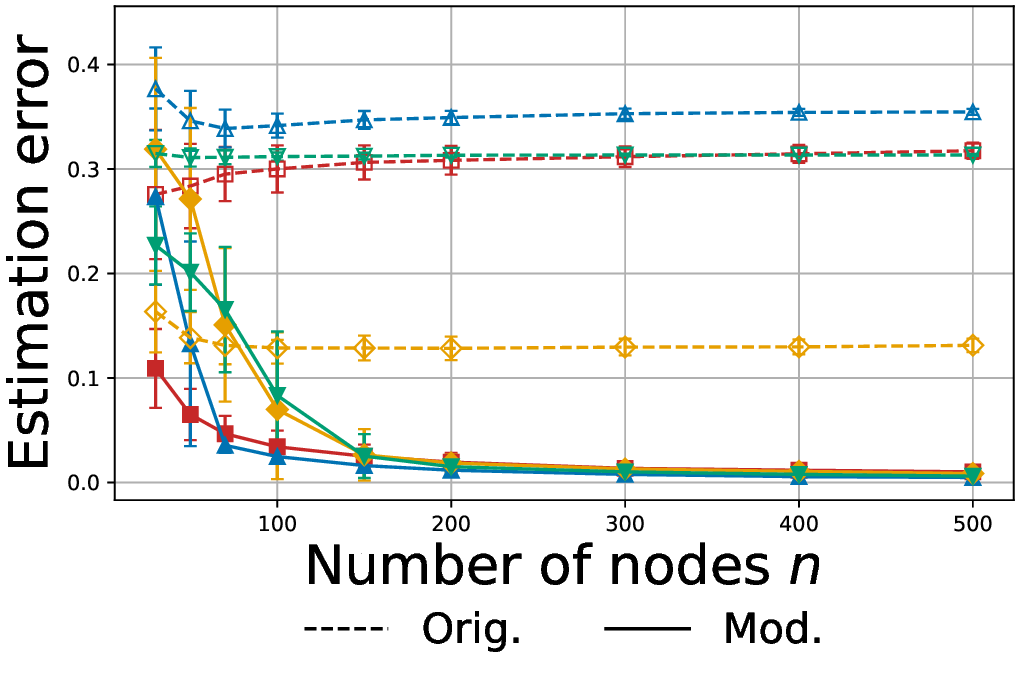}
		\caption{Trajectory estimation error}
		\label{fig:mirror_overview_2_mv}
	\end{subfigure}\hfill
	\begin{subfigure}[t]{0.34\linewidth}
		\centering
		\includegraphics[width=\linewidth]{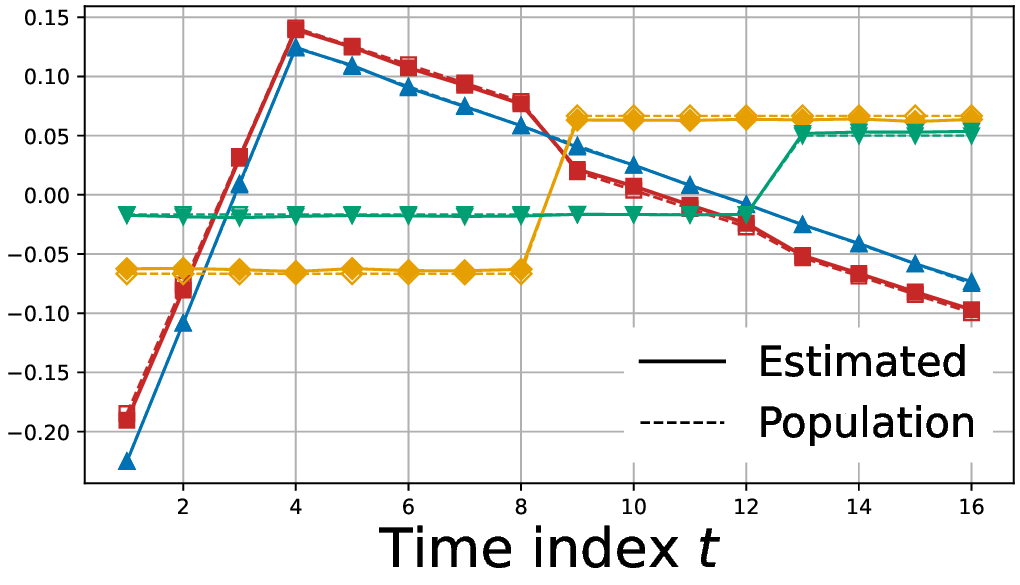}
		\caption{Trace and mode-wise trajectories}
		\label{fig:mirror_overview_2_mode}
	\end{subfigure}\hfill
	\begin{subfigure}[t]{0.24\linewidth}
		\centering
		\includegraphics[width=\linewidth]{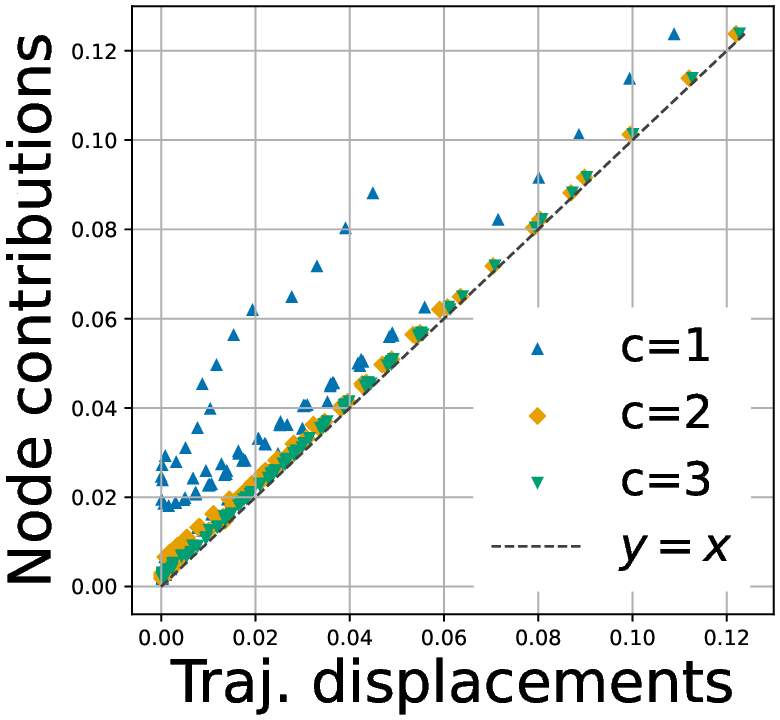}
		\caption{Trajectory-level displacements and node-level contributions}
		\label{fig:global_vs_local_scatter}
	\end{subfigure}
	
	\caption{
		Trajectory recovery and node attribution results on Dataset 1. Panel~(a) shows trajectory estimation error for the trace and mode-wise trajectories, comparing the proposed modified UASE with the original UASE. Panel~(b) compares the estimated and population trajectories under the modified UASE for one realization of the dataset where the number of nodes is 500. Panel~(c) compares displacements of the trace trajectory with aggregated node-level contributions for each time pair \((t,s)\) with \(t<s\) using the same data as (b). In panels~(a) and (b), the trace variation is depicted in red (square markers), and the three mode-wise components are depicted in blue (triangle-up), orange (diamond), and green (triangle-down), respectively.
	}
	\label{fig:mirror_overview_scatter}
\end{figure}


\textbf{Synthetic Dataset 2:} Next, we evaluated the proposed framework on a change point detection task. We used the more challenging synthetic Dataset~2, in which six ground-truth change points occur at \(t\in\{11,21,31,41,51,61\}\) across three modes with heterogeneous signal strengths and change orders. We compared the proposed method against DeltaCon~\citep{koutra2016deltacon}, LAD~\citep{huang2020laplacian}, HCDL~\citep{fukushima2020detecting}, Euclidean Mirror~\citep{athreya2025euclidean}, and the very recently proposed two-stage method of \citet{wang2026change}. For each method, we ranked the candidate change points and evaluated the top-\(K\) performance for \(K\in\{3,6,9\}\) using the F1 score and timing mean absolute error (MAE), averaged over \(100\) independent trials. For MENT, we fit a local linear trend model~\citep{seabold2010statsmodels} for each mode \(k\) to the 1D trajectory \(\hat{\psi}_k(t)\). This produces two complementary scores: \(s_{k,t}^{(0)}\), which detects level shifts, and \(s_{k,t}^{(1)}\), which detects slope changes. We fuse the normalized scores (see Appendix~\ref{app:exp-cp-comparison} for details). All methods use the same matching tolerance and top-\(K\) evaluation protocol. Candidate separation is two time units for single score methods and four for fused score methods, as detailed in Appendix~\ref{app:exp-cp-comparison}. Method-specific hyperparameters are tuned by grid search as detailed in Appendix~\ref{app:exp-cp-comparison}.


The left table in Figure~\ref{tab:cp_results_combined} shows that the proposed method performs best overall across all three values of $K$ (see Appendix Table~\ref{tab:performance_ci} for results with 95\% confidence intervals). At $K=3$, it matched the best baseline F1 while achieving perfect localization on matched detections. At $K=6$, which equals the true number of change points, it achieved near-perfect recovery (F1 $=0.965$, timing MAE $=0.1$), substantially outperforming all competitors. It also remained strongest at $K=9$, where the nine selected candidates outnumber the six true change points. Among the competitors, LAD and the two-stage method were the most competitive but still missed weaker mode-specific events. Overall, these results support the claim that multiscale trajectories can improve detection in settings where distinct temporal events occur in different latent modes and are blurred by global summaries. Appendix~\ref{app:experiments} provides additional validation of distance recovery, basis sensitivity, trajectory recovery, attribution bounds, dimension robustness, and baseline comparisons.

\begin{figure}[t]
	\centering
    \begin{minipage}[c]{0.3\linewidth}
    	\centering
    	\small
    	\setlength{\tabcolsep}{3.5pt}
    	\begin{tabular}{lrrrrrr}
    		\toprule
    		& \multicolumn{2}{c}{$K=3$} & \multicolumn{2}{c}{$K=6$} & \multicolumn{2}{c}{$K=9$} \\
    		\cmidrule(lr){2-3}\cmidrule(lr){4-5}\cmidrule(lr){6-7}
    		Method & F1 & MAE & F1 & MAE & F1 & MAE \\
    		\midrule
    		DeltaCon  & 0.447          & 2.09          & 0.513          & 2.40          & 0.579          & 2.29 \\
    		LAD       & \textbf{0.667} & 0.73          & 0.833          & 1.26          & 0.688          & 1.89 \\
    		HCDL      & 0.322          & 2.42          & 0.480          & 2.21          & 0.532          & 2.15 \\
    		Mirror    & 0.664          & 0.01          & 0.573          & 1.56          & 0.512          & 2.08 \\
    		Two-stage & \textbf{0.667} & \textbf{0.00} & 0.815          & 1.04          & 0.692          & 2.07 \\
    		\midrule
    		MENT      & \textbf{0.667} & \textbf{0.00} & \textbf{0.965} & \textbf{0.10} & \textbf{0.800} & \textbf{1.00} \\
    		\bottomrule
    	\end{tabular}
    \end{minipage}
	\hfill
	\begin{minipage}[c]{0.44\linewidth}
		\centering
		\includegraphics[width=\linewidth,trim=0 0.3cm 0 0.2cm,clip]{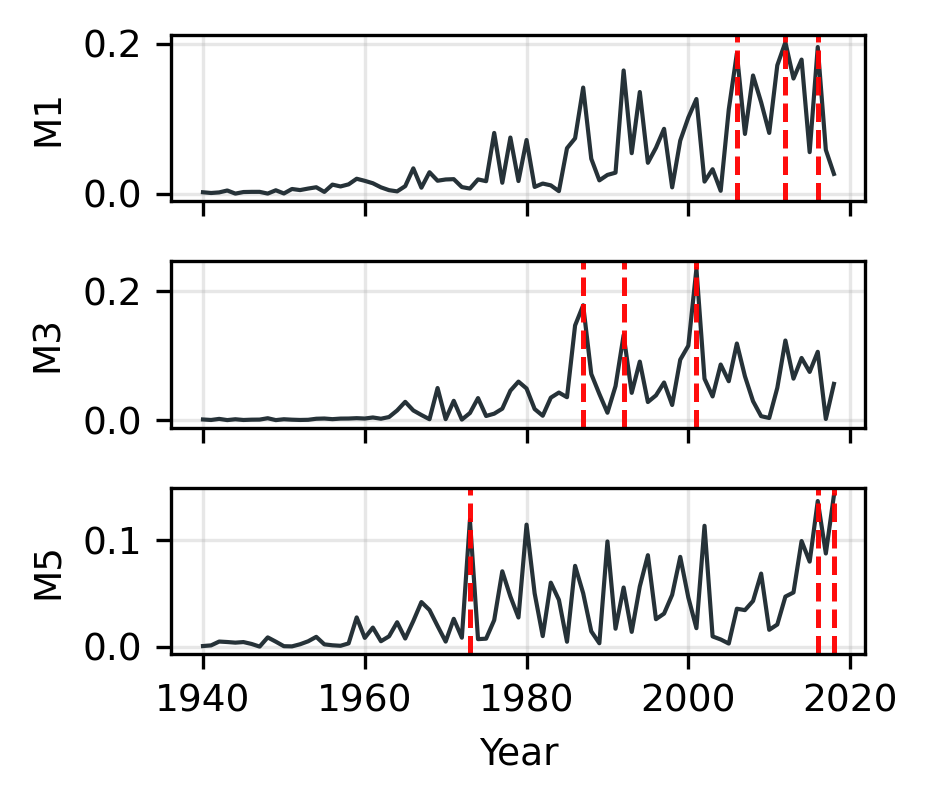}
	\end{minipage}
	\vspace{-0.4em}
	\caption{Results of change point detection on synthetic and real-world data. Left: F1 and timing MAE across competing methods on Dataset~2. Right: Mode-wise level change point score plots on the U.S. court citation dataset. Dashed red line shows detected change points.}
	\label{tab:cp_results_combined}
\end{figure}

\textbf{Real-world interpretation of U.S. court citation dynamics:} We conclude with a case study on the U.S. court opinion citation network dataset (See Appendix~\ref{app:exp-real} for details on how we constructed this dataset). The right panel of Figure~\ref{tab:cp_results_combined} shows selected mode-wise level ($0$th order) change point score plots. In this case study, we focus on \(s_{k,t}^{(0)}\) because level shifts correspond naturally to abrupt changes in the co-citation patterns. We also present the detailed node attribution tables in Appendix~\ref{app:exp-real-additional}. 

The selected modes reveal several interpretable forms of citation reorganization. Mode~1 shows the broadest restructuring, with its strongest level shifts occurring from the mid-2000s to the mid-2010s. Its attributions separate a broad constitutional standards and rights or remedies neighborhood, including cases such as \textit{City of Cleburne}, \textit{Graham}, \textit{Mapp}, and \textit{Roper}, from a more procedure-specific criminal adjudication and post-conviction neighborhood, including \textit{Stovall}, \textit{Ker}, \textit{Schlup}, and \textit{Cronic}. Mode~3 is more concentrated. Its shifts around the late 1980s, early 1990s, and early 2000s organize Fourth Amendment and police litigation, including encounter, consent, automobile search, search incident, warrant review, detention, and qualified immunity cases. Mode~5 is broader, but provides a useful contrast: its peaks in 1973, 2016, and 2018 link constitutional claims against government action to doctrines governing reviewability, liability, remedies, and punishment. Taken together, the case study suggests that the learned trajectories recover structured shifts in citation patterns corresponding to recognizable doctrinal regimes rather than merely surfacing isolated highly cited cases. We provide an additional real-world example using the Enron e-mail dataset in Appendix~\ref{app:enron}.

\section{Conclusion}

We introduced Multiscale Euclidean Network Trajectories (MENT), an identifiable framework for describing dynamic networks through pairwise second-moment operators. Canonical anchor normalization reduces the \(\mathrm{GL}(d)\) ambiguity of unfolded latent factorizations to a shared orthogonal transformation. This makes trace and shared mode-wise distances well defined, and CMDS converts them into complementary global and mode-wise trajectories of network evolution. Modified UASE provides a consistent estimator of this geometry, with second-moment recovery propagating to the induced trajectories. The same construction supports node-level attribution and mode-specific change point localization. Across synthetic and real networks, the experiments illustrate that these resolutions reveal not only when a network changes, but also how its structure changes and which nodes contribute to that change. All code accompanying this paper is hosted on \url{https://github.com/hisanor013/ment}.

\section*{Acknowledgments}
We thank Yuki Takazawa, Ryoma Kondo, Kohei Miyaguchi, Tilmann Altwicker and Zhivko Taushanov for helpful discussions. R.H. is supported by JST FOREST Program (JPMJFR216Q), JST PRESTO Program (JPMJPR2469), Grant-in-Aid for Scientific Research (KAKENHI, JP24K03043), JSPS International Joint Research Program (JRPs with SNSF: 20251501), and the UTEC-UTokyo FSI Research Grant Program.

\bibliographystyle{unsrtnat}
\bibliography{references}

\newpage
\appendix

\renewcommand{\thefigure}{\thesection.\arabic{figure}}
\renewcommand{\thetable}{\thesection.\arabic{table}}
\setcounter{figure}{0}
\setcounter{table}{0}

\section{Appendix Guide and Notation}
\label{app:guide}
The appendix follows the chain from canonical population geometry to statistical recovery and downstream interpretation. Appendix~\ref{app:proofs} proves the invariance and decomposition results underlying the population geometry. Appendix~\ref{app:uase-consistency} proves the consistency of modified UASE estimator under canonical parametrization. Appendix~\ref{app:statistical-recovery} studies recovery of the second-moment operators and propagates the resulting error through shared modes, distances, and CMDS trajectories. Appendix~\ref{app:node-attribution} proves the exact node-level decomposition and the CMDS truncation bound. Appendix~\ref{app:cp} establishes change point localization and verifies separation for level and slope changes. Appendix~\ref{app:experiments} provides the experimental constructions, implementation details, and additional analyses of the synthetic and real networks.

The arguments distinguish the population operator \(\mathbf M(t,s)\), its finite-sample latent analogue \(\tilde{\mathbf M}(t,s)\), and the observed-data estimate \(\hat{\mathbf M}(t,s)\). Table~\ref{tab:appendix-notation} summarizes the notation used throughout.

\begingroup
\small
\begin{longtable}{@{}p{0.31\textwidth}p{0.63\textwidth}@{}}
\caption{Notation used in the main text and appendix.}
\label{tab:appendix-notation}\\
\toprule
Notation & Meaning \\
\midrule
\endfirsthead
\toprule
Notation & Meaning \\
\midrule
\endhead

\multicolumn{2}{@{}l}{\emph{Model and observations}}\\
\(n,T,d\) & Numbers of nodes, time points, and latent dimensions.\\
\(\mathcal T=\{1,\ldots,T\}\) & Set of observed time points.\\
\(\chi,\phi(t)\) & Population anchor and dynamic latent random vectors.\\
\(\mathbf B(t)\) & Symmetric matrix relating \(\phi(t)\) to \(\chi\) in the generative model.\\
\(\mathbf X,\mathbf Y(t)\) & Finite-sample anchor and dynamic latent position matrices.\\
\(\mathbf Y\) & Matrix stacking \(\mathbf Y(t)\) across time.\\
\(\mathbf A(t),\mathbf P(t)\) & Observed adjacency and edge-probability matrices at time \(t\).\\
\(\mathbf A,\mathbf P\) & Unfolded adjacency and probability matrices.\\
\(\mathrm{GL}(d),\mathbb O(d)\) & General linear and orthogonal groups.\\

\addlinespace
\multicolumn{2}{@{}l}{\emph{Second-moment geometry and estimation}}\\
\(\mathbf M(t,s)\) & Population second-moment matrix of \(\phi(t)-\phi(s)\).\\
\(\tilde{\mathbf M}(t,s)\) & Finite-sample second-moment matrix formed from \(\mathbf Y(t)-\mathbf Y(s)\).\\
\(\hat{\mathbf Y}(t)\) & Modified-UASE node embedding at time \(t\).\\
\(\hat{\mathbf M}(t,s)\) & Estimated second-moment matrix formed from \(\hat{\mathbf Y}(t)-\hat{\mathbf Y}(s)\).\\
\(\mathcal S\) & Selected set of time pairs used to aggregate second-moment matrices.\\
\(\mathcal M,\hat{\mathcal M}\) & Population and estimated second-moment matrices aggregated over \(\mathcal S\).\\
\(k_{\mathrm{max}}\) & Number of modes retained for individual analysis.\\
\(\bm u_k,\hat{\bm u}_k\) & Population and estimated eigenvectors defining shared mode \(k\).\\
\(d_{\mathrm{tr}},d_k\) & Population trace and mode-wise variation distances.\\
\(\hat d_{\mathrm{tr}},\hat d_k\) & Estimated trace and mode-wise variation distances.\\

\addlinespace
\multicolumn{2}{@{}l}{\emph{Trajectories, attribution, and inference}}\\
\(\rho\in\{\mathrm{tr},1,\ldots,d\}\) & Index for the trace geometry or a mode-wise geometry.\\
\(\mathbf D_\rho^{(2)},\hat{\mathbf D}_\rho^{(2)}\) & Population and estimated matrices of squared temporal distances.\\
\(\mathbf J\) & Centering matrix \(\mathbf I_T-T^{-1}\mathbf 1\mathbf 1^\top\).\\
\(\mathbf E_\rho,\hat{\mathbf E}_\rho\) & Population and estimated centered Gram matrices used by CMDS.\\
\(c_\rho\) & Dimension retained for the trajectory of type \(\rho\).\\
\(\psi_\rho(t),\hat\psi_\rho(t)\) & Population and estimated trajectory coordinates at time \(t\).\\
\(\hat a_{\mathrm{tr}}(t,s;i),\hat a_k(t,s;i)\) & Trace and mode-wise attributions of node \(i\) to the comparison of \(t\) and \(s\).\\
\(\mathcal K,\Psi(t;k,\theta)\) & Candidate knot set and piecewise parametric trajectory model.\\
\(Q_\rho(k),\hat Q_\rho(k)\) & Population and estimated piecewise-trajectory fitting criteria.\\
\(t_\rho^*,\hat t_\rho\) & Population and estimated change-point locations.\\
\(\alpha,D(\theta^*)\) & Separation coefficient and model-specific change strength.\\

\addlinespace
\multicolumn{2}{@{}l}{\emph{Alignment and error notation}}\\
\(\mathbf W,\mathbf R_\rho\) & Shared orthogonal alignment of node embeddings and orthogonal alignment of CMDS trajectories.\\
\(E_M,E_{u,k}\) & Operator-estimation and sign-invariant mode-\(k\) basis errors used in the proofs.\\
\(\|\cdot\|_2,\|\cdot\|_F\) & Euclidean or operator norm, as applicable, and Frobenius norm.\\
\(\lambda_k(\cdot)\) & \(k\)th eigenvalue in decreasing order.\\

\bottomrule
\end{longtable}
\endgroup

We use almost sure asymptotic notation for statistical results \citep{jones2020multilayer, gallagher2021spectral}.
\begin{definition}[$O(f)$ almost surely]
	Let $\{X_n\}$ be a sequence of random variables and let $f(n)$ be a positive function.
	We say that
	\[
	| X_n | = O(f(n)) \quad \text{almost surely}
	\]
	if the following holds. For every $A > 0$, there exist a constant $c_A > 0$ and an integer $N_A$ such that, for all $n \ge N_A$,
	\[
	\Pr\bigl(|X_n| \le c_A f(n)\bigr) \ge 1 - n^{-A}.
	\]
\end{definition}

\section{Population Model and Canonical Geometry}
\label{app:proofs}

\subsection{Identifiability and Invariance of Temporal Geometry}
\label{app:identifiability-invariance}

We first formalize the equivalence relation between population objects.
Under this equivalence, the induced edge probabilities satisfy
\[
\mathbf P_{ij}(t)=\tilde{\mathbf P}_{ij}(t)
\qquad\text{a.s. for all }i<j,\ t\in\mathcal T.
\]

\begin{definition}[Equivalence of population objects]
\label{def:equivalence}
We say that two population objects
$(\chi,\{\phi(t)\}_{t\in\mathcal T})$
and
$(\tilde\chi,\{\tilde\phi(t)\}_{t\in\mathcal T})$
are equivalent, denoted by
\[
(\chi,\{\phi(t)\}_{t\in\mathcal T})
\equiv
(\tilde\chi,\{\tilde\phi(t)\}_{t\in\mathcal T}),
\]
if
\[
\chi(\omega)^\top\phi(t,\omega')
=
\tilde\chi(\omega)^\top\tilde\phi(t,\omega')
\]
almost surely for two independent population draws
$\omega,\omega'\in\Omega$, for every $t\in\mathcal T$.
\end{definition}

We next prove Proposition~\ref{prop:identifiability}.

\begin{proof}
We first prove sufficiency. Suppose that there exists
$\mathbf G\in\mathrm{GL}(d)$ such that
$\tilde\chi=\mathbf G^{-1}\chi$ and
$\tilde\phi(t)=\mathbf G^\top\phi(t)$ a.s. for all $t\in\mathcal T$.
Then, for two independent population draws $\omega,\omega'\in\Omega$,
\[
\tilde\chi(\omega)^\top\tilde\phi(t,\omega')
=
\chi(\omega)^\top\mathbf G^{-\top}
\mathbf G^\top\phi(t,\omega')
=
\chi(\omega)^\top\phi(t,\omega')
\]
almost surely for every $t\in\mathcal T$. Hence the two population objects
are equivalent. Under Assumption~\ref{assump:isotropy}, the same argument
applies for $\mathbf G\in\mathbb O(d)$, since
$\mathbb O(d)\subseteq\mathrm{GL}(d)$.

We next prove necessity. Suppose that
$(\chi,\{\phi(t)\}_{t\in\mathcal T})$ and
$(\tilde\chi,\{\tilde\phi(t)\}_{t\in\mathcal T})$ are equivalent. By
Definition~\ref{def:equivalence},
\[
\chi(\omega)^\top\phi(t,\omega')
=
\tilde\chi(\omega)^\top\tilde\phi(t,\omega')
\quad\text{a.s.}
\]
for two independent population draws $\omega,\omega'\in\Omega$ and every
$t\in\mathcal T$. Multiplying both sides by $\tilde\chi(\omega)$ and taking
expectation with respect to the first population draw $\omega$ gives
\[
\mathbb E[\tilde\chi\chi^\top]\phi(t,\omega')
=
\mathbb E[\tilde\chi\tilde\chi^\top]\tilde\phi(t,\omega')
\quad\text{a.s.}
\]
Since $\mathbb E[\tilde\chi\tilde\chi^\top]$ is invertible, define
\[
\mathbf G^\top
:=
\{\mathbb E[\tilde\chi\tilde\chi^\top]\}^{-1}
\mathbb E[\tilde\chi\chi^\top].
\]
It follows that
\[
\tilde\phi(t)=\mathbf G^\top\phi(t)
\quad\text{a.s.}
\]
for every $t\in\mathcal T$. Consequently,
\[
\mathbb E[\tilde\phi(t)\tilde\phi(t)^\top]
=
\mathbf G^\top
\mathbb E[\phi(t)\phi(t)^\top]
\mathbf G.
\]
Since both
$\mathbb E[\tilde\phi(t)\tilde\phi(t)^\top]$ and
$\mathbb E[\phi(t)\phi(t)^\top]$ are invertible, we have
$\mathbf G\in\mathrm{GL}(d)$.

Substituting $\tilde\phi(t,\omega')=\mathbf G^\top\phi(t,\omega')$ back into
the equality of inner products gives
\[
\chi(\omega)^\top\phi(t,\omega')
=
\tilde\chi(\omega)^\top\mathbf G^\top\phi(t,\omega')
\quad\text{a.s.}
\]
Since $\mathbb E[\phi(t)\phi(t)^\top]$ is invertible, this implies
\[
\chi(\omega)=\mathbf G\tilde\chi(\omega)
\quad\text{a.s.}
\]
Hence,
\[
\tilde\chi=\mathbf G^{-1}\chi
\quad\text{a.s.}
\]

Finally, under Assumption~\ref{assump:isotropy} for both $\chi$ and
$\tilde\chi$, we have
\[
\mathbb E[\chi\chi^\top]
=
\mathbb E[\tilde\chi\tilde\chi^\top]
=
\mathbf I_d.
\]
Since $\tilde\chi=\mathbf G^{-1}\chi$ a.s.,
\[
\mathbf I_d
=
\mathbb E[\tilde\chi\tilde\chi^\top]
=
\mathbf G^{-1}
\mathbb E[\chi\chi^\top]
\mathbf G^{-\top}
=
\mathbf G^{-1}\mathbf G^{-\top},
\]
and hence $\mathbf G\in\mathbb O(d)$.
\end{proof}

The preceding proposition shows that, under anchor isotropy, equivalent
population objects differ only by a common orthogonal transformation. We next
show that the temporal geometry is therefore invariant within each
equivalence class.

\begin{proposition}[Invariance of temporal geometry]
\label{prop:geometry-invariance}
Suppose Assumption~\ref{assump:isotropy} holds for both population objects, and let
$(\chi,\{\phi(t)\}_{t\in\mathcal T})$ and
$(\tilde\chi,\{\tilde\phi(t)\}_{t\in\mathcal T})$ be equivalent population
objects. Then, for any $t,s\in\mathcal T$,
\[
d_{\mathrm{tr}}\bigl(\phi(t),\phi(s)\bigr)
=
d_{\mathrm{tr}}\bigl(\tilde\phi(t),\tilde\phi(s)\bigr).
\]
Moreover, if the orthonormal eigenbases
$\{\bm u_k\}_{k=1}^d$ and $\{\tilde{\bm u}_k\}_{k=1}^d$ are constructed from
the corresponding aggregated second-moment matrices $\mathcal M$ and
$\tilde{\mathcal M}:=\sum_{(t,s)\in\mathcal S}
\mathbf M\bigl(\tilde\phi(t),\tilde\phi(s)\bigr)$, and Assumption~\ref{assump:separated_eigenvalues} holds,
then, for every retained mode $k\in[k_{\mathrm{max}}]$,
\[
d_k\bigl(\phi(t),\phi(s)\bigr)
=
d_k\bigl(\tilde\phi(t),\tilde\phi(s)\bigr).
\]
\end{proposition}

\begin{proof}
By Proposition~\ref{prop:identifiability}, there exists
$\mathbf G\in\mathbb O(d)$ such that
\(\tilde\phi(t)=\mathbf G^\top\phi(t)\) a.s. for all $t\in\mathcal T$. Hence,
\[
\mathbf M\bigl(\tilde\phi(t),\tilde\phi(s)\bigr)
=
\mathbf G^\top
\mathbf M\bigl(\phi(t),\phi(s)\bigr)
\mathbf G.
\]
Since the trace is invariant under orthogonal transformations,
\[
\begin{aligned}
d_{\mathrm{tr}}\bigl(\tilde\phi(t),\tilde\phi(s)\bigr)
&=
\left\{
\mathrm{tr}\left[
\mathbf G^\top
\mathbf M\bigl(\phi(t),\phi(s)\bigr)
\mathbf G
\right]
\right\}^{1/2}\\
&=
\left\{
\mathrm{tr}\,
\mathbf M\bigl(\phi(t),\phi(s)\bigr)
\right\}^{1/2}
=
d_{\mathrm{tr}}\bigl(\phi(t),\phi(s)\bigr).
\end{aligned}
\]
Moreover, the aggregated second-moment matrices satisfy
\[
\tilde{\mathcal M}
=
\sum_{(t,s)\in\mathcal S}
\mathbf G^\top
\mathbf M\bigl(\phi(t),\phi(s)\bigr)
\mathbf G
=
\mathbf G^\top\mathcal M\mathbf G.
\]
Under Assumption~\ref{assump:separated_eigenvalues}, the retained eigenvalues
are non-degenerate, so the corresponding eigenvectors agree up to sign:
\[
\tilde{\bm u}_k=\pm \mathbf G^\top\bm u_k,
\qquad k\in[k_{\mathrm{max}}].
\]
Since $\mathbf G\in\mathbb O(d)$, it follows that, for every $k\in[k_{\mathrm{max}}]$,
\[
\begin{aligned}
d_k\bigl(\tilde\phi(t),\tilde\phi(s)\bigr)
&=
\left\{
(\pm \mathbf G^\top\bm u_k)^\top
\mathbf G^\top
\mathbf M\bigl(\phi(t),\phi(s)\bigr)
\mathbf G
(\pm \mathbf G^\top\bm u_k)
\right\}^{1/2}\\
&=
\left\{
\bm u_k^\top
\mathbf M\bigl(\phi(t),\phi(s)\bigr)
\bm u_k
\right\}^{1/2}
=
d_k\bigl(\phi(t),\phi(s)\bigr).
\end{aligned}
\]
\end{proof}

We emphasize that the anchor isotropy assumption (Assumption~\ref{assump:isotropy}) is a normalization rather
than a substantive restriction on the model. Under the assumed invertibility
of the anchor second-moment matrix, any population object admits an equivalent
representation satisfying the isotropy assumption.

\begin{proposition}[Existence of an isotropic parametrization]
\label{prop:isotropic-parametrization}
For any population object
$(\chi,\{\phi(t)\}_{t\in\mathcal T})$,
there exists an equivalent population object
$(\tilde\chi,\{\tilde\phi(t)\}_{t\in\mathcal T})$
satisfying Assumption~\ref{assump:isotropy}.
\end{proposition}

\begin{proof}
Let
\[
\mathbf G
:=
\left\{\mathbb E[\chi\chi^\top]\right\}^{1/2},
\]
which is invertible by the model assumptions, and define
\[
\tilde\chi
:=
\mathbf G^{-1}\chi,
\qquad
\tilde\phi(t)
:=
\mathbf G^\top\phi(t).
\]
By Proposition~\ref{prop:identifiability}, the resulting population objects
are equivalent. Moreover,
\[
\mathbb E[\tilde\chi\tilde\chi^\top]
=
\mathbf G^{-1}
\mathbb E[\chi\chi^\top]
\mathbf G^{-\top}
=
\mathbf I_d.
\]
Thus, the equivalent population object satisfies
Assumption~\ref{assump:isotropy}.
\end{proof}

\subsection{Trace Variation Decomposition and Shared Mode Characterization}
\label{app:trace-variation-shared-modes}

We next establish the orthogonal decomposition of trace variation in
Theorem~\ref{thm:tv-decomposition-paper} and the spectral characterization of
the shared modes in Proposition~\ref{prop:shared_modes}. The former follows
from the orthonormality of the shared basis, while the latter follows from
the variational characterization of the eigenvalues of the aggregated
second-moment matrix.

\begin{proof}[Proof of Theorem~\ref{thm:tv-decomposition-paper}]
Let \(\mathbf U_{\mathrm{b}}=[\bm u_1,\ldots,\bm u_d]\). Since the basis is
orthonormal, \(\mathbf U_{\mathrm{b}}\mathbf U_{\mathrm{b}}^\top
=\mathbf I_d\). For any \(t,s\in\mathcal T\),
\[
\sum_{k=1}^d d_k(t,s)^2
=
\sum_{k=1}^d\bm u_k^\top\mathbf M(t,s)\bm u_k
=
\operatorname{tr}\bigl(
\mathbf U_{\mathrm{b}}^\top\mathbf M(t,s)\mathbf U_{\mathrm{b}}
\bigr)
=
\operatorname{tr}\mathbf M(t,s)
=
d_{\mathrm{tr}}(t,s)^2.
\]
\end{proof}

\begin{proof}[Proof of Proposition~\ref{prop:shared_modes}]
Since $\bm u_k$ is an eigenvector of $\mathcal M$ with $\|\bm u_k\|=1$,
\[
\lambda_k(\mathcal M)
=
\bm u_k^\top\mathcal M\bm u_k
=
\sum_{(t,s)\in\mathcal S}
\bm u_k^\top\mathbf M(t,s)\bm u_k
=
\sum_{(t,s)\in\mathcal S}d_k(t,s)^2.
\]

Moreover, the Courant--Fischer variational characterization of
the eigenvalues of the symmetric matrix $\mathcal M$ gives
\[
\bm u_k
\in
\arg\max_{\substack{
\|\bm u\|=1\\
\bm u\perp\bm u_1,\ldots,\bm u_{k-1}
}}
\bm u^\top\mathcal M\bm u
=
\arg\max_{\substack{
\|\bm u\|=1\\
\bm u\perp\bm u_1,\ldots,\bm u_{k-1}
}}
\sum_{(t,s)\in\mathcal S}
\bm u^\top\mathbf M(t,s)\bm u.
\]
\end{proof}

\section{Consistency of Modified UASE under canonical geometry}
\label{app:uase-consistency}

In this section, we prove Theorem~\ref{thm:modified-uase-paper} by adapting the multilayer UASE analysis of \citet{jones2020multilayer} to the present isotropic setting.

\noindent\textbf{Step 1: Population identification (what unfolding targets).}
We study the rank \(d\) singular value decomposition of the unfolded \emph{population} matrix
\(
\mathbf P=[\mathbf P(1)\mid \cdots\mid \mathbf P(T)]=
\tilde{\mathbf U}\tilde{\mathbf\Sigma}\tilde{\mathbf V}^\top
\).
This analysis shows that the block of right singular vectors \(\tilde{\mathbf V}(t)\) encodes \(\mathbf Y(t)\) in a shared coordinate system, up to a time independent linear map. Assumption~\ref{assump:isotropy} fixes the population second-moment structure and thereby restores the identifiability from general linear transformations up to a single orthogonal factor. Concretely, we obtain an alignment of the form
\(
\tilde{\mathbf V}(t)\tilde{\mathbf \Sigma}\approx n^{1/2}\mathbf Y(t)\mathbf W_1
\)
with a common orthogonal transformation $\mathbf W_1$ across all \(t\).

\noindent\textbf{Step 2: Empirical perturbation (how replacing \(\mathbf P\) by \(\mathbf A\) affects the embedding).} We then compare the empirical unfolded adjacency matrix \(\mathbf A\) to its expectation \(\mathbf P\). Using multilayer spectral perturbation bounds, we control the deviation between the empirical singular factors \((\mathbf V(t),\mathbf \Sigma)\) and their population counterparts \((\tilde{\mathbf V}(t),\tilde{\mathbf \Sigma})\). 
This yields
\(
\mathbf V(t)\mathbf\Sigma\approx \tilde{\mathbf V}(t)\tilde{\mathbf \Sigma}\mathbf W_2
\)
with a common orthogonal transformation $\mathbf W_2$ across all \(t\).

\noindent\textbf{Combining the steps.}
Composing the two alignments with \(\mathbf W=\mathbf W_1\mathbf W_2\in\mathbb O(d)\) and applying the triangle inequality gives uniform operator norm consistency of the modified UASE embeddings:
\[
\sup_{t\in\mathcal T}\|\hat{\mathbf Y}(t)-\mathbf Y(t)\mathbf W\|_2 =O(\log n)\qquad \text{a.s.}
\]

The first result identifies the population right singular vectors of the unfolded probability matrix $\mathbf P$ with the latent positions $\mathbf Y(t)$, up to an orthogonal transformation that is common across time.

\begin{lemma}\label{lem:pop-ident}
	Assume Assumption~\ref{assump:isotropy} holds. There exists a possibly
	random orthogonal matrix \(\mathbf{W} \in \mathbb O(d)\) such that
	\[
	\sup_{t\in\mathcal T}\left\|
	\tilde{\mathbf{V}}(t)\tilde{\mathbf{\Sigma}}
	-
	n^{1/2}\mathbf{Y}(t)\mathbf{W}
	\right\|_2
	=
	O\left(n^{1/2}\log n\right)
	\quad \text{a.s.}
	\]
\end{lemma}

\begin{proof}
	This lemma establishes a population level alignment between the right singular vectors of the unfolded probability matrix and the latent positions. The argument begins by controlling finite sample second-moments of the latent positions. Concentration of $\mathbf X^\top\mathbf X$ and $\mathbf Y(t)^\top\mathbf Y(t)$ around their expectations ensures that both matrices are well conditioned and of full rank, which is required to invoke singular value decompositions at the population level.

First, by Hoeffding's inequality and Assumption~\ref{assump:isotropy}, we have
\[
\begin{aligned}
\left\|\mathbf X^\top\mathbf X-n\mathbf I_d\right\|_2
=O(n^{1/2}\log n)
\qquad\text{a.s.}
\end{aligned}
\]
and $\mathbf X$ is of rank $d$ almost surely.

Moreover, since $\mathbf X^\top\mathbf X$ is positive semidefinite,
\[
\begin{aligned}
\left\|(\mathbf X^\top\mathbf X)^{1/2}-n^{1/2}\mathbf I_d\right\|_2
&=
\max_{k=1,\dots,d}
\left|
\lambda_k(\mathbf X^\top\mathbf X)^{1/2}-n^{1/2}
\right|\\
&=
\max_{k=1,\dots,d}
\frac{
\left|\lambda_k(\mathbf X^\top\mathbf X)-n\right|
}{
\lambda_k(\mathbf X^\top\mathbf X)^{1/2}+n^{1/2}
}\\
&\le
\frac{1}{n^{1/2}}
\left\|\mathbf X^\top\mathbf X-n\mathbf I_d\right\|_2\\
&=
O(\log n)
\qquad\text{a.s.}
\end{aligned}
\]
Similarly, for any fixed $t\in\mathcal T$, $\mathbf Y(t)$ is of rank $d$
almost surely. Since $\mathbf Y$ contains $\mathbf Y(t)$ as a row block,
$\mathbf Y$ is also of rank $d$ almost surely.
Moreover, since $\mathcal Y(t)$ is uniformly bounded by $\beta$,
\[
\sup_{t\in\mathcal T}
\left\|\mathbf{Y}(t)\right\|_2
\le
\beta n^{1/2}
=
O(n^{1/2})
\qquad\text{a.s.}
\]

On the event that $\mathbf X$ and $\mathbf Y$ have rank $d$, $\mathbf P=\mathbf X\mathbf
Y^\top$ has rank $d$, and hence $\tilde{\mathbf\Sigma}$ is invertible. From
\[
\mathbf X\mathbf Y^\top
=
\tilde{\mathbf U}\tilde{\mathbf\Sigma}\tilde{\mathbf V}^\top,
\]
we obtain
\[
\tilde{\mathbf U}
=
\mathbf X\mathbf Y^\top
\tilde{\mathbf V}\tilde{\mathbf\Sigma}^{-1}
=:\mathbf X\mathbf R,
\qquad
\mathbf R
:=
\mathbf Y^\top\tilde{\mathbf V}\tilde{\mathbf\Sigma}^{-1}.
\]
Since $\mathbf X^\top\mathbf X$ is invertible, define
\[
\mathbf W:=(\mathbf X^\top\mathbf X)^{1/2}\mathbf R.
\]
Then
\[
\mathbf W^\top\mathbf W
=
\mathbf R^\top\mathbf X^\top\mathbf X\mathbf R
=
\tilde{\mathbf U}^\top\tilde{\mathbf U}
=
\mathbf I_d.
\]
Thus, $\mathbf W\in\mathbb O(d)$. Consequently,
\[
\tilde{\mathbf V}\tilde{\mathbf\Sigma}
=
\mathbf P^\top\tilde{\mathbf U}
=
\mathbf Y\mathbf X^\top\mathbf X
\mathbf R
=
\mathbf Y(\mathbf X^\top\mathbf X)^{1/2}\mathbf W.
\]

Hence, there exists a possibly random orthogonal matrix \(\mathbf{W} \in \mathbb O(d)\) such that
\[
\tilde{\mathbf{V}}\tilde{\mathbf{\Sigma}} = \mathbf Y  (\mathbf X ^\top \mathbf X)^{1/2} \mathbf{W}
\]
almost surely.
	
Finally, using submultiplicativity of the spectral norm and \(\mathbf{W} \in \mathbb O(d)\), we have
\[
\begin{aligned}
\sup_{t\in\mathcal T}\left\|\tilde{\mathbf{V}}(t)\tilde{\mathbf{\Sigma}}-n^{1/2} \mathbf Y(t)\mathbf{W}\right\|_2 
&= \sup_{t\in\mathcal T}\left\|\mathbf Y(t)  (\mathbf X ^\top \mathbf X)^{1/2} \mathbf{W}-n^{1/2} \mathbf Y(t)\mathbf{W}\right\|_2 \\
&\le \sup_{t\in\mathcal T}\left\|\mathbf Y(t)\right\|_2\left\|  (\mathbf X ^\top \mathbf X)^{1/2} -n^{1/2}\mathbf I_d \right\|_2\\
&= O\left(n^{1/2}\log n\right) \quad \text{a.s.}
\end{aligned}
\]
This completes the proof.
\end{proof}

The second result shows that the corresponding empirical singular vectors concentrate around their population analogues at a controlled rate.

\begin{lemma}\label{lem:emp-pert} 
	Assume that
	$T=O((\log n)^{{1/3}})$ as $n\to\infty$. Then, there exists a possibly random
	orthogonal matrix \(\mathbf{W} \in \mathbb O(d)\) such that
	\[
	\sup_{t\in\mathcal T}\left\|
	\mathbf{V}(t)\mathbf{\Sigma}
	-
	\tilde{\mathbf{V}}(t)\tilde{\mathbf{\Sigma}}\mathbf{W}
	\right\|_2
	=
	O\left(
	T^{3/2}n^{1/2}(\log n)^{1/2}
	\right)
	\quad \text{a.s.}
	\]
\end{lemma}

\begin{proof}

To bound the difference in spectral norm, we use the triangle inequality and
the standard norm inequalities:
\[
\begin{aligned}
\left\|
\mathbf V\mathbf\Sigma
-
\tilde{\mathbf V}\tilde{\mathbf\Sigma}\mathbf W
\right\|_2
&\le
\left\|
\mathbf V-\tilde{\mathbf V}\mathbf W
\right\|_2
\left\|\mathbf\Sigma\right\|_2
+
\left\|\tilde{\mathbf V}\right\|_2
\left\|
\mathbf W\mathbf\Sigma
-
\tilde{\mathbf\Sigma}\mathbf W
\right\|_F.
\end{aligned}
\]

By \citet{jones2020multilayer}, there exists a possibly random orthogonal
matrix \(\mathbf W\in\mathbb O(d)\) such that
\[
\begin{aligned}
\left\|
\mathbf V-\tilde{\mathbf V}\mathbf W
\right\|_2
&=
O\left(
Tn^{-1/2}(\log n)^{1/2}
\right)
\quad \text{a.s.},\\
\left\|
\mathbf W\mathbf\Sigma
-
\tilde{\mathbf\Sigma}\mathbf W
\right\|_F
&=
O\left(
T^{5/2}\log n
\right)
\quad \text{a.s.}
\end{aligned}
\]
where the first bound is given in (114)--(115) in the proof of Proposition~19, and the second bound is
given in Proposition~15.

Moreover, by Corollary~10 of \citet{jones2020multilayer},
\[
\left\|\mathbf\Sigma\right\|_2
=
O(T^{1/2}n)
\quad \text{a.s.}
\]
Since \(\tilde{\mathbf V}\) has orthonormal columns,
\[
\left\|\tilde{\mathbf V}\right\|_2=1.
\]
Substituting the above bounds and using \(T=O((\log n)^{1/3})\), we have
\[
\left\|
\mathbf{V}\mathbf{\Sigma}
-
\tilde{\mathbf{V}}\tilde{\mathbf{\Sigma}}\mathbf{W}
\right\|_2
=
O\left(
T^{3/2}n^{1/2}(\log n)^{1/2}
\right)
\quad \text{a.s.}
\]

Since
\[
\left\|
\mathbf{V}(t)\mathbf{\Sigma}
-
\tilde{\mathbf{V}}(t)\tilde{\mathbf{\Sigma}}\mathbf{W}
\right\|_2
\le
\left\|
\mathbf{V}\mathbf{\Sigma}
-
\tilde{\mathbf{V}}\tilde{\mathbf{\Sigma}}\mathbf{W}
\right\|_2,
\]
taking the supremum over \(t\in\mathcal T\) gives the desired bound.
\end{proof}

We now combine these two arguments.
\begin{proof}[Proof of Theorem~\ref{thm:modified-uase-paper}]
	
The first lemma identifies the population right singular vectors with the latent positions $\mathbf Y(t)$ up to an orthogonal transformation. The second lemma controls the deviation between the estimated embedding and its population counterpart. Combining these two steps yields uniform consistency of the modified UASE embeddings across time.

By the preceding lemmas, there exist possibly random orthogonal matrices \(\mathbf{W}_1,\mathbf{W}_2 \in \mathbb O(d)\) such that
\[
\begin{aligned}
    &\sup_{t\in\mathcal T}\left\|\tilde{\mathbf{V}}(t)\tilde{\mathbf{\Sigma}}-n^{1/2}\mathbf{Y}(t)\mathbf{W}_1\right\|_2 = O\left(n^{1/2}\log n\right) \quad \text{a.s.}, \\
    &\sup_{t\in\mathcal T}\left\|\mathbf{V}(t)\mathbf{\Sigma}-\tilde{\mathbf{V}}(t)\tilde{\mathbf{\Sigma}}\mathbf{W}_2\right\|_2 = O\left(T^{3/2}n^{1/2}(\log n)^{1/2}\right) \quad \text{a.s.}
\end{aligned}
\]

Let $\mathbf W =\mathbf W_1\mathbf W_2$, which is again an orthogonal matrix. Then, by the triangle inequality,
\[
\begin{aligned}
&\sup_{t\in\mathcal T}\left\|\hat{\mathbf Y}(t)-\mathbf{Y}(t)\mathbf{W}\right\|_2 \\
&\quad = \frac{1}{n^{1/2}} \sup_{t\in\mathcal T}\left\|\mathbf{V}(t)\mathbf{\Sigma}-n^{1/2}\mathbf{Y}(t)\mathbf{W}\right\|_2 \\
&\quad\le \frac{1}{n^{1/2}}\left( \sup_{t\in\mathcal T}\left\|\tilde{\mathbf{V}}(t)\tilde{\mathbf{\Sigma}}-n^{1/2}\mathbf{Y}(t)\mathbf{W}_1\right\|_2  
+ \sup_{t\in\mathcal T}\left\|\mathbf{V}(t)\mathbf{\Sigma}-\tilde{\mathbf{V}}(t)\tilde{\mathbf{\Sigma}}\mathbf{W}_2\right\|_2\right).
\end{aligned}
\]

Since \(T=O((\log n)^{1/3})\), substituting the above bounds yields the desired bound.
\end{proof}

\section{Statistical Recovery of Geometry and Trajectories}
\label{app:statistical-recovery}

In this section, we establish the statistical recovery of the induced temporal geometry and prove Theorem~\ref{thm:trajectory-consistency-main}.

\subsection{Proof Strategy}

The proof consists of two main steps. First, we establish the recovery of the second-moment operator.

\begin{proposition}[Recovery of the second-moment operator]
\label{prop:M-hatM}
Under the assumptions of Theorem~\ref{thm:trajectory-consistency-main}, there exists a possibly random orthogonal matrix $\mathbf W\in\mathbb O(d)$ such that
\[
\sup_{t,s\in\mathcal T}
\left\|
\hat{\mathbf M}(t,s)
-
\mathbf W^\top\mathbf M(t,s)\mathbf W
\right\|
=
O\!\left(\frac{\log n}{n^{1/2}}\right)
\qquad\text{a.s.}
\]
\end{proposition}

As an intermediate object, we define the finite-sample second-moment matrix corresponding to the latent position matrices by
\[
\tilde{\mathbf M}(t,s)
:=
\frac{1}{n}
\bigl(\mathbf Y(t)-\mathbf Y(s)\bigr)^\top
\bigl(\mathbf Y(t)-\mathbf Y(s)\bigr).
\]
The overall flow of this step is summarized in Figure~\ref{fig:second-moment-recovery}.

Second, we establish error bounds for the induced Euclidean trajectories in terms of the second-moment operators by propagating the error through the induced distances. This step is entirely deterministic, and its overall flow is summarized in Figure~\ref{fig:geometry-recovery}.

For a common orthogonal transformation $\mathbf W\in\mathbb O(d)$, define
\[
\begin{aligned}
E_M
&:=
\sup_{t,s\in\mathcal T}
\left\|
\hat{\mathbf M}(t,s)
-
\mathbf W^\top\mathbf M(t,s)\mathbf W
\right\|_2,\\
E_{u,k}
&:=
\min_{q\in\{-1,1\}}
\left\|
\hat{\bm u}_k-q\mathbf W^\top\bm u_k
\right\|,
\qquad k\in[k_{\mathrm{max}}].
\end{aligned}
\]
For the trace variation geometry, the perturbation is controlled solely by $E_M$, whereas the mode-wise geometry additionally involves the basis error $E_{u,k}$.

\begin{proposition}[Deterministic trajectory perturbation]
\label{prop:M-u-to-psi}
For any $c_{\mathrm{tr}}\le \operatorname{rank}(\mathbf E_{\mathrm{tr}})$
satisfying
\[
\lambda_{c_{\mathrm{tr}}}(\mathbf E_{\mathrm{tr}})
-
\lambda_{c_{\mathrm{tr}}+1}(\mathbf E_{\mathrm{tr}})
\ge \gamma
\]
for some constant $\gamma>0$ independent of $n$, there exists a constant
$C_{\mathrm{tr}}>0$ such that,
as $n\to\infty$ with $T^2 E_M\to0$,
\[
\begin{aligned}
\min_{\mathbf R\in\mathbb O(c_{\mathrm{tr}})}
\left(
\sum_{t=1}^T
\left\|
\hat{\psi}_{\mathrm{tr}}(t)
-
\mathbf R\psi_{\mathrm{tr}}(t)
\right\|^2
\right)^{1/2}
\le
C_{\mathrm{tr}}T^{3/2}E_M
+
o(TE_M).
\end{aligned}
\]

For any $k\in[k_{\mathrm{max}}]$ and
$c_k\le \operatorname{rank}(\mathbf E_k)$ satisfying
\[
\lambda_{c_k}(\mathbf E_k)
-
\lambda_{c_k+1}(\mathbf E_k)
\ge \gamma
\]
for some constant $\gamma>0$ independent of $n$, there exists a constant
$C_k>0$ such that, as
$n\to\infty$ with $T^2(E_M+E_{u,k})\to0$,
\[
\begin{aligned}
&\min_{\mathbf R\in\mathbb O(c_k)}
\left(
\sum_{t=1}^T
\left\|
\hat{\psi}_{k}(t)
-
\mathbf R\psi_{k}(t)
\right\|^2
\right)^{1/2}\\
&\quad\le
C_kT^{3/2}(E_M+E_{u,k})
+
o\left(
T(E_M+E_{u,k})
\right).
\end{aligned}
\]
\end{proposition}

Moreover, since the basis is constructed from the aggregated second-moment operator, the basis error can itself be controlled by the second-moment error.

\begin{proposition}[Perturbation of the orthogonal basis]
\label{prop:M-to-u}
Under Assumption~\ref{assump:separated_eigenvalues}, for every
$k\in[k_{\mathrm{max}}]$,
\[
E_{u,k}
\le
\frac{2^{3/2}|\mathcal S|}{\delta}E_M.
\]
\end{proposition}

The estimated basis is constructed once from the aggregate \(\hat{\mathcal M}\) and is then used for every time-pair comparison. Thus, unlike a direction optimized separately for each pair, the index \(k\) has a consistent meaning throughout the dynamic sequence. Eigenvalue separation ensures that these individual directions are stable, up to their unavoidable sign ambiguity. Combining the second-moment recovery established in Step~1 with the deterministic perturbation bounds established in Step~2 gives the desired control of the Euclidean trajectory estimation error. This completes the proof of Theorem~\ref{thm:trajectory-consistency-main}.

\begin{figure}[t]
\centering
\begin{tikzpicture}[
    node distance=2.0cm,
    >=latex,
    every node/.style={align=center}
]

\node (a)
{\(\mathbf M(t,s)\)};

\node (b) [right=of a]
{\(\tilde{\mathbf M}(t,s)\)};

\node (c) [right=of b]
{\(\hat{\mathbf M}(t,s)\)};

\draw[->] (a) -- node[above]{\scriptsize Lemma~\ref{lem:M-tildeM}} (b);

\draw[->] (b) -- node[above]{\scriptsize Lemma~\ref{lem:tildeM-hatM}} (c);
\draw[->]
    (a.south)
    -- ++(0,-0.5)
    -| node[pos=0.25, below]{\scriptsize Proposition~\ref{prop:M-hatM}}
    (c.south);
    
\end{tikzpicture}
\caption{Recovery of the second-moment operator.}
\label{fig:second-moment-recovery}
\end{figure}

\begin{figure}[t]
\centering
\begin{tikzpicture}[
    node distance=2.5cm,
    >=latex,
    every node/.style={align=center}
]

\node (a)
{\(\mathbf M(t,s)\)};

\node (b) [right=of a]
{\(d_{\mathrm{tr}}(t,s)\)};

\node (c) [right=of b]
{\(\psi_{\mathrm{tr}}(t)\)};

\node (d) [below=of a]
{\(\mathbf M(t,s),\)};

\node (u) [right=0cm of d]
{\(\bm u_k\)};

\node (e) [below=of b]
{\(d_k(t,s)\)};

\node (f) [below=of c]
{\(\psi_k(t)\)};

\node[
    draw,
    rounded corners,
    fit=(d)(u),
    inner sep=0.25cm
] (box) {};

\draw[->,dashed]
    (d.north)
    -- ++(0,0.45)
    -| (u.north)
    node[pos=0.2,above]{\scriptsize Proposition~\ref{prop:M-to-u}};

\draw[->,blue]
    (a) -- node[above]{\scriptsize Lemma~\ref{lem:M-u-to-D}} (b);

\draw[->,green!50!black]
    (b) -- node[above]{\scriptsize Lemma~\ref{lem:D-to-psi}} (c);

\draw[->,red]
    (a.south)
    -- ++(0,-0.8)
    -| node[pos=0.25,below]{\scriptsize Proposition~\ref{prop:M-u-to-psi}}
    (c.south);

\draw[->,blue]
    (box.east)
    -- node[above]{\scriptsize Lemma~\ref{lem:M-u-to-D}} (e.west);

\draw[->,green!50!black]
    (e) -- node[above]{\scriptsize Lemma~\ref{lem:D-to-psi}} (f);

\draw[->,red]
    (box.south)
    -- ++(0,-0.8)
    -| node[pos=0.25,below]{\scriptsize Proposition~\ref{prop:M-u-to-psi}}
    (f.south);

\end{tikzpicture}

\caption{Deterministic propagation of second-moment perturbations to Euclidean trajectories.}
\label{fig:geometry-recovery}
\end{figure}

\subsection{Step 1: Recovery of the Second-Moment Operator}

We first prove Proposition~\ref{prop:M-hatM} by decomposing the estimation error into two components. Lemma~\ref{lem:M-tildeM} controls the discrepancy between the population $\mathbf M(t,s)$ and the finite-sample $\tilde{\mathbf M}(t,s)$, while Lemma~\ref{lem:tildeM-hatM} controls the discrepancy between $\tilde{\mathbf M}(t,s)$ and the estimated $\hat{\mathbf M}(t,s)$, up to a common orthogonal transformation.

\paragraph{From the population to the finite-sample second-moment matrix.}
The first lemma establishes the finite-sample concentration of the second-moment matrix. Since the rows of $\mathbf Y(t)$ correspond to independent samples from the random vectors $\phi(t)$, $\tilde{\mathbf M}(t,s)$ can be viewed as a finite-sample analogue of the second-moment operator $\mathbf M(t,s)$. Its concentration then follows from the matrix Bernstein inequality.

\begin{lemma}[Latent second-moment concentration]
\label{lem:M-tildeM}
Under the assumptions of Theorem~\ref{thm:trajectory-consistency-main},
\[
\sup_{t,s\in\mathcal T}
\left\|
\tilde{\mathbf M}(t,s)-\mathbf M(t,s)
\right\|
=
O\!\left(\frac{(\log n)^{1/2}}{n^{1/2}}\right)
\qquad\text{a.s.}
\]
\end{lemma}

\begin{proof}
Define
\[
\mathbf Z_i(t,s)
:=
\frac{1}{n}
\Bigl(
\bigl(\mathbf Y_{i:}(t)-\mathbf Y_{i:}(s)\bigr)^\top
\bigl(\mathbf Y_{i:}(t)-\mathbf Y_{i:}(s)\bigr)
-
\mathbf M(t,s)
\Bigr).
\]
Then
\[
\begin{aligned}
\tilde{\mathbf M}(t,s)-\mathbf M(t,s)
&=
\Bigl(\frac{1}{n}
\sum_{i=1}^n
\bigl(\mathbf Y_{i:}(t)-\mathbf Y_{i:}(s)\bigr)^\top
\bigl(\mathbf Y_{i:}(t)-\mathbf Y_{i:}(s)\bigr)
\Bigr)-
\mathbf M(t,s)
\\
&=
\sum_{i=1}^n \mathbf Z_i(t,s).
\end{aligned}
\]
Moreover, since the rows $\mathbf Y_{i:}(t)$ are i.i.d. copies of
$\phi(t)$,
\[
\begin{aligned}
\mathbb E\left[
\bigl(\mathbf Y_{i:}(t)-\mathbf Y_{i:}(s)\bigr)^\top
\bigl(\mathbf Y_{i:}(t)-\mathbf Y_{i:}(s)\bigr)
\right]
&=
\mathbb E\left[
\bigl(\phi(t)-\phi(s)\bigr)
\bigl(\phi(t)-\phi(s)\bigr)^\top
\right]\\
&=
\mathbf M(t,s).
\end{aligned}
\]
Hence, $\{\mathbf Z_i(t,s)\}_{i=1}^n$ are independent, mean-zero,
symmetric random matrices.

Since $\phi(t)$ has bounded support with
$\|\phi(t)\|_2\le \beta$ almost surely for all $t\in\mathcal T$, we have
\[
\max_{1\le i\le n}\|\mathbf Z_i(t,s)\|_2
\le \frac{8\beta^2}{n}
\qquad\text{a.s.}
\]
Moreover, the matrix variance statistic satisfies
\[
\begin{aligned}
&\max\left\{
\left\|
\mathbb E\left[
\left(\tilde{\mathbf M}(t,s)-\mathbf M(t,s)\right)
\left(\tilde{\mathbf M}(t,s)-\mathbf M(t,s)\right)^\top
\right]
\right\|_2,
\right.\\
&\hspace{6em}\left.
\left\|
\mathbb E\left[
\left(\tilde{\mathbf M}(t,s)-\mathbf M(t,s)\right)^\top
\left(\tilde{\mathbf M}(t,s)-\mathbf M(t,s)\right)
\right]
\right\|_2
\right\}\\
&\qquad=
n\left\|
\mathbb E\left[\mathbf Z_1(t,s)^2\right]
\right\|_2
\le \frac{64\beta^4}{n}.
\end{aligned}
\]
Therefore, by the Matrix Bernstein inequality,
\[
\mathbb P\left(
\left\|
\tilde{\mathbf M}(t,s)-\mathbf M(t,s)
\right\|_2
\ge \epsilon
\right)
\le
2d\exp\left\{
-\frac{n\epsilon^2}
{128\beta^4+\frac{16}{3}\beta^2\epsilon}
\right\}.
\]
Applying a union bound over all \(T^2\) pairs \((t,s)\in\mathcal T^2\) and using \(T=O((\log n)^{1/3})\), for any \(A>0\), there exist a constant \(C_A>0\) and an integer
\(N_A\) such that, for all \(n\ge N_A\),
\[
\mathbb P\left(
\sup_{t,s\in\mathcal T}
\left\|
\tilde{\mathbf M}(t,s)-\mathbf M(t,s)
\right\|_2
\ge
C_A\frac{(\log n)^{1/2}}{n^{1/2}}
\right)
\le n^{-A}.
\]
\end{proof}

\paragraph{From the finite-sample to the estimated second-moment matrix.}
The second lemma controls the error arising from replacing the latent positions by the modified UASE embeddings. The common orthogonal alignment in Theorem~\ref{thm:modified-uase-paper} ensures that the resulting second-moment matrices are aligned by a single orthogonal transformation across all snapshots.

\begin{lemma}[Estimated second-moment approximation]
\label{lem:tildeM-hatM}
Under the assumptions of Theorem~\ref{thm:trajectory-consistency-main}, there exists a possibly random orthogonal matrix $\mathbf W\in\mathbb O(d)$ such that
\[
\sup_{t,s\in\mathcal T}
\left\|
\hat{\mathbf M}(t,s)
-
\mathbf W^\top\tilde{\mathbf M}(t,s)\mathbf W
\right\|
=
O\!\left(\frac{\log n}{n^{1/2}}\right)
\qquad\text{a.s.}
\]
\end{lemma}

\begin{proof}
Define
\[
\tilde{\Delta}(t,s):=\mathbf Y(t)-\mathbf Y(s),
\qquad
\hat{\Delta}(t,s):=\hat{\mathbf Y}(t)-\hat{\mathbf Y}(s).
\]
Then, by Theorem~\ref{thm:modified-uase-paper},
\[
\begin{aligned}
\sup_{t,s\in\mathcal T}
\left\|
\hat{\Delta}(t,s)-\tilde{\Delta}(t,s)\mathbf W
\right\|_2
&\le
2\sup_{t\in\mathcal T}
\left\|
\hat{\mathbf Y}(t)-\mathbf Y(t)\mathbf W
\right\|_2\\
&=
O(\log n)
\qquad\text{a.s.}
\end{aligned}
\]
Moreover, since $\mathcal Y(t)$ is uniformly bounded by $\beta$, we have
\[
\sup_{t,s\in\mathcal T}
\|\tilde{\Delta}(t,s)\|_2
\le
2\beta n^{1/2}
\qquad\text{a.s.}
\]
Therefore,
\[
\begin{aligned}
&\hat{\mathbf M}(t,s)
-\mathbf W^\top\tilde{\mathbf M}(t,s)\mathbf W\\
&=
\frac{1}{n}
\Big[
\mathbf W^\top\tilde{\Delta}(t,s)^\top
\{\hat{\Delta}(t,s)-\tilde{\Delta}(t,s)\mathbf W\}\\
&\qquad+
\{\hat{\Delta}(t,s)-\tilde{\Delta}(t,s)\mathbf W\}^\top
\tilde{\Delta}(t,s)\mathbf W\\
&\qquad+
\{\hat{\Delta}(t,s)-\tilde{\Delta}(t,s)\mathbf W\}^\top
\{\hat{\Delta}(t,s)-\tilde{\Delta}(t,s)\mathbf W\}
\Big],
\end{aligned}
\]
which gives
\[
\sup_{t,s\in\mathcal T}
\left\|
\hat{\mathbf M}(t,s)
-\mathbf W^\top\tilde{\mathbf M}(t,s)\mathbf W
\right\|_2
=
O\left(\frac{\log n}{n^{1/2}}\right)
\qquad\text{a.s.}
\]
\end{proof}

Combining Lemmas~\ref{lem:M-tildeM} and~\ref{lem:tildeM-hatM} with the triangle inequality yields Proposition~\ref{prop:M-hatM}.

\subsection{Step 2: Deterministic Perturbation of the Induced Geometry}

We next establish deterministic perturbation bounds for the induced Euclidean trajectories in Proposition~\ref{prop:M-u-to-psi}. The argument propagates perturbations in the second-moment operators through the induced distances and then to the corresponding Euclidean trajectories. Specifically, Lemma~\ref{lem:M-u-to-D} controls the perturbation of the induced distances in terms of $E_M$ and $E_{u,k}$, while Lemma~\ref{lem:D-to-psi} propagates these distance perturbations to the corresponding Euclidean trajectories.

\paragraph{From second-moment operators and bases to induced distances.}

We first control the perturbation of the induced distances. The trace variation distance depends only on the second-moment operator, whereas the $k$th mode-wise distance additionally depends on the corresponding orthogonal basis vector.

\begin{proposition}[Perturbation of induced distances]
\label{lem:M-u-to-D}
For a common orthogonal transformation $\mathbf W\in\mathbb O(d)$,
\[
\sup_{t,s\in\mathcal T}
\left|
\hat d_{\mathrm{tr}}(t,s)^2-d_{\mathrm{tr}}(t,s)^2
\right|
\le d E_M,
\]
and, for every $k\in[k_{\mathrm{max}}]$,
\[
\sup_{t,s\in\mathcal T}
\left|
\hat d_k(t,s)^2-d_k(t,s)^2
\right|
\le
E_M
+
2E_{u,k}
\sup_{t,s\in\mathcal T}\|\mathbf M(t,s)\|.
\]
\end{proposition}

\begin{proof}
For the trace variation distance, using the invariance of the trace under
orthogonal transformations,
\[
\begin{aligned}
\left|
\hat d_{\mathrm{tr}}(t,s)^2-d_{\mathrm{tr}}(t,s)^2
\right|
&=
\left|
\operatorname{tr}\hat{\mathbf M}(t,s)
-
\operatorname{tr}\mathbf M(t,s)
\right|\\
&=
\left|
\operatorname{tr}
\left\{
\hat{\mathbf M}(t,s)
-
\mathbf W^\top\mathbf M(t,s)\mathbf W
\right\}
\right|_2\\
&\le
d
\left\|
\hat{\mathbf M}(t,s)
-
\mathbf W^\top\mathbf M(t,s)\mathbf W
\right\|.
\end{aligned}
\]
Taking the supremum over $t,s\in\mathcal T$ gives the first claim.

For the mode-wise variation distance, let
$q\in\{-1,1\}$ attain the minimum in the definition of $E_{u,k}$.
Since
\[
(q\mathbf W^\top\bm u_k)^\top
\mathbf W^\top\mathbf M(t,s)\mathbf W
(q\mathbf W^\top\bm u_k)
=
(\mathbf W^\top\bm u_k)^\top
\mathbf W^\top\mathbf M(t,s)\mathbf W
(\mathbf W^\top\bm u_k),
\]
we have
\[
\begin{aligned}
&\left|
\hat d_k(t,s)^2-d_k(t,s)^2
\right|\\
&=
\left|
\hat{\bm u}_k^\top
\hat{\mathbf M}(t,s)\hat{\bm u}_k
-
(\mathbf W^\top\bm u_k)^\top
\mathbf W^\top\mathbf M(t,s)\mathbf W
(\mathbf W^\top\bm u_k)
\right|\\
&\le
\left|
\hat{\bm u}_k^\top
\left\{
\hat{\mathbf M}(t,s)
-
\mathbf W^\top\mathbf M(t,s)\mathbf W
\right\}
\hat{\bm u}_k
\right|\\
&\quad+
\left|
\hat{\bm u}_k^\top
\mathbf W^\top\mathbf M(t,s)\mathbf W
\hat{\bm u}_k
-
(q\mathbf W^\top\bm u_k)^\top
\mathbf W^\top\mathbf M(t,s)\mathbf W
(q\mathbf W^\top\bm u_k)
\right|\\
&\le
E_M
+
2
\left\|
\mathbf W^\top\mathbf M(t,s)\mathbf W
\right\|_2
\left\|
\hat{\bm u}_k-q\mathbf W^\top\bm u_k
\right\|\\
&=
E_M
+
2E_{u,k}\|\mathbf M(t,s)\|_2.
\end{aligned}
\]
Taking the supremum over $t,s\in\mathcal T$ completes the proof.
\end{proof}

\paragraph{From induced distances to Euclidean trajectories.}

We next propagate perturbations in the induced distance matrices to the corresponding Euclidean trajectories. This step follows from perturbation bounds for the eigenstructure of the associated centered squared-distance matrices.

\begin{lemma}[Perturbation of Euclidean trajectories]
\label{lem:D-to-psi}
For any $\rho\in\{\mathrm{tr},1,\ldots,k_{\mathrm{max}}\}$ and
$c_\rho\le \operatorname{rank}(\mathbf E_\rho)$ satisfying
\(
\lambda_{c_\rho}(\mathbf E_\rho)
-
\lambda_{c_\rho+1}(\mathbf E_\rho)
\ge \gamma
\)
for some constant $\gamma>0$ independent of $n$, as $n\to\infty$ with $\nu_\rho, \lambda_1(\mathbf E_\rho)\nu_\rho\to0$,
\[
\begin{aligned}
&\min_{\mathbf R\in\mathbb O(c_\rho)}
\left(
\sum_{t=1}^T
\left\|
\hat{\psi}_\rho(t)
-
\mathbf R\psi_\rho(t)
\right\|^2
\right)^{1/2}
\le
\frac{2\lambda_1(\mathbf E_\rho)^{1/2}}
{
\lambda_{c_\rho}(\mathbf E_\rho)
-\lambda_{c_\rho+1}(\mathbf E_\rho)
}
\nu_\rho
+
o(\nu_\rho),
\end{aligned}
\]
where
\[
\nu_\rho
:=
\left\|
\mathbf D_\rho^{(2)}
-
\hat{\mathbf D}_\rho^{(2)}
\right\|_F.
\]
\end{lemma}

\begin{proof}
Let $\mathbf Q_\rho,\hat{\mathbf Q}_\rho\in\mathbb R^{T\times c_\rho}$ denote the matrices of the top $c_\rho$ eigenvectors of $\mathbf E_\rho$ and $\hat{\mathbf E}_\rho$, respectively, and let $\mathbf S_\rho$ and $\hat{\mathbf S}_\rho$ denote the corresponding diagonal matrices of eigenvalues. Thus, $\psi_\rho(t)$ and $\hat\psi_\rho(t)$ are the $t$th rows of
\[
\mathbf Q_\rho\mathbf S_\rho^{1/2},
\qquad
\hat{\mathbf Q}_\rho\hat{\mathbf S}_\rho^{1/2},
\]
respectively.

Let
\[
\mathbf Q_\rho^\top\hat{\mathbf Q}_\rho
=
\mathbf W_1\mathbf C\mathbf W_2^\top
\]
be a singular value decomposition and set
\[
\mathbf R_\rho
=
\mathbf W_1\mathbf W_2^\top
\in\mathbb O(c_\rho).
\]
Then, we have
\[
\begin{aligned}
&\min_{\mathbf R\in\mathbb O(c_\rho)}
\left(
\sum_{t=1}^T
\left\|
\hat\psi_\rho(t)
-
\mathbf R\psi_\rho(t)
\right\|^2
\right)^{1/2}\\
&\qquad\le
\left\|
\hat{\mathbf Q}_\rho\hat{\mathbf S}_\rho^{1/2}
-
\mathbf Q_\rho\mathbf S_\rho^{1/2}\mathbf R_\rho
\right\|_F\\
&\qquad\le
\left(
\left\|
\hat{\mathbf Q}_\rho\hat{\mathbf Q}_\rho^\top
-
\mathbf Q_\rho\mathbf Q_\rho^\top
\right\|_F
+
\left\|
\mathbf Q_\rho^\top\hat{\mathbf Q}_\rho
-
\mathbf R_\rho
\right\|_F
\right)
\|\hat{\mathbf S}_\rho\|_2^{1/2}
+
\left\|
\mathbf R_\rho\hat{\mathbf S}_\rho^{1/2}
-
\mathbf S_\rho^{1/2}\mathbf R_\rho
\right\|_F.
\end{aligned}
\]
We evaluate each term.

First, we have
\[
\left\|
\mathbf E_\rho-\hat{\mathbf E}_\rho
\right\|_F
=
\left\|
-\frac12\mathbf J\mathbf D_\rho^{(2)}\mathbf J
+\frac12\mathbf J\hat{\mathbf D}_\rho^{(2)}\mathbf J
\right\|_F
\le
\frac12\nu_\rho.
\]

By the Davis--Kahan theorem,
\[
\left\|
\sin\Theta
\left(
\mathbf Q_\rho,\hat{\mathbf Q}_\rho
\right)
\right\|_F
\le
\frac{2\|\mathbf E_\rho-\hat{\mathbf E}_\rho\|_F}
{\lambda_{c_\rho}(\mathbf E_\rho)
-
\lambda_{c_\rho+1}(\mathbf E_\rho)}
\le
\frac{\nu_\rho}
{\lambda_{c_\rho}(\mathbf E_\rho)
-
\lambda_{c_\rho+1}(\mathbf E_\rho)}.
\]

Writing $\sigma_i$ for the singular values and
$\theta_i:=\cos^{-1}(\sigma_i)$ for the principal angles, we obtain
\[
\begin{aligned}
\left\|
\mathbf Q_\rho^\top\hat{\mathbf Q}_\rho
-
\mathbf R_\rho
\right\|_F
&=
\left\|
\mathbf C-\mathbf I_{c_\rho}
\right\|_F
=
\left(
\sum_{i=1}^{c_\rho}
(1-\cos\theta_i)^2
\right)^{1/2}\\
&\le
\sum_{i=1}^{c_\rho}
\sin^2\theta_i
=
\left\|
\sin\Theta
\left(
\mathbf Q_\rho,\hat{\mathbf Q}_\rho
\right)
\right\|_F^2\\
&\le
\frac{\nu_\rho^2}
{
\left\{
\lambda_{c_\rho}(\mathbf E_\rho)
-
\lambda_{c_\rho+1}(\mathbf E_\rho)
\right\}^2
}.
\end{aligned}
\]
and
\[
\begin{aligned}
\left\|
\mathbf Q_\rho\mathbf Q_\rho^\top
-
\hat{\mathbf Q}_\rho\hat{\mathbf Q}_\rho^\top
\right\|_F
&=
\sqrt{2}
\left\|
\sin\Theta
\left(
\mathbf Q_\rho,\hat{\mathbf Q}_\rho
\right)
\right\|_F\\
&\le
\frac{\sqrt{2}\nu_\rho}
{\lambda_{c_\rho}(\mathbf E_\rho)
-
\lambda_{c_\rho+1}(\mathbf E_\rho)}.
\end{aligned}
\]

Moreover, Weyl's inequality gives
\[
\|\hat{\mathbf S}_\rho\|_2
\le
\lambda_1(\mathbf E_\rho)+\frac{\nu_\rho}{2},
\]
and, since $\|\mathbf S_\rho\|_2=\lambda_1(\mathbf E_\rho)>0$,
\[
\|\hat{\mathbf S}_\rho\|_2^{1/2}
\le
\|\mathbf S_\rho\|_2^{1/2}
+
\frac{
\left|
\|\hat{\mathbf S}_\rho\|_2-\|\mathbf S_\rho\|_2
\right|
}{
2\|\mathbf S_\rho\|_2^{1/2}
}
\le
\lambda_1(\mathbf E_\rho)^{1/2}
+
\frac{\nu_\rho}
{4\lambda_1(\mathbf E_\rho)^{1/2}}.
\]

Furthermore,
\[
\begin{aligned}
\left\|
\mathbf R_\rho\hat{\mathbf S}_\rho
-
\mathbf S_\rho\mathbf R_\rho
\right\|_F
&\le
\left\|
\mathbf Q_\rho^\top\hat{\mathbf Q}_\rho
-
\mathbf R_\rho
\right\|_F
\left(
\|\hat{\mathbf S}_\rho\|_2
+
\|\mathbf S_\rho\|_2
\right)
+
\left\|
\mathbf E_\rho-\hat{\mathbf E}_\rho
\right\|_F\\
&\le
\frac{\nu_\rho^2}
{
\left\{
\lambda_{c_\rho}(\mathbf E_\rho)
-
\lambda_{c_\rho+1}(\mathbf E_\rho)
\right\}^2
}
\left(
2\lambda_1(\mathbf E_\rho)+\frac{\nu_\rho}{2}
\right)
+
\frac{\nu_\rho}{2}\\
&\le
\frac{\nu_\rho}{2}\left(1+\frac{
4\lambda_1(\mathbf E_\rho)\nu_\rho+\nu_\rho^2}
{
\gamma^2
}
\right)
\\
&=
\frac{\nu_\rho}{2}
+
o(\nu_\rho).
\end{aligned}
\]

Since $\lambda_{c_\rho}(\mathbf E_\rho)>0$,
\[
\begin{aligned}
\left\|
\mathbf R_\rho\hat{\mathbf S}_\rho^{1/2}
-
\mathbf S_\rho^{1/2}\mathbf R_\rho
\right\|_F
&\le
\frac{
\left\|
\mathbf R_\rho\hat{\mathbf S}_\rho
-
\mathbf S_\rho\mathbf R_\rho
\right\|_F
}
{\lambda_{c_\rho}(\mathbf E_\rho)^{1/2}}\\
&=
\frac{\nu_\rho}
{2\lambda_{c_\rho}(\mathbf E_\rho)^{1/2}}
+
o(\nu_\rho).
\end{aligned}
\]

Finally, since $\lambda_{c_\rho+1}(\mathbf E_\rho)\ge0$,
\[
\begin{aligned}
&\min_{\mathbf R\in\mathbb O(c_\rho)}
\left(
\sum_{t=1}^T
\left\|
\hat\psi_\rho(t)
-
\mathbf R\psi_\rho(t)
\right\|^2
\right)^{1/2}\\
&\qquad\le
\frac{\sqrt{2}\lambda_1(\mathbf E_\rho)^{1/2}}
{\lambda_{c_\rho}(\mathbf E_\rho)
-
\lambda_{c_\rho+1}(\mathbf E_\rho)}
\nu_\rho
+
\frac{\nu_\rho}
{2\lambda_{c_\rho}(\mathbf E_\rho)^{1/2}}
+
o(\nu_\rho)\\
&\qquad\le
\frac{
2\lambda_1(\mathbf E_\rho)^{1/2}
}{
\lambda_{c_\rho}(\mathbf E_\rho)
-
\lambda_{c_\rho+1}(\mathbf E_\rho)
}
\nu_\rho
+
o(\nu_\rho).
\end{aligned}
\]
\end{proof}

The preceding two lemmas yield Proposition~\ref{prop:M-u-to-psi}, which gives the desired deterministic perturbation bound for the Euclidean trajectories in terms of $E_M$ and, for the mode-wise geometry, $E_{u,k}$.

\begin{proof}
By Proposition~\ref{lem:M-u-to-D},
\[
\begin{aligned}
\nu_{\mathrm{tr}}
&=
\left\|
\mathbf D_{\mathrm{tr}}^{(2)}
-
\hat{\mathbf D}_{\mathrm{tr}}^{(2)}
\right\|_F\\
&\le
T\sup_{t,s\in\mathcal T}
\left|
\hat d_{\mathrm{tr}}(t,s)^2-d_{\mathrm{tr}}(t,s)^2
\right|
\le
dTE_M
\le 
dT^2E_M.
\end{aligned}
\]
Hence, since $d$ is fixed, $T^2E_M\to0$ implies
$\nu_{\mathrm{tr}}\to0$.
Moreover, since
$d_{\mathrm{tr}}(\phi(t),\phi(s))^2\le4\beta^2$,
\[
\lambda_1(\mathbf E_{\mathrm{tr}})
\le
\|\mathbf E_{\mathrm{tr}}\|_2
\le
\frac12
\left\|
\mathbf D_{\mathrm{tr}}^{(2)}
\right\|_2
\le
2\beta^2T.
\]
Hence, $T^2E_M\to0$ implies
$\lambda_1(\mathbf E_{\mathrm{tr}})\nu_{\mathrm{tr}}\to0$.
Applying Lemma~\ref{lem:D-to-psi} and using
\[
\lambda_{c_{\mathrm{tr}}}(\mathbf E_{\mathrm{tr}})
-
\lambda_{c_{\mathrm{tr}}+1}(\mathbf E_{\mathrm{tr}})
\ge\gamma,
\]
we obtain
\[
\begin{aligned}
\min_{\mathbf R\in\mathbb O(c_{\mathrm{tr}})}
\left(
\sum_{t=1}^T
\left\|
\hat{\psi}_{\mathrm{tr}}(t)
-
\mathbf R\psi_{\mathrm{tr}}(t)
\right\|^2
\right)^{1/2}
&\le
\frac{
2\sqrt{2}\beta d
}{
\gamma
}
T^{3/2}E_M
+
o(TE_M).
\end{aligned}
\]
Thus, the claimed bound holds with
\[
C_{\mathrm{tr}}
=
\frac{2\sqrt{2}\beta d}{\gamma}.
\]

Similarly, for any $k\in[k_{\mathrm{max}}]$, Proposition~\ref{lem:M-u-to-D} yields
\[
\begin{aligned}
\nu_k
&=
\left\|
\mathbf D_k^{(2)}
-
\hat{\mathbf D}_k^{(2)}
\right\|_F\\
&\le
T\sup_{t,s\in\mathcal T}
\left|
\hat d_k(t,s)^2-d_k(t,s)^2
\right|\\
&\le
\max\{1,8\beta^2\}
T(E_M+E_{u,k})
\le
\max\{1,8\beta^2\}
T^2(E_M+E_{u,k}),
\end{aligned}
\]
where we used the bounded-support assumption
$\|\phi(t)\|_2\le\beta$ almost surely for all $t\in\mathcal T$, which implies
\[
\sup_{t,s\in\mathcal T}
\|\mathbf M(t,s)\|
\le
4\beta^2.
\]
Hence,
$T^2(E_M+E_{u,k})\to0$ implies $\nu_k\to0$.
Moreover, since
\[
d_k(\phi(t),\phi(s))^2
\le
d_{\mathrm{tr}}(\phi(t),\phi(s))^2
\le
4\beta^2,
\]
we similarly have
\[
\lambda_1(\mathbf E_k)
\le
2\beta^2T.
\]
Hence, $T^2(E_M+E_{u,k})\to0$ implies
$\lambda_1(\mathbf E_k)\nu_k\to0$.
Applying Lemma~\ref{lem:D-to-psi} again and using
\[
\lambda_{c_k}(\mathbf E_k)
-
\lambda_{c_k+1}(\mathbf E_k)
\ge\gamma,
\]
we obtain
\[
\begin{aligned}
\min_{\mathbf R\in\mathbb O(c_k)}
\left(
\sum_{t=1}^T
\left\|
\hat{\psi}_{k}(t)
-
\mathbf R\psi_{k}(t)
\right\|^2
\right)^{1/2}
&\le
\frac{
2\sqrt{2}\beta\max\{1,8\beta^2\}
}{
\gamma
}
T^{3/2}(E_M+E_{u,k})\\
&\qquad+
o\bigl(T(E_M+E_{u,k})\bigr).
\end{aligned}
\]
Thus, the claimed bound holds with
\[
C_k
=
\frac{
2\sqrt{2}\beta\max\{1,8\beta^2\}
}{
\gamma
}.
\]
\end{proof}

\paragraph{Control of the mode-wise basis error.}

Proposition~\ref{prop:M-to-u} controls the basis error $E_{u,k}$ in terms of $E_M$. The result follows by controlling the perturbation of the population and estimated basis vectors constructed from the aggregated second-moment operators under the corresponding eigengap condition.

\begin{proof}
By the definition of $\hat{\mathcal M}$ and $\mathcal M$,
\[
\begin{aligned}
\left\|
\hat{\mathcal M}
-
\mathbf W^\top\mathcal M\mathbf W
\right\|_2
&\le
\sum_{(t,s)\in\mathcal S}
\left\|
\hat{\mathbf M}(t,s)
-
\mathbf W^\top\mathbf M(t,s)\mathbf W
\right\|_2\\
&\le
|\mathcal S|E_M.
\end{aligned}
\]
Moreover, $\mathbf W^\top\bm u_k$ is the $k$th eigenvector of
$\mathbf W^\top\mathcal M\mathbf W$. Hence, by the Davis--Kahan theorem,
\[
\min_{q\in\{-1,1\}}
\left\|
\hat{\bm u}_k-q\mathbf W^\top\bm u_k
\right\|
\le
\frac{
2^{3/2}
\left\|
\hat{\mathcal M}
-
\mathbf W^\top\mathcal M\mathbf W
\right\|_2
}{
\min\left\{
\lambda_{k-1}(\mathcal M)-\lambda_k(\mathcal M),
\lambda_k(\mathcal M)-\lambda_{k+1}(\mathcal M)
\right\}
},
\]
where we define \(\lambda_0(\mathcal M)=+\infty\).
By Assumption~\ref{assump:separated_eigenvalues}, we have
\[
E_{u,k}
\le
\frac{2^{3/2}|\mathcal S|}{\delta}E_M.
\]
\end{proof}

\subsection{Proof of Theorem~\ref{thm:trajectory-consistency-main}}

We now combine the second-moment recovery established in Step~1 with the deterministic perturbation results from Step~2 to prove Theorem~\ref{thm:trajectory-consistency-main}.

\begin{proof}
By Proposition~\ref{prop:M-hatM},
\[
E_M
=
O\!\left(\frac{\log n}{n^{1/2}}\right)
\qquad\text{a.s.},
\]
where $\mathbf W$ is chosen at each realization as the random orthogonal
matrix in Proposition~\ref{prop:M-hatM}. Hence, since
$T=O((\log n)^{1/3})$,
\[
T^2E_M\to0
\qquad\text{a.s.}
\]

By Proposition~\ref{prop:M-to-u}, since
$|\mathcal S|/\delta=O(1)$,
\[
E_{u,k}
\le
\frac{2^{3/2}|\mathcal S|}{\delta}E_M
=
O(E_M).
\]
Consequently,
\[
T^2(E_M+E_{u,k})\to0
\qquad\text{a.s.}
\]

For each $\rho\in\{\mathrm{tr},1,\ldots,k_{\mathrm{max}}\}$, define $\mathbf R_\rho\in\mathbb O(c_\rho)$ to be a random orthogonal matrix attaining
the minimum in Proposition~\ref{prop:M-u-to-psi}.

Then, Proposition~\ref{prop:M-u-to-psi} gives the desired bound,
since $C_\rho$ is independent of $n$.
\end{proof}

\section{Node Attribution and CMDS Truncation}
\label{app:node-attribution}

We now prove the relationship between node attributions and temporal trajectory
displacements stated in Theorem~\ref{thm:local-global-main}. The key observation
is that squared node attributions sum exactly to the corresponding estimated
squared distance:
\[
\hat d_{\rho}(t,s)^2
=
\sum_{i=1}^n \hat a_{\rho}(t,s;i)^2,
\qquad
\rho\in\{\mathrm{tr},1,\ldots,d\}.
\]
The only difference between the corresponding estimated squared distance and
a squared CMDS trajectory displacement comes from dimensions omitted by CMDS.

Moreover, since \(\{\hat{\bm u}_k\}_{k=1}^d\) is an orthonormal basis, the same orthogonal decomposition as in Theorem~\ref{thm:tv-decomposition-paper}
gives, for each node,
\[
\hat a_{\mathrm{tr}}(t,s;i)^2
=
\sum_{k=1}^d \hat a_k(t,s;i)^2.
\]
The mode-wise attribution \(\hat a_k(t,s;i)\) retains a sign relative to the chosen orientation of \(\hat{\bm u}_k\), while its squared contribution and all identities above are invariant to reversing that orientation.

We next give the formal proof.

\begin{proof}[Proof of Theorem~\ref{thm:local-global-main}]
We first establish the relationship between the estimated distances and the
node attributions.
By the definitions of the node attributions and estimated distances,
\[
\hat d_{\mathrm{tr}}(t,s)^2
=\frac{1}{n}\|\hat{\mathbf Y}(t)-\hat{\mathbf Y}(s)\|_F^2 
=\sum_{i=1}^n\frac1n\|\hat{\mathbf Y}_{i:}(t)-\hat{\mathbf Y}_{i:}(s)\|^2 
= \sum_{i=1}^n\hat a_{\mathrm{tr}}(t,s;i)^2,
\]
and for every \(k\in\{1,\ldots,d\}\), 
\[
\hat d_k(t,s)^2 
= \frac{1}{n}\|(\hat{\mathbf Y}(t)-\hat{\mathbf Y}(s))\hat{\bm u}_k\|^2 
=\sum_{i=1}^n\frac1n\langle\hat{\mathbf Y}_{i:}(t)-\hat{\mathbf Y}_{i:}(s),\hat{\bm u}_k\rangle^2 
=\sum_{i=1}^n\hat a_k(t,s;i)^2.
\]

It remains to compare the squared displacement in the CMDS trajectory with the
corresponding estimated distance.
Let
\[
\hat{\mathbf Z}_{\rho}
:=
\bigl(
\hat\psi_{\rho}(1),\ldots,\hat\psi_{\rho}(T)
\bigr)^\top
\in\mathbb R^{T\times c_\rho},
\]
and denote by $\mathbf e_t\in\mathbb R^T$ the $t$th standard basis vector.
Since $\hat{\mathbf E}_\rho$ is the centered Gram matrix associated with $\hat d_\rho$,
\[
\hat d_\rho(t,s)^2
=
(\mathbf e_t-\mathbf e_s)^\top
\hat{\mathbf E}_\rho
(\mathbf e_t-\mathbf e_s),
\]
while
\[
\|\hat\psi_\rho(t)-\hat\psi_\rho(s)\|^2
=
(\mathbf e_t-\mathbf e_s)^\top
\hat{\mathbf Z}_\rho\hat{\mathbf Z}_\rho^\top
(\mathbf e_t-\mathbf e_s).
\]
Therefore,
\[
\begin{aligned}
\left|
\|\hat\psi_\rho(t)-\hat\psi_\rho(s)\|^2
-
\hat d_\rho(t,s)^2
\right|
&=
\left|
(\mathbf e_t-\mathbf e_s)^\top
\left(
\hat{\mathbf Z}_\rho\hat{\mathbf Z}_\rho^\top
-
\hat{\mathbf E}_\rho
\right)
(\mathbf e_t-\mathbf e_s)
\right|\\
&\le
\|\mathbf e_t-\mathbf e_s\|^2
\left\|
\hat{\mathbf Z}_\rho\hat{\mathbf Z}_\rho^\top
-
\hat{\mathbf E}_\rho
\right\|_2\\
&\le
2
\left\|
\hat{\mathbf Z}_\rho\hat{\mathbf Z}_\rho^\top
-
\hat{\mathbf E}_\rho
\right\|_F\\
&\le
2
\left(
\sum_{i=c_\rho+1}^T
\lambda_i(\hat{\mathbf E}_{\rho})^2
\right)^{1/2},
\end{aligned}
\]
where the last inequality follows from the strain bound for CMDS.
Substituting the node-attribution identity into this bound yields the stated result.
\end{proof}

Here, we establish tighter bound by summing the squared pairwise errors over all time pairs.

\begin{proposition}
\label{prop:local-global-app}
Suppose $T\geq 2$. For any $\rho\in\{\mathrm{tr},1,\ldots,d\}$ and $c_\rho\le \operatorname{rank}(\mathbf E_\rho)$,
\[
\sum_{t<s}
\left|
\|\hat{\psi}_{\rho}(t)-\hat{\psi}_{\rho}(s)\|^2
-
\sum_{i=1}^n \hat a_{\rho}(t,s;i)^2
\right|^2
\le
4(T-1)
\left(
\sum_{i=c_\rho+1}^T
\lambda_i(\hat{\mathbf E}_{\rho})^2
\right).
\]
\end{proposition}

\begin{proof}
By the proof of Theorem~\ref{thm:local-global-main}, for any
$\rho\in\{\mathrm{tr},1,\ldots,d\}$,
\[
\left|
\|\hat\psi_\rho(t)-\hat\psi_\rho(s)\|^2
-
\hat d_\rho(t,s)^2
\right|
=
\left|
(\mathbf e_t-\mathbf e_s)^\top
\Xi_\rho
(\mathbf e_t-\mathbf e_s)
\right|,
\]
where
\[
\Xi_\rho
:=
\hat{\mathbf Z}_\rho\hat{\mathbf Z}_\rho^\top
-
\hat{\mathbf E}_\rho.
\]
Thus,
\[
\begin{aligned}
&\sum_{t<s}
\left|
\|\hat\psi_\rho(t)-\hat\psi_\rho(s)\|^2
-
\hat d_\rho(t,s)^2
\right|^2\\
&\quad=
\sum_{t<s}
\left|
\Xi_{\rho,tt}
-\Xi_{\rho,ts}
-\Xi_{\rho,st}
+\Xi_{\rho,ss}
\right|^2\\
&\quad\leq
4\sum_{t<s}
\left(
|\Xi_{\rho,tt}|^2
+
|\Xi_{\rho,ts}|^2
+
|\Xi_{\rho,st}|^2
+
|\Xi_{\rho,ss}|^2
\right),
\end{aligned}
\]
where the last inequality follows from the Cauchy--Schwarz inequality.
Since $\Xi_\rho$ is symmetric,
\[
\begin{aligned}
&\sum_{t<s}
\left(
|\Xi_{\rho,tt}|^2
+
|\Xi_{\rho,ts}|^2
+
|\Xi_{\rho,st}|^2
+
|\Xi_{\rho,ss}|^2
\right)\\
&\quad=
(T-1)\sum_{t=1}^T|\Xi_{\rho,tt}|^2
+
2\sum_{t<s}|\Xi_{\rho,ts}|^2
\leq
(T-1)\|\Xi_\rho\|_F^2.
\end{aligned}
\]
Therefore,
\[
\sum_{t<s}
\left|
\|\hat\psi_\rho(t)-\hat\psi_\rho(s)\|^2
-
\hat d_\rho(t,s)^2
\right|^2
\leq
4(T-1)\|\Xi_\rho\|_F^2.
\]
Combining this with the CMDS truncation bound and the node-attribution identity from the proof of Theorem~\ref{thm:local-global-main} gives the stated result.
\end{proof}

The argument is deterministic conditional on the estimated embeddings. If all positive CMDS dimensions are retained, the discarded spectrum is zero and squared trajectory displacement equals the sum of squared node attributions for every time pair.

\section{Change Point Localization}
\label{app:cp}

In this section, we formalize change point estimation based on a one-dimensional Euclidean trajectory. Throughout, let \(\psi_\rho(t)\) denote a one-dimensional population trajectory and let \(\hat\psi_\rho(t)\) denote its estimated counterpart, where \(\rho\in\{\mathrm{tr},1,\ldots,d\}\) and \(c_\rho=1\). The index \(\rho\) is fixed throughout this section.

\begin{definition}\label{def:cp}
Let \(\mathcal K\subseteq\{2,\ldots,T\}\) be the set of candidate knot locations, and let \(k\in\mathcal K\). Let \(\theta\) denote a parameter set, and let \(\Psi_L(\cdot;k,\theta)\) and \(\Psi_R(\cdot;k,\theta)\) be parametric functions before and after \(k\), respectively. Define
\[
\Psi(t;k,\theta)=
\begin{cases}
\Psi_L(t;k,\theta),&t<k,\\
\Psi_R(t;k,\theta),&t\ge k.
\end{cases}
\]
\end{definition}

\textbf{Example (\(0\)th-order case).}
Let \(\mathcal K^{(0)}=\{2,\ldots,T\}\) and \(\theta=(a_L,a_R)\in\mathbb R^2\). Define
\[
\Psi_L^{(0)}(t;k,\theta)=a_L,
\qquad
\Psi_R^{(0)}(t;k,\theta)=a_R.
\]
This specification models a \(0\)th-order change through a jump in the trajectory coordinate at \(k\).

\textbf{Example (\(1\)st-order case).}
Let \(\mathcal K^{(1)}=\{2,\ldots,T-1\}\) and \(\theta=(a,b_L,b_R)\in\mathbb R^3\). Define
\[
\Psi_L^{(1)}(t;k,\theta)=a+b_Lt,
\qquad
\Psi_R^{(1)}(t;k,\theta)=a+b_Lk+b_R(t-k).
\]
This formulation enforces continuity at \(k\) while allowing a change in slope. Candidate fits may have \(a_L=a_R\) or \(b_L=b_R\). The nonzero-change requirements in the propositions below apply only to the population parameter \(\theta^*\). This allows the least-squares minima over candidate parameters to be attained.

For each candidate \(k\in\mathcal K\), fit the best piecewise function with a knot at \(k\) by minimizing
\[
\hat Q_\rho(k):=\min_\theta\sum_{t=1}^T\bigl(\hat\psi_\rho(t)-\Psi(t;k,\theta)\bigr)^2.
\]
We then define
\[
\hat t_\rho\in\arg\min_{k\in\mathcal K}\hat Q_\rho(k).
\]
The results below treat \(\hat\psi_\rho(t)\) as an estimate of the population trajectory \(\psi_\rho(t)\). Conditional on the trajectory estimation error, change point estimation is a deterministic approximation problem for piecewise parametric functions.

\begin{assumption}\label{assump:unique}
Suppose
\[
Q_\rho(k):=\min_\theta\sum_{t=1}^T\bigl(\psi_\rho(t)-\Psi(t;k,\theta)\bigr)^2
\]
is uniquely minimized at \(t_\rho^*\in\mathcal K\), and
\[
\min_\theta\sum_{t=1}^T\bigl(\psi_\rho(t)-\Psi(t;t_\rho^*,\theta)\bigr)^2
\]
has a unique minimizer \(\theta^*\).
\end{assumption}

This assumption specifies a unique population change point and a unique best-fitting model at that point. The next assumption quantifies how rapidly the discrepancy between piecewise templates grows as a candidate knot moves away from \(t_\rho^*\).

\begin{assumption}\label{assump:Psi-error}
Let \(D(\theta^*)\) be strictly positive. There exists a constant \(\alpha>0\) such that, for all \(k\in\mathcal K\),
\[
\min_\theta\sum_{t=1}^T
\bigl(\Psi(t;t_\rho^*,\theta^*)-\Psi(t;k,\theta)\bigr)^2
\ge \alpha D(\theta^*)|k-t_\rho^*|.
\]
\end{assumption}

Assumption~\ref{assump:Psi-error} provides a linear separation condition between piecewise templates around the population knot. It is not a direct lower bound on \(Q_\rho(k)-Q_\rho(t_\rho^*)\) when \(\psi_\rho\) is only approximately piecewise. The approximation error \(Q_\rho(t_\rho^*)\) remains in the localization bound.

We now verify that this separation condition holds for the \(0\)th- and \(1\)st-order models introduced above.

\begin{proposition}\label{prop:cp-separation-0}
Let \(T\ge3\), \(\Psi=\Psi^{(0)}\), \(\mathcal K=\mathcal K^{(0)}\), and assume Assumption~\ref{assump:unique}. Suppose that \(\theta^*=(a_L^*,a_R^*)\) with \(a_L^*\ne a_R^*\). Then Assumption~\ref{assump:Psi-error} holds with
\[
D(\theta^*)=|a_L^*-a_R^*|^2.
\]
\end{proposition}

For a \(0\)th-order change, the separation strength is determined by the squared jump size.

\begin{proposition}\label{prop:cp-separation-1}
Let \(T\ge4\), \(\Psi=\Psi^{(1)}\), \(\mathcal K=\mathcal K^{(1)}\), and assume Assumption~\ref{assump:unique}. Suppose that \(\theta^*=(a^*,b_L^*,b_R^*)\) with \(b_L^*\ne b_R^*\). Then Assumption~\ref{assump:Psi-error} holds with
\[
D(\theta^*)=|b_L^*-b_R^*|^2.
\]
\end{proposition}

For a \(1\)st-order change, the separation strength is determined by the squared change in slope. We first prove the localization theorem and then verify the separation bound in each case.

\begin{proof}[Proof of Theorem~\ref{thm:cp-localization}]
Fix \(k\in\mathcal K\) and let \(\theta_k^*\) and \(\hat\theta_k\) be minimizers defining \(Q_\rho(k)\) and \(\hat Q_\rho(k)\), respectively. By the triangle inequality,
\[
\begin{aligned}
\hat Q_\rho(k)^{1/2}-Q_\rho(k)^{1/2}
&\le
\left(\sum_{t=1}^T\bigl(\hat\psi_\rho(t)-\Psi(t;k,\theta_k^*)\bigr)^2\right)^{1/2}
-\left(\sum_{t=1}^T\bigl(\psi_\rho(t)-\Psi(t;k,\theta_k^*)\bigr)^2\right)^{1/2}\\
&\le
\left(\sum_{t=1}^T\bigl(\hat\psi_\rho(t)-\psi_\rho(t)\bigr)^2\right)^{1/2}.
\end{aligned}
\]
Applying the same argument with \(\hat\theta_k\) gives the reverse inequality. Thus, for every \(k\in\mathcal K\),
\[
\bigl|\hat Q_\rho(k)^{1/2}-Q_\rho(k)^{1/2}\bigr|
\le
\left(\sum_{t=1}^T\bigl(\hat\psi_\rho(t)-\psi_\rho(t)\bigr)^2\right)^{1/2}.
\]
Since \(\hat t_\rho\) minimizes \(\hat Q_\rho\), it follows that
\[
Q_\rho(\hat t_\rho)^{1/2}-Q_\rho(t_\rho^*)^{1/2}
\le
2\left(\sum_{t=1}^T\bigl(\hat\psi_\rho(t)-\psi_\rho(t)\bigr)^2\right)^{1/2}.
\]
Next, by the reverse triangle inequality and Assumption~\ref{assump:Psi-error}, for every \(k\in\mathcal K\),
\[
\begin{aligned}
Q_\rho(k)^{1/2}
&=\min_\theta\left(\sum_{t=1}^T\bigl(\psi_\rho(t)-\Psi(t;k,\theta)\bigr)^2\right)^{1/2}\\
&\ge
\min_\theta\left(\sum_{t=1}^T\bigl(\Psi(t;t_\rho^*,\theta^*)-\Psi(t;k,\theta)\bigr)^2\right)^{1/2}
-Q_\rho(t_\rho^*)^{1/2}\\
&\ge
\bigl(\alpha D(\theta^*)|k-t_\rho^*|\bigr)^{1/2}
-Q_\rho(t_\rho^*)^{1/2}.
\end{aligned}
\]
Applying this at \(k=\hat t_\rho\) and combining the bounds yields
\[
\bigl(\alpha D(\theta^*)|\hat t_\rho-t_\rho^*|\bigr)^{1/2}
\le
2\left(\sum_{t=1}^T\bigl(\hat\psi_\rho(t)-\psi_\rho(t)\bigr)^2\right)^{1/2}
+2Q_\rho(t_\rho^*)^{1/2}.
\]
Squaring both sides proves the theorem.
\end{proof}

\begin{proof}[Proof of Proposition~\ref{prop:cp-separation-0}]
Fix \(k\in\mathcal K^{(0)}\). For \(k=t_\rho^*\), the inequality holds trivially. If \(k<t_\rho^*\), the candidate's right segment combines \(t_\rho^*-k\) observations at level \(a_L^*\) with \(T-t_\rho^*+1\) observations at level \(a_R^*\). Hence
\[
\begin{aligned}
&\min_\theta\sum_{t=1}^T
\bigl(\Psi^{(0)}(t;t_\rho^*,\theta^*)-\Psi^{(0)}(t;k,\theta)\bigr)^2\\
&=\min_{a_R}\left[
\sum_{t=k}^{t_\rho^*-1}(a_L^*-a_R)^2
+\sum_{t=t_\rho^*}^{T}(a_R^*-a_R)^2
\right]\\
&=\frac{(t_\rho^*-k)(T-t_\rho^*+1)}{T-k+1}D(\theta^*)\\
&\ge\frac{T-t_\rho^*+1}{T-1}D(\theta^*)|k-t_\rho^*|.
\end{aligned}
\]
If \(k>t_\rho^*\), the candidate's left segment combines \(t_\rho^*-1\) observations at level \(a_L^*\) with \(k-t_\rho^*\) observations at level \(a_R^*\). Thus
\[
\begin{aligned}
&\min_\theta\sum_{t=1}^T
\bigl(\Psi^{(0)}(t;t_\rho^*,\theta^*)-\Psi^{(0)}(t;k,\theta)\bigr)^2\\
&=\min_{a_L}\left[
\sum_{t=1}^{t_\rho^*-1}(a_L^*-a_L)^2
+\sum_{t=t_\rho^*}^{k-1}(a_R^*-a_L)^2
\right]\\
&=\frac{(t_\rho^*-1)(k-t_\rho^*)}{k-1}D(\theta^*)\\
&\ge\frac{t_\rho^*-1}{T-1}D(\theta^*)|k-t_\rho^*|.
\end{aligned}
\]
Therefore, Assumption~\ref{assump:Psi-error} holds with
\[
\alpha=\frac{\min\{t_\rho^*-1,T-t_\rho^*+1\}}{T-1}>0.
\]
\end{proof}

\begin{proof}[Proof of Proposition~\ref{prop:cp-separation-1}]
Fix \(k\in\mathcal K^{(1)}\). For \(k=t_\rho^*\), the inequality holds trivially. We consider \(k<t_\rho^*\) and \(k>t_\rho^*\) separately.

For \(k<t_\rho^*\), dropping the fitting error on \(t<k\) gives
\[
\begin{aligned}
&\min_\theta\sum_{t=1}^T
\bigl(\Psi^{(1)}(t;t_\rho^*,\theta^*)-\Psi^{(1)}(t;k,\theta)\bigr)^2\\
&\ge
\min_{a,b}\left[
\sum_{t=k}^{t_\rho^*}(a^*+b_L^*t-a-bt)^2
+\sum_{t=t_\rho^*+1}^{T}
\bigl(a^*+b_L^*t_\rho^*+b_R^*(t-t_\rho^*)-a-bt\bigr)^2
\right]\\
&=
\min_{a,b}\left[
\sum_{t=k}^{t_\rho^*}\bigl((b_L^*-b_R^*)t-a-bt\bigr)^2
+\sum_{t=t_\rho^*+1}^{T}
\bigl((b_L^*-b_R^*)t_\rho^*-a-bt\bigr)^2
\right].
\end{aligned}
\]
In the last equality, the common affine term \(a^*+b_R^*t\) is absorbed into \(a+bt\). Solving the least-squares problem yields the lower bound
\[
\frac{x_1y_1(y_1+1)(2x_1y_1+x_1-y_1+1)}
{6z_1(z_1^2-1)}
D(\theta^*)|t_\rho^*-k|,
\]
where \(x_1=t_\rho^*-k+1\ge2\), \(y_1=T-t_\rho^*\ge1\), and \(z_1=T-k+1\). For \(t_\rho^*>2\), define
\[
\alpha_1=
\min_{2\le k<t_\rho^*}
\frac{x_1y_1(y_1+1)(2x_1y_1+x_1-y_1+1)}
{6z_1(z_1^2-1)}.
\]
The index set is finite and every term is positive, so \(\alpha_1>0\).

For \(k>t_\rho^*\), the candidate template is affine on \(t=1,\ldots,k\), including \(t=k\). Dropping the fitting error on \(t>k\) gives
\[
\begin{aligned}
&\min_\theta\sum_{t=1}^T
\bigl(\Psi^{(1)}(t;t_\rho^*,\theta^*)-\Psi^{(1)}(t;k,\theta)\bigr)^2\\
&\ge
\min_{a,b}\left[
\sum_{t=1}^{t_\rho^*}(a^*+b_L^*t-a-bt)^2
+\sum_{t=t_\rho^*+1}^{k}
\bigl(a^*+b_L^*t_\rho^*+b_R^*(t-t_\rho^*)-a-bt\bigr)^2
\right]\\
&=
\min_{a,b}\left[
\sum_{t=1}^{t_\rho^*}\bigl((b_L^*-b_R^*)t-a-bt\bigr)^2
+\sum_{t=t_\rho^*+1}^{k}
\bigl((b_L^*-b_R^*)t_\rho^*-a-bt\bigr)^2
\right].
\end{aligned}
\]
Solving the least-squares problem yields the lower bound
\[
\frac{x_2(x_2-1)(y_2+1)(2x_2y_2+x_2-y_2+1)}
{6z_2(z_2^2-1)}
D(\theta^*)|t_\rho^*-k|,
\]
where \(x_2=t_\rho^*\ge2\), \(y_2=k-t_\rho^*\ge1\), and \(z_2=k\). For \(t_\rho^*<T-1\), define
\[
\alpha_2=
\min_{t_\rho^*<k\le T-1}
\frac{x_2(x_2-1)(y_2+1)(2x_2y_2+x_2-y_2+1)}
{6z_2(z_2^2-1)}.
\]
The index set is finite and every term is positive, so \(\alpha_2>0\). Therefore, Assumption~\ref{assump:Psi-error} holds with
\[
\alpha=
\begin{cases}
\alpha_2,&t_\rho^*=2,\\
\min\{\alpha_1,\alpha_2\},&2<t_\rho^*<T-1,\\
\alpha_1,&t_\rho^*=T-1.
\end{cases}
\]
This establishes separation for each fixed \(T\). The value of \(\alpha\) may depend on \(T\).
\end{proof}

The sign ambiguity of a one-dimensional CMDS trajectory does not change the estimated knot for either model. Replacing \(\hat\psi_\rho\) by \(-\hat\psi_\rho\) negates the fitted levels or slopes and leaves every \(\hat Q_\rho(k)\) unchanged. Thus the theorem also applies with the sign-aligned error
\[
\min_{s\in\{-1,1\}}
\left(\sum_{t=1}^T\bigl(s\hat\psi_\rho(t)-\psi_\rho(t)\bigr)^2\right)^{1/2}.
\]
When the population trajectory belongs exactly to the chosen piecewise family, \(Q_\rho(t_\rho^*)=0\); otherwise this approximation error remains in the localization bound.

\section{Additional Experimental Details}
\label{app:experiments}


This section provides the full experimental details omitted from the main text due to space constraints. It describes the complete synthetic data generation procedure, the construction of the real world legal citation network, implementation details for all methods and baselines, and additional quantitative and qualitative results. These supplementary experiments are organized around the paper’s main claims, including canonical geometry, multiscale trajectory recovery, node attribution and its trajectory-level approximation bound, change point detection, and interpretability on real world data.

\subsection{Maximum Directional Variation Baseline and Non-Euclidean Realizability}
\label{app:mv-non-euclidean}

Some comparisons retain the maximum directional variation (MV) distance of
\citet{athreya2025euclidean} as a baseline, with the Procrustes alignment step omitted. 
Its population form is
\[
d_{\mathrm{MV}}(t,s)
=
\|\mathbf M(t,s)\|_2^{1/2}
=
\sup_{\|\bm u\|=1}
\|(\phi(t)-\phi(s))^\top\bm u\|_{L^2(\Omega;\mathbb R)},
\]
and the estimated distance is
\[
\hat d_{\mathrm{MV}}(t,s)
=
\left\|
\hat{\mathbf M}(t,s)
\right\|_2^{1/2}
=
\frac1{n^{1/2}}\|\hat{\mathbf Y}(t)-\hat{\mathbf Y}(s)\|_2.
\]

The original formulation in \citet{athreya2025euclidean} instead incorporates a Procrustes transformation. A limitation of their original formulation is that its resulting pairwise distances cannot, in general, be exactly realized by Euclidean trajectories constructed via CMDS. We illustrate this limitation with the following counterexample.
Fix latent dimension $d=2$ and time point set $\mathcal{T}=\{1,2,3,4\}$.
Let $(\Omega,\mathcal{F},\mathbb{P})$ be given by
$\Omega = \{\omega_1,\omega_2\}, 
\mathbb{P}(\{\omega_1\}) = \mathbb{P}(\{\omega_2\}) = 0.5$.
Define random vectors $\chi \in L^2(\Omega;\mathbb{R}^2), \phi(t) \in L^2(\Omega;\mathbb{R}^2), \; t\in\mathcal{T}$ by specifying their realizations as follows:
\[
\chi(\omega_1) = (\sqrt{2},0), 
\qquad 
\chi(\omega_2) = (0,\sqrt{2}),
\]
and for $t\in\{1,2,3,4\}$,
\[
\begin{aligned}
	\phi(1,\omega_1) &= (1/8,0), & \phi(1,\omega_2) &= (0,1/8), \\
	\phi(2,\omega_1) &= (1/8,0), & \phi(2,\omega_2) &= (0,1/4), \\
	\phi(3,\omega_1) &= (1/4,0), & \phi(3,\omega_2) &= (0,1/8), \\
	\phi(4,\omega_1) &= (1/8,0), & \phi(4,\omega_2) &= (0,3/8).
\end{aligned}
\]

For this construction, for each pair \(t,s\in\mathcal T\), we perform the Procrustes optimization by choosing the orthogonal transformation that minimizes the corresponding pairwise distance. The resulting Procrustes-based MV distances of \citet{athreya2025euclidean} give the following population squared distance matrix:
\[
\mathbf D^{(2)} = \frac{1}{128}\begin{pmatrix}
	0 & 1 & 1 & 4 \\
	1 & 0 & 1 & 1 \\
	1 & 1 & 0 & 4 \\
	4 & 1 & 4 & 0 
\end{pmatrix},
\]
thus the exact eigenvalues of the centered Gram matrix are:
\[
\lambda_1 = \frac{5 + 3\sqrt{3}}{512}, \quad
\lambda_2 = \frac{1}{256}, \quad
\lambda_3 = -\frac{3\sqrt{3} - 5}{512}, \quad
\lambda_4 = 0.
\]
Since it has a negative eigenvalue $\lambda_3$, it is not positive semidefinite. Therefore, this distance matrix cannot be exactly embedded in a Euclidean space.

\subsection{Synthetic Datasets}
\label{app:exp-synthetic}

The two synthetic datasets are generated according to the following process. They can be viewed as dynamic stochastic block models (SBM)~\citep{holland1983stochastic} with three equally sized communities. We set the number of latent modes to \(3\) and let \(\{\bm e_k\}_{k=1}^3\) denote the standard basis of \(\mathbb R^3\). The anchor random vector is sampled as
\[
\chi =\sqrt{3}\,\bm e_k \qquad \text{with probability } \frac13,\quad k\in\{1,2,3\},
\]
and, conditional on \(\chi\), the dynamic random vector at time \(t\) is defined by
\[
\phi(t)=\frac13\,\mathcal B(t)\chi,
\]
where \(\mathbf B(t)\in\mathbb R^{3\times 3}\) is a time-varying block probability matrix. This yields a dynamic SBM with uniform community proportions. The network structure is decomposed into the orthonormal modes
\[
\bm u_1=\frac{1}{\sqrt3}(1,1,1),\qquad
\bm u_2=\frac{1}{\sqrt6}(1,1,-2),\qquad
\bm u_3=\frac{1}{\sqrt2}(1,-1,0),
\]
where \(\bm u_1\) captures global connectivity, \(\bm u_2\) captures separation between community \(3\) and communities \(1\) and \(2\), and \(\bm u_3\) captures separation between communities \(1\) and \(2\). Let $\xi_k(t)$ denote the strength of mode $k$ at time $t$. The block probability matrix is defined as
\[
\mathcal B(t)=\sum_{k=1}^3 \xi_k(t)\,\bm u_k\bm u_k^\top.
\]
Under this parametrization, the 1D mode-wise trajectory coincides with $\xi_k(t)/3$ up to a translation specific to each $k$ and a possible sign ambiguity. Each \(\xi_k(t)\) is chosen to exhibit either a \(0\)th order or \(1\)st order change point, with the change points occurring at different times across modes. Since the population trajectories are known, this benchmark allows direct evaluation of geometric distortion, mode recovery, node attribution, and change point localization against the ground truth. 

We construct two synthetic datasets with distinct ground truth trajectories using the basis defined above, each tailored to a specific experimental objective. Dataset 1 is used exclusively to empirically validate our theoretical results under controlled conditions that exactly match their assumptions. In particular, each mode contains a single change point, which is a standard simplifying assumption in change point localization. By contrast, Dataset 2 is designed for comparison with existing state-of-the-art change point detection methods and features more complex temporal dynamics in order to assess robustness beyond these assumptions.


\textbf{Dataset 1 (Single change point per mode):} We show the mode strength \(\xi_k(t)\) in the left panel of Figure~\ref{fig:xi}, with \(T=16\) time points. For this dataset, the ordering of the modes by their contribution to temporal variation is exactly the same as the basis ordering. In other words, mode \(k\) corresponds to \(\bm u_k\) for \(k=1,2,3\). Each mode contains a single change point, summarized as follows.

\begin{itemize}
    \item Mode 1 (\(\bm u_1\)): a \(1\)st order change point at \(t_1^*=4\),
    \item Mode 2 (\(\bm u_2\)): a \(0\)th order change point at \(t_2^*=9\),
    \item Mode 3 (\(\bm u_3\)): a \(0\)th order change point at \(t_3^*=13\).
\end{itemize}

To validate the convergence of our proposed estimators, we generate networks with the following numbers of nodes:
\[
n \in \{30,50,70,100,150,200,300,400,500\},
\]
\noindent and produce \(100\) independent realizations for each \(n\).

\textbf{Dataset 2 (Multiple change points per mode):} The underlying mode strength is shown in the right panel of Figure~\ref{fig:xi}, with \(T=70\) time points. We can confirm that each mode-wise trajectory contains multiple change points. We fix the number of nodes at \(n=500\) and generate \(100\) independent realizations. Unlike Dataset 1, the correspondence between the mode indices and the basis vectors is not aligned, and each mode exhibits multiple change points of varying order. Specifically, the modes are defined as follows:
\begin{itemize}
    \item Mode 1 (\(\bm u_3\)): \(0\)th order change points at \(t_{1,1}^*=31\) and \(t_{1,2}^*=61\),
    \item Mode 2 (\(\bm u_2\)): a \(0\)th order change point at \(t_{2,1}^*=21\) and a \(1\)st order change point at \(t_{2,2}^*=51\),
    \item Mode 3 (\(\bm u_1\)): a \(1\)st order change point at \(t_{3,1}^*=11\) and a \(0\)th order change point at \(t_{3,2}^*=41\).
\end{itemize}

To provide an intuitive, plain language view of how the network evolves over time for both datasets, we visualize each network snapshot using Gephi~\citep{bastian2009gephi} in Figure~\ref{fig:network_snapshots}. At every time point, communities are detected via modularity maximization~\citep{newman2006modularity,blondel2008fast} and nodes are colored according to the resulting partition. Node positions are computed independently at each snapshot using the ForceAtlas2 layout~\citep{jacomy2014forceatlas2}, yielding visually interpretable representations of local and global connectivity structure. Because both community detection and layout are recomputed separately for each time point, node colors and positions are not intended to be directly comparable across time. The visualizations serve solely as qualitative references for the evolving network structure that is later summarized by the proposed variation and mirror based metrics.

\begin{figure}[t]
\centering
\begin{subfigure}[t]{0.48\linewidth}
    \centering
    \includegraphics[width=\linewidth]{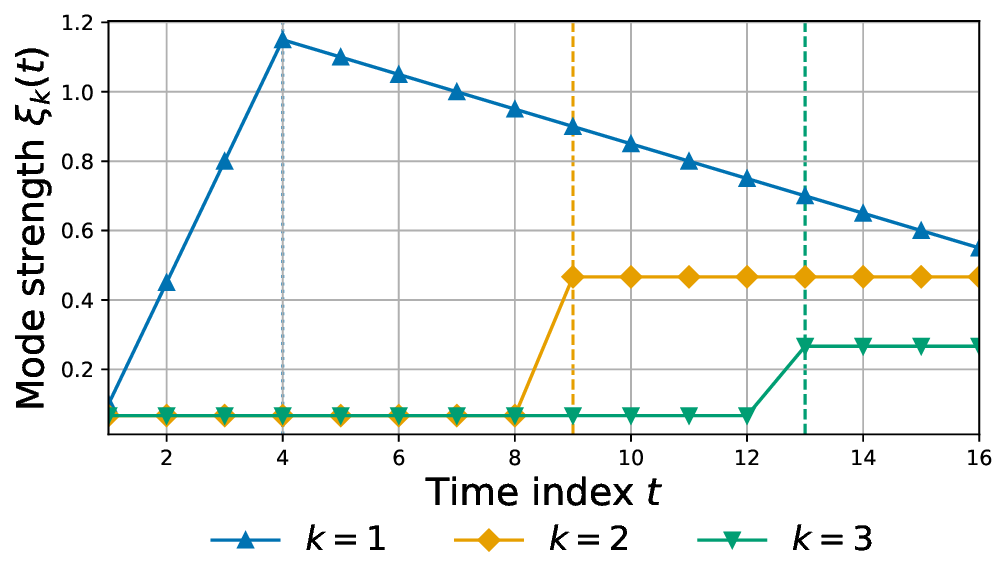}
    \caption{Dataset 1 (single change point per mode).}
\end{subfigure}
\hfill
\begin{subfigure}[t]{0.48\linewidth}
    \centering
    \includegraphics[width=\linewidth]{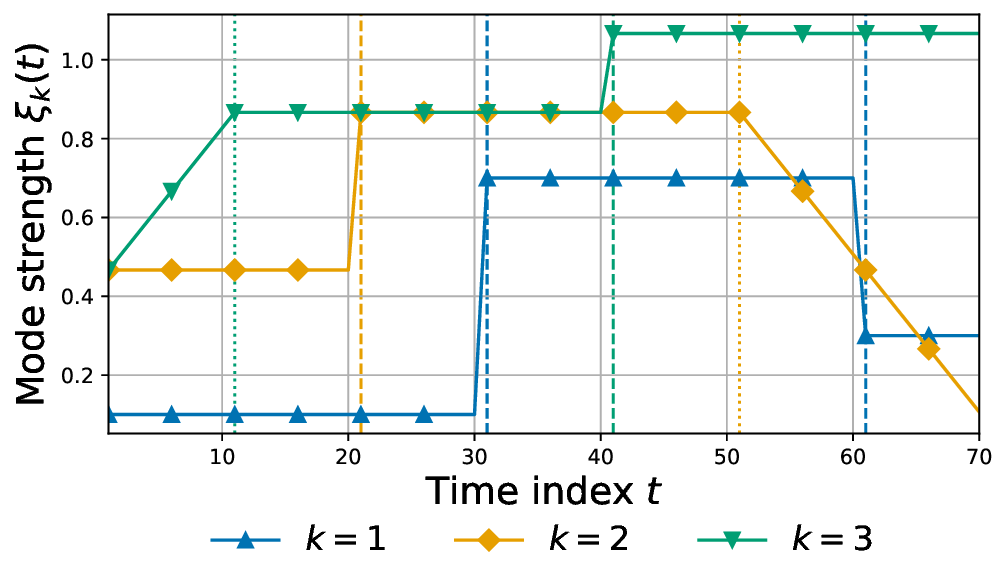}
    \caption{Dataset 2 (multiple change points per mode).}
\end{subfigure}
\caption{Mode strengths $\xi_k(t)$ for the two datasets. Vertical segments indicate change points for the corresponding mode. Dashed lines mark a jump in $\xi_k$ (order~0) and dotted lines mark a change in slope (order~1).}
\label{fig:xi}
\end{figure}


\begin{figure}[t]
\centering

\begin{subfigure}[t]{\linewidth}
    \centering
    \begin{tabular}{@{}cccc@{}}
        \includegraphics[width=0.95\linewidth]{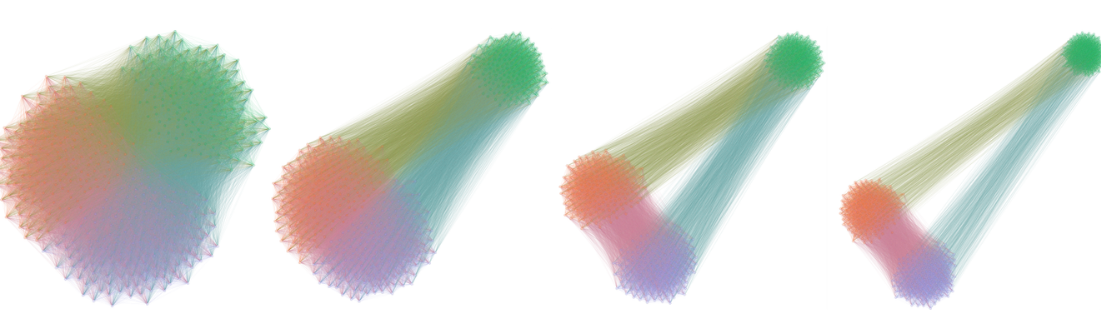} 
    \end{tabular}
    \caption{Dataset 1 (\(t=6,10,14,16\)). The modularity scores are 0.178, 0.312, 0.45 and 0.497 respectively.}
\end{subfigure}

\vspace{0.5em}

\begin{subfigure}[t]{\linewidth}
    \centering
    \begin{tabular}{@{}cccc@{}}
        \includegraphics[width=0.9\linewidth]{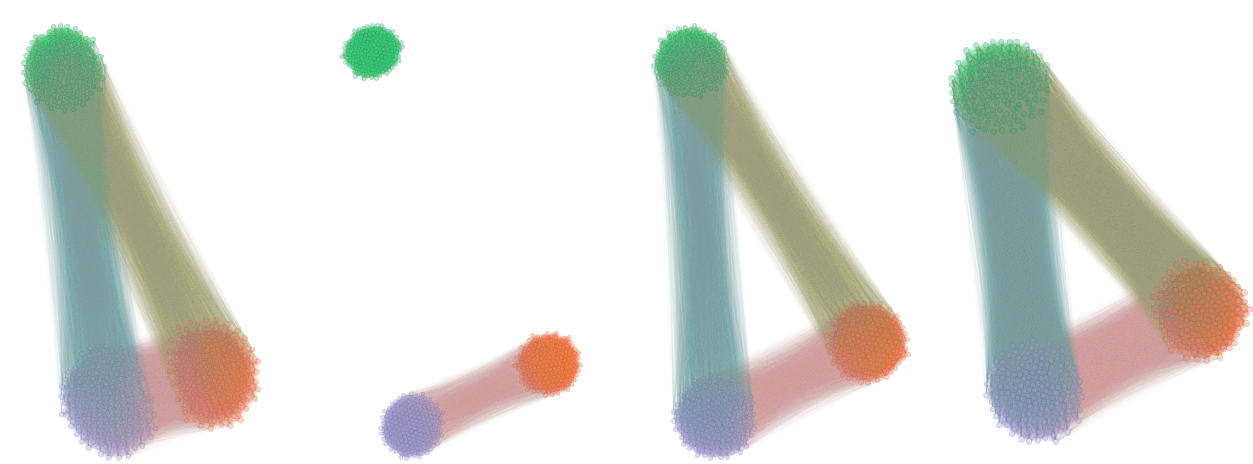} &
    \end{tabular}
    \caption{Dataset 2 (\(t=15,35,55,70\)). The modularity scores are 0.458, 0.624, 0.567 and 0.446 respectively.}
\end{subfigure}

\caption{Visualization of network evolution (both with \(n=500\)). Nodes are colored by communities detected via modularity maximization, and layouts are computed independently at each time point using ForceAtlas2.}
\label{fig:network_snapshots}

\end{figure}

\subsection{Additional Results for Canonical Second-moment Geometry}
\label{app:exp-geometry}

First, using Dataset 1, we examine whether (i) the \(\mathrm{GL}(d)\) ambiguity affects the preservation of latent geometry, and (ii) the choice of basis affects the consistent recovery of mode-wise distances and trajectories, and in particular whether a canonical basis enables successful recovery. Our primary goal is trajectory recovery. Because accurate trajectory estimation depends on faithful pairwise distance recovery, we begin by evaluating distance preservation, since accurate trajectory estimation relies on the faithful recovery of pairwise distances and therefore provides a natural way to verify that each step of the procedure successfully estimates the latent structure.

For each distance type \(\rho\in\{\mathrm{MV},\mathrm{tr},1,\ldots,d\}\), we compare three squared distance matrices:
\begin{itemize}
    \item the estimated squared distance matrix \(\hat{\mathbf D}^{(2)}_\rho\),
    \item the finite-sample squared distance matrix \(\tilde{\mathbf D}^{(2)}_\rho\), induced from \(\tilde{\mathbf M}(t,s)\) analogously to \(\hat{\mathbf D}^{(2)}_\rho\), which serves as an intermediate ground truth, and
    \item the population squared distance matrix \(\mathbf D^{(2)}_\rho\), which represents the target geometry.
\end{itemize}


We compare two methods: (i) the original UASE~\citep{gallagher2021spectral} and (ii) our proposed modified UASE. We consider three types of distances: the MV baseline defined above, the trace variation distance, and the mode-wise distance. For the mode-wise analysis, we distinguish between representations under (i) a canonical basis and (ii) a fixed standard basis. Since basis selection requires multiple modes, this distinction is relevant only for mode-wise distances.

For the modified UASE, we measure geometric error using the Frobenius norm between squared distance matrices:
\[
\left\| \hat{\mathbf D}^{(2)}_\rho - \tilde{\mathbf D}^{(2)}_\rho \right\|_F, \qquad \left\| \hat{\mathbf D}^{(2)}_\rho - \mathbf D^{(2)}_\rho \right\|_F.
\]

For the original UASE, we allow an optimal global scaling to remove any purely scalar mismatch before measuring the remaining geometric distortion:
\[
\min_{r>0} \left\| \hat{\mathbf D}^{(2)}_\rho - r \tilde{\mathbf D}^{(2)}_\rho \right\|_F ,\qquad \min_{r>0}\left\| \hat{\mathbf D}^{(2)}_\rho - r\mathbf D^{(2)}_\rho \right\|_F.
\]

We then evaluate trajectory reconstruction up to sign, and up to scaling as well for the original UASE. For the modified UASE, the error is defined as
\[
\min_{s \in \{\pm 1\}} \sum_{t=1}^{T} \bigl(\hat{\psi}(t) - s\,\psi(t)\bigr)^2,
\]
where the sign ambiguity is removed by optimizing over \(s \in \{\pm 1\}\).

For each method, we evaluate both distance reconstruction accuracy and trajectory reconstruction accuracy. We report the mean $\pm$ standard deviation over \(100\) independent trials for all experiments using Dataset 1.

Figure~\ref{fig:latent-vs-embed} shows the results for finite sample geometric recovery. The original UASE exhibits persistent distortion across all distance types, as indicated by the dashed lines, which do not converge to \(0\). In contrast, the proposed method substantially reduces this mismatch and achieves stable convergence to \(0\) for the MV, trace, and mode-wise distances when the latter are expressed in the canonical basis, as shown in Figure~\ref{fig:latent-vs-embed} (a) to (c). Figure~\ref{fig:latent-vs-embed} (d), however, shows that the standard basis also fails to converge. This highlights that \(\mathrm{GL}(d)\) transformations can distort geometric structure and underscores the importance of Theorem~\ref{thm:modified-uase-paper}.

We next examine whether the estimated embedding also recovers the population geometry accurately. Figure~\ref{fig:population-vs-embed} compares the estimated squared distance matrix \( \hat{\mathbf D}^{(2)}_\rho\) with the population squared distance matrix \( \mathbf D^{(2)}_\rho\), providing a direct assessment of population level recovery. We observe the same qualitative pattern. Both the original UASE and the mode-wise representation obtained using the standard basis fail to converge to \(0\).

Having established the importance of geometric recovery, we finally report trajectory reconstruction results in Figure~\ref{fig:trajectory}. The solid lines in Figure~\ref{fig:trajectory} (a) to (c) show that the proposed method achieves consistent convergence for the MV, trace, and all mode-wise trajectories. By contrast, the original UASE and the mode-wise trajectories based on the standard basis again fail to converge, as shown by the dashed lines in Figure~\ref{fig:trajectory} (a) to (d).

To visually compare how these distortions appear in each trajectory, Figure~\ref{fig:multiscale-trajectories} compares the MV, trace, and mode-wise trajectories under the canonical basis, with the original UASE shown in the top row, that is, Figure~\ref{fig:multiscale-trajectories} (a) to (c), and the proposed method shown in the bottom row, that is, Figure~\ref{fig:multiscale-trajectories} (d) to (f). The dashed lines represent the population trajectories, whereas the solid lines show the estimated trajectories. Under the proposed method, the mode-wise trajectories clearly recover the distinct temporal structures embedded in the latent process. In particular, the \(0\)th and \(1\)st order variations are separated into different modes, enabling a more fine grained view of the temporal dynamics.

At the same time, Figure~\ref{fig:multiscale-trajectories} (a), (b), (d), and (e) show that the MV and trace trajectories aggregate these changes into a single global summary, thereby compressing multiple sources of variation. As a result, weaker but meaningful temporal changes that are visible in the mode-wise representation are attenuated in the global trajectories.

For the original UASE, although the overall temporal patterns are roughly preserved, noticeable distortions remain, as shown in Figure~\ref{fig:multiscale-trajectories} (a) to (c). In particular, the mode-wise trajectories exhibit inconsistencies in alignment and ordering across modes, suggesting that the relative importance of the latent components is not reliably recovered. This behavior is consistent with the geometric distortions induced by the \(\mathrm{GL}(d)\) ambiguity.

These results lead to three main findings: (i) the \(\mathrm{GL}(d)\) ambiguity induces geometric distortion in both the global structure, as measured by the MV and trace distances, and the mode-wise structure, (ii) the proposed method resolves this issue at the global level, enabling consistent recovery of MV and trace geometry and trajectories, and (iii) the proposed method successfully recovers mode-wise geometry and trajectories under a canonical basis, showing that consistent mode-wise recovery becomes achievable with an appropriate choice of basis.

\begin{figure}[t]
    \centering

    \begin{subfigure}[t]{0.48\textwidth}
        \centering
        \includegraphics[width=\textwidth]{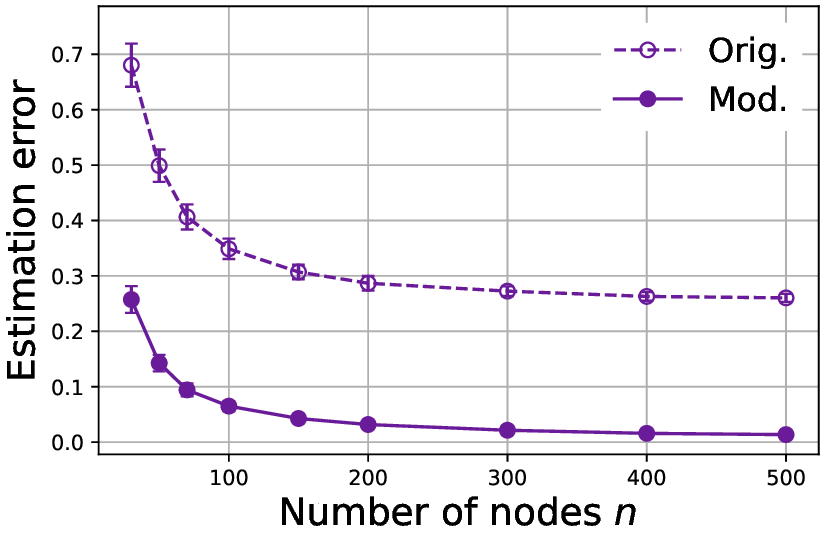}
        \caption{MV}
    \end{subfigure}
    \hfill
    \begin{subfigure}[t]{0.48\textwidth}
        \centering
        \includegraphics[width=\textwidth]{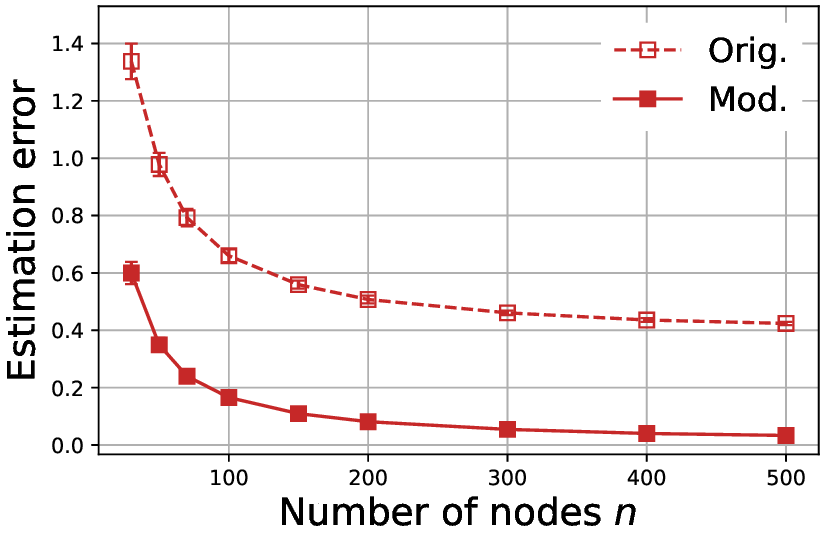}
        \caption{Trace}
    \end{subfigure}

    \vspace{0.5em}

    \begin{subfigure}[t]{0.48\textwidth}
        \centering
        \includegraphics[width=\textwidth]{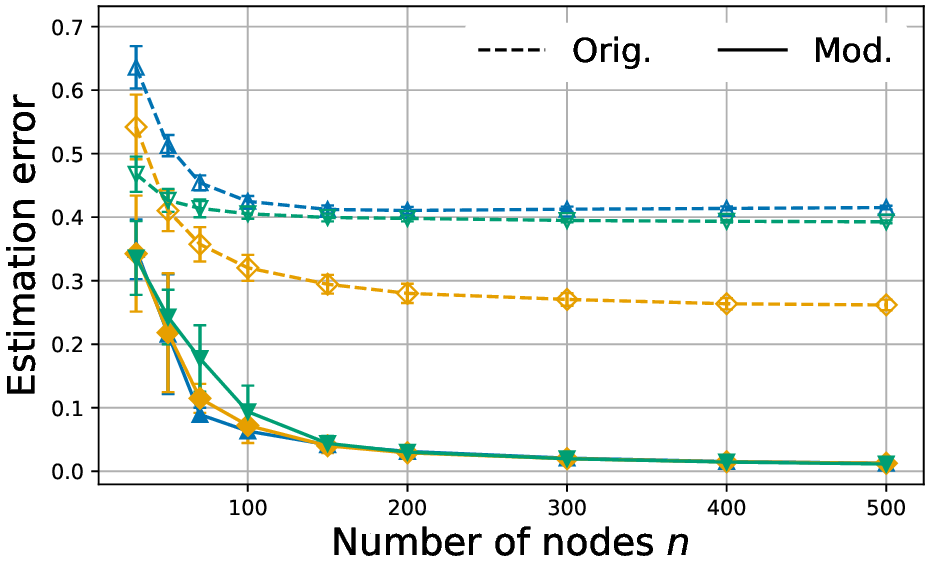}
        \caption{Mode-wise (canonical)}
    \end{subfigure}
    \hfill
    \begin{subfigure}[t]{0.48\textwidth}
        \centering
        \includegraphics[width=\textwidth]{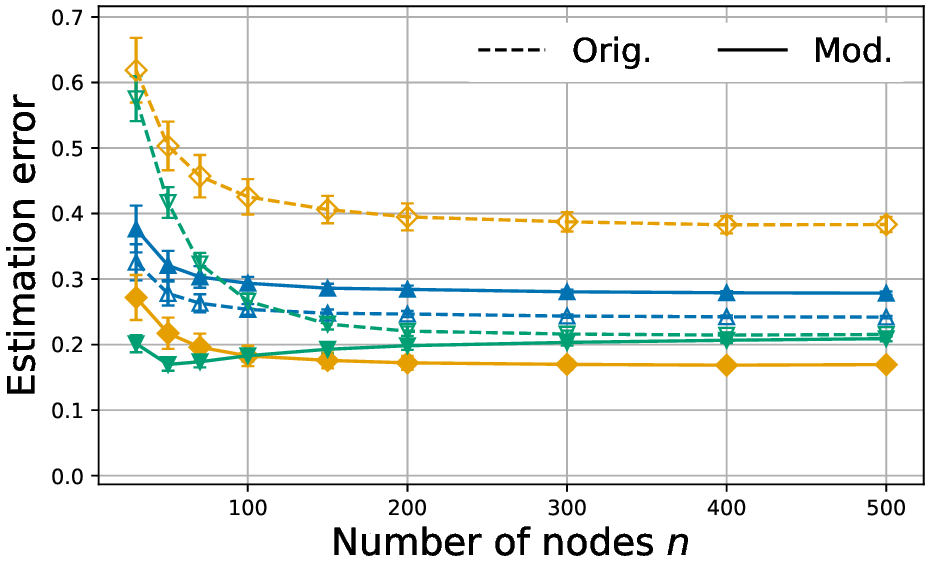}
        \caption{Mode-wise (standard)}
    \end{subfigure}

    \caption{Convergence of the estimation error between the finite sample squared distance matrix $\tilde{\mathbf D}^{(2)}_\rho$ and the estimated squared distance matrix $\hat{\mathbf D}^{(2)}_\rho$, comparing the original UASE and the proposed modified UASE. This evaluates geometric distortion induced by the embedding across MV, trace, and mode-wise distances under canonical and standard bases.}
    \label{fig:latent-vs-embed}
\end{figure}

\begin{figure}[t]
    \centering

    \begin{subfigure}[t]{0.48\textwidth}
        \centering
        \includegraphics[width=\textwidth]{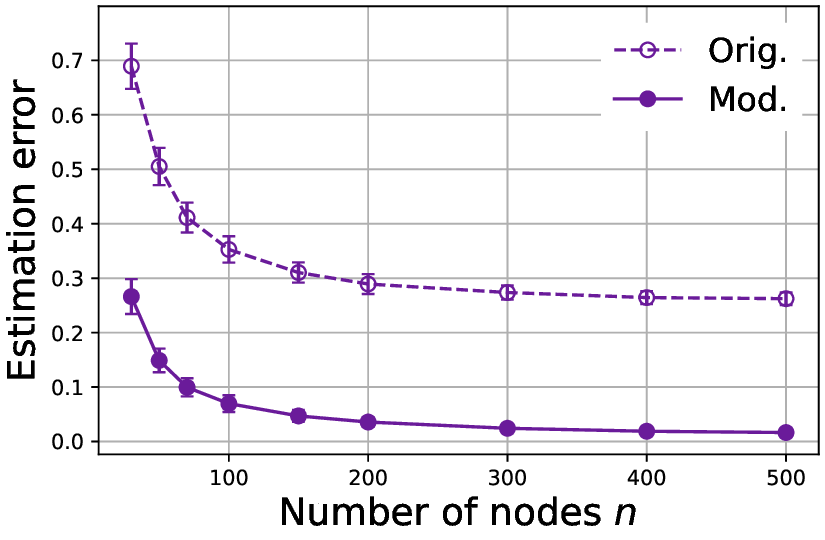}
        \caption{MV}
    \end{subfigure}
    \hfill
    \begin{subfigure}[t]{0.48\textwidth}
        \centering
        \includegraphics[width=\textwidth]{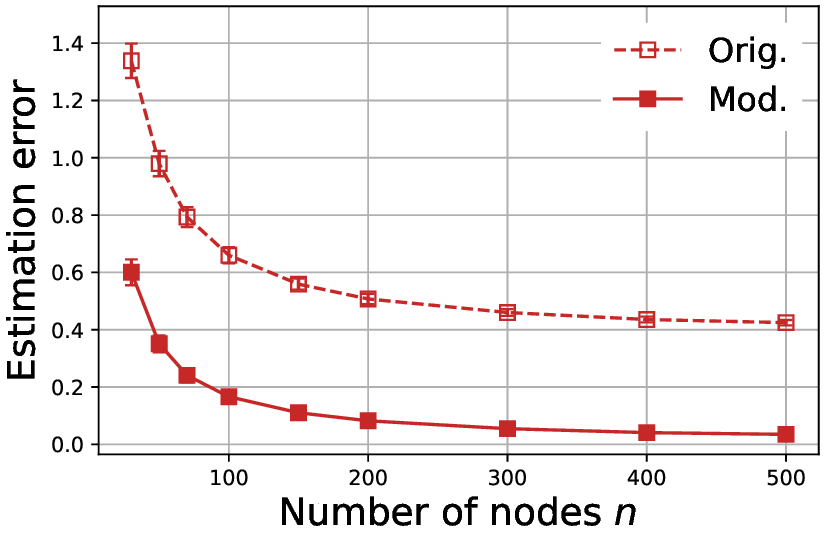}
        \caption{Trace}
    \end{subfigure}

    \vspace{0.5em}

    \begin{subfigure}[t]{0.48\textwidth}
        \centering
        \includegraphics[width=\textwidth]{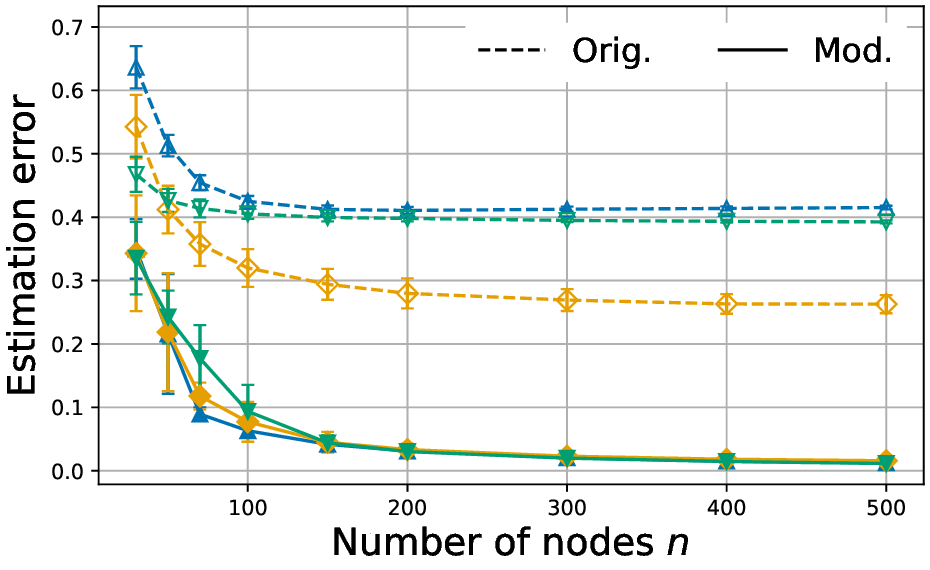}
        \caption{Mode-wise (canonical)}
    \end{subfigure}
    \hfill
    \begin{subfigure}[t]{0.48\textwidth}
        \centering
        \includegraphics[width=\textwidth]{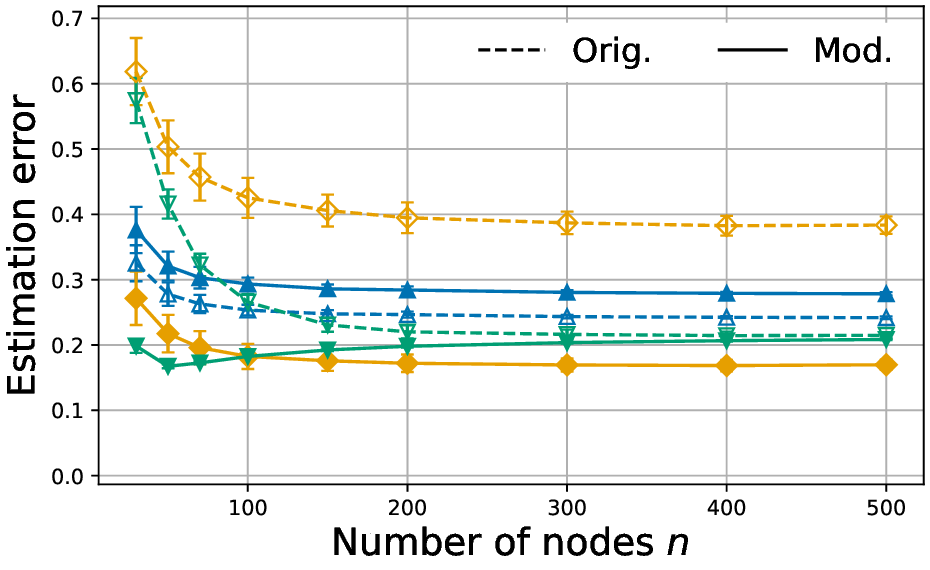}
        \caption{Mode-wise (standard)}
    \end{subfigure}

    \caption{Convergence of the estimation error between the population squared distance matrix $\mathbf D^{(2)}_\rho$ and the estimated squared distance matrix $\hat{\mathbf D}^{(2)}_\rho$, comparing the original UASE and the proposed modified UASE. This evaluates recovery of intrinsic geometry across MV, trace, and mode-wise distances under canonical and standard bases.}
    \label{fig:population-vs-embed}
\end{figure}

\begin{figure}[t]
    \centering

    \begin{subfigure}[t]{0.48\textwidth}
        \centering
        \includegraphics[width=\textwidth]{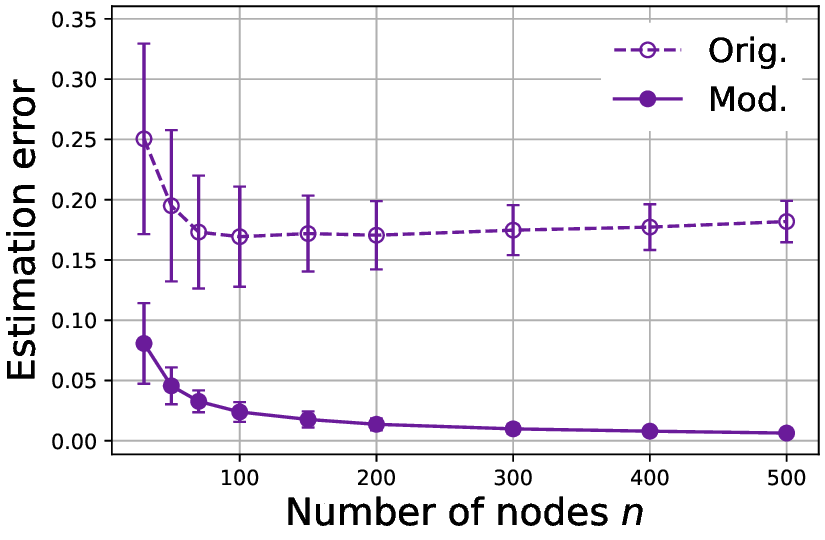}
        \caption{MV}
    \end{subfigure}
    \hfill
    \begin{subfigure}[t]{0.48\textwidth}
        \centering
        \includegraphics[width=\textwidth]{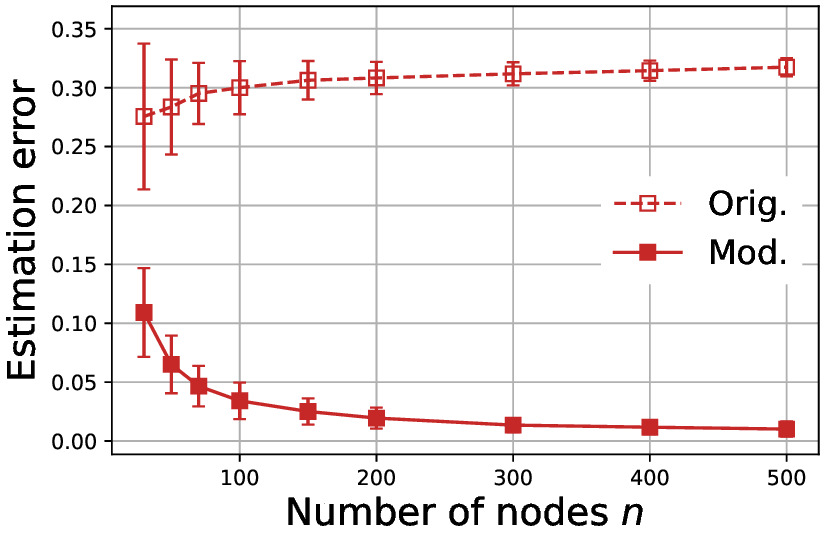}
        \caption{Trace}
    \end{subfigure}

    \vspace{0.5em}

    \begin{subfigure}[t]{0.48\textwidth}
        \centering
        \includegraphics[width=\textwidth]{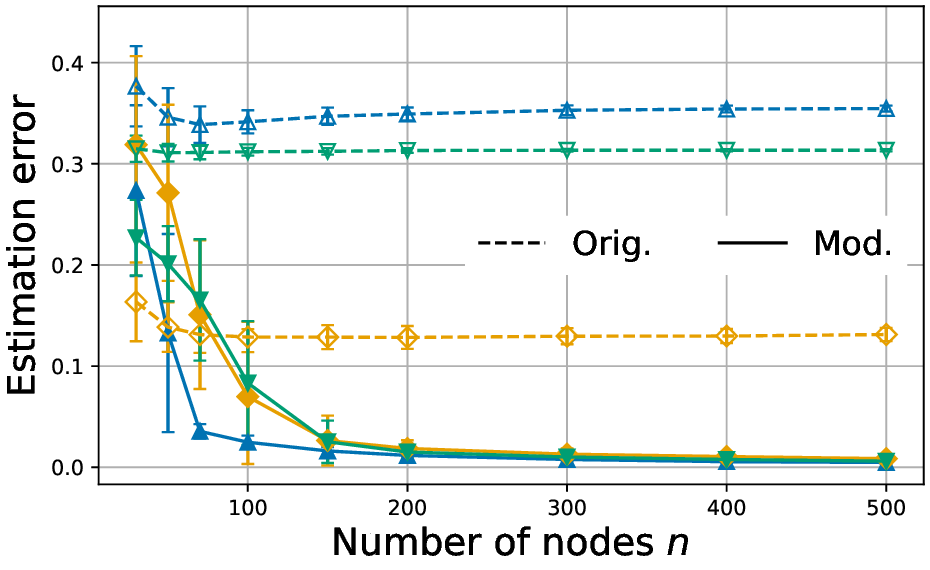}
        \caption{Mode-wise (canonical)}
    \end{subfigure}
    \hfill
    \begin{subfigure}[t]{0.48\textwidth}
        \centering
        \includegraphics[width=\textwidth]{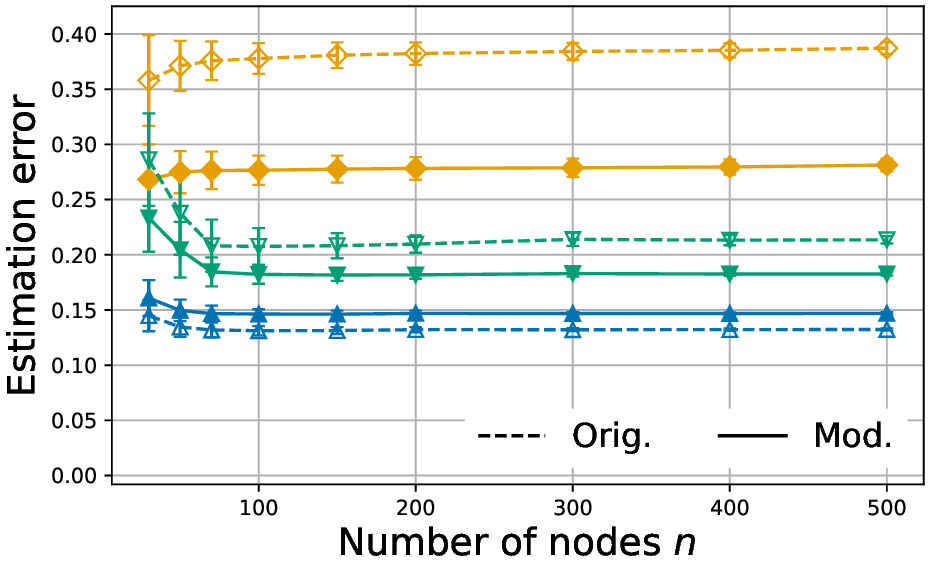}
        \caption{Mode-wise (standard)}
    \end{subfigure}

    \caption{Convergence of trajectory estimation error between the population and estimated trajectories, comparing the original UASE and the proposed modified UASE across MV, trace, and mode-wise cases under canonical and standard bases.}
    \label{fig:trajectory}
\end{figure}

\subsection{Mode Decomposition of Trace Variation}
\label{app:exp-trajectories}

Theorem~\ref{thm:tv-decomposition-paper} establishes an exact decomposition at the distance level. Since the trajectories are constructed to approximate these distances, we expect a corresponding decomposition to also appear at the trajectory-level. We verify this using Dataset 1.

We measure temporal change at each time point by taking the difference between consecutive trajectory values and then taking its squared norm. We visualize these quantities in two ways: (i) line plots for the trace trajectory with trajectory dimensions \(c=1\) and \(c=3\), and (ii) a stacked plot for the mode-wise trajectories, where the contributions of the individual modes are aggregated. The case \(c=3\) corresponds to the intrinsic dimensionality required to reconstruct the population distance in Dataset 1.

Figure~\ref{fig:tv-decomposition} shows the results. The temporal variation captured by the trace trajectory, shown in the line plots, closely aligns with the aggregated variation of the mode-wise trajectories, shown in the bar chart. In particular, we observe that
\[
\|\psi_{\mathrm{tr}}(t+1) - \psi_{\mathrm{tr}}(t)\|_2^2
\;\approx\;
\sum_{k=1}^3 \|\psi_k(t+1) - \psi_k(t)\|_2^2,
\]
\noindent indicating that the variation of the trace trajectory is well explained by the sum of the mode-wise variations. This provides empirical evidence that the trajectory construction preserves not only distances, but also their decomposition structure.

When only the 1D representation is used, as shown by the solid line, the magnitude of variation decreases, although the overall temporal pattern remains broadly consistent with the sum of the mode-wise variations. By contrast, the 3-dimensional representation exhibits much closer quantitative agreement with the sum of the mode-wise variations. This confirms that the proposed construction preserves a coherent decomposition of temporal variation at the trajectory-level, consistent with the exact decomposition at the distance level.

\begin{figure}[t]
    \centering

    \begin{subfigure}[t]{0.3\textwidth}
        \centering
        \includegraphics[width=\textwidth]{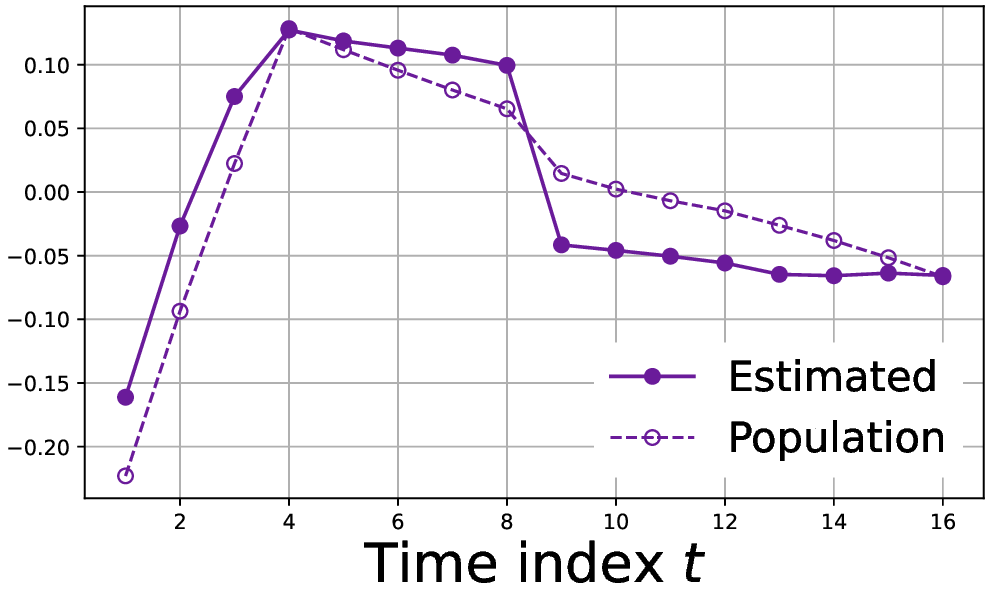}
        \caption{MV}
    \end{subfigure}
    \hfill
    \begin{subfigure}[t]{0.3\textwidth}
        \centering
        \includegraphics[width=\textwidth]{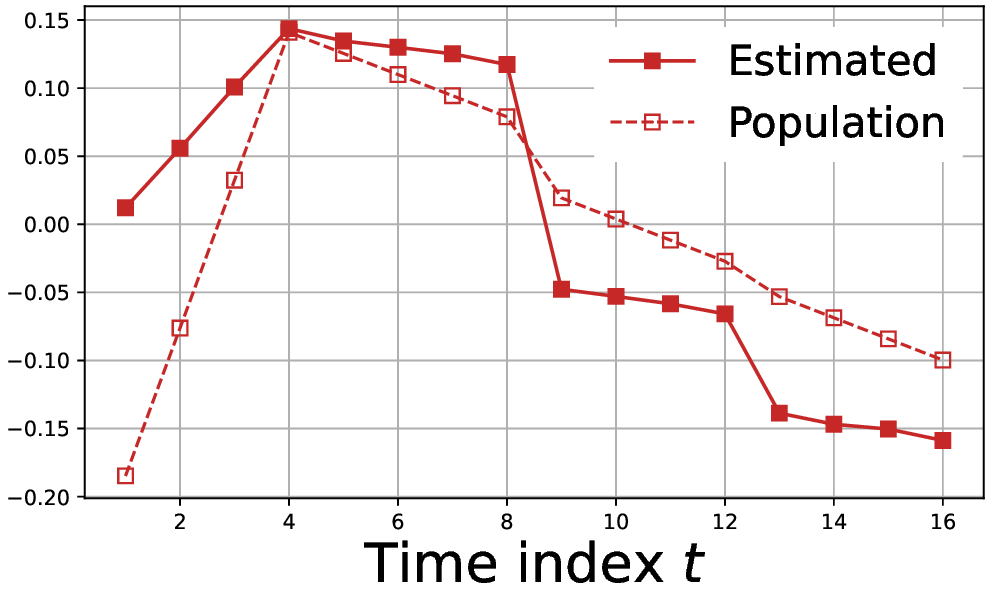}
        \caption{Trace}
    \end{subfigure}
    \hfill
    \begin{subfigure}[t]{0.3\textwidth}
        \centering
        \includegraphics[width=\textwidth]{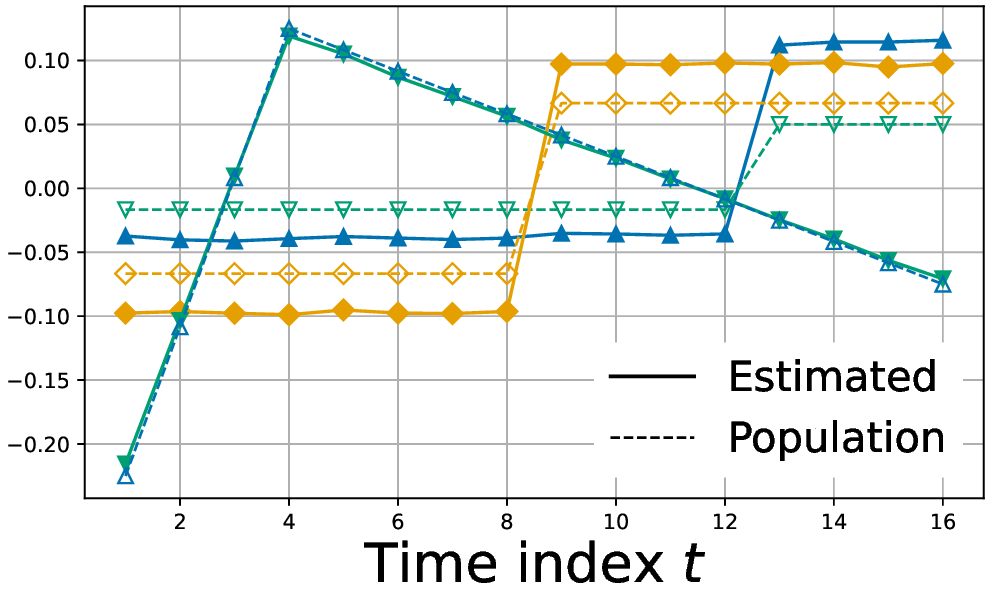}
        \caption{Mode-wise (canonical)}
    \end{subfigure}

    \vspace{0.5em}

    \begin{subfigure}[t]{0.3\textwidth}
        \centering
        \includegraphics[width=\textwidth]{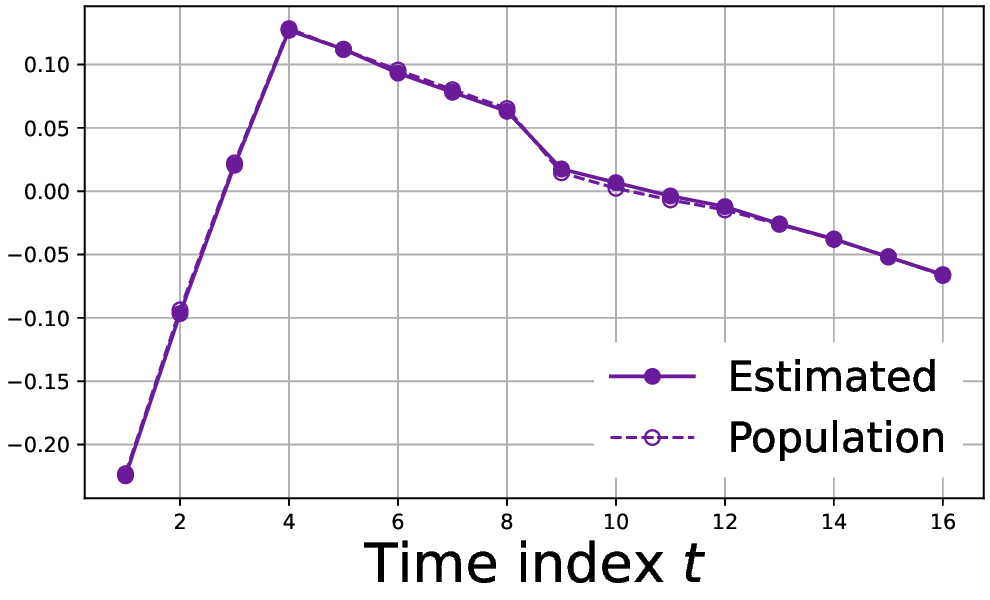}
        \caption{MV}
    \end{subfigure}
    \hfill
    \begin{subfigure}[t]{0.3\textwidth}
        \centering
        \includegraphics[width=\textwidth]{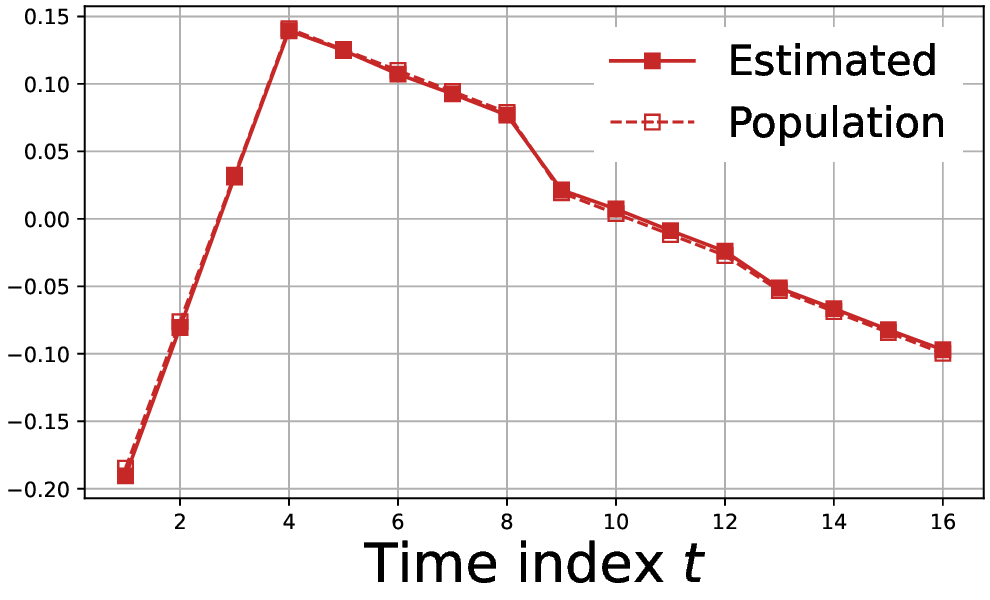}
        \caption{Trace}
    \end{subfigure}
    \hfill
    \begin{subfigure}[t]{0.3\textwidth}
        \centering
        \includegraphics[width=\textwidth]{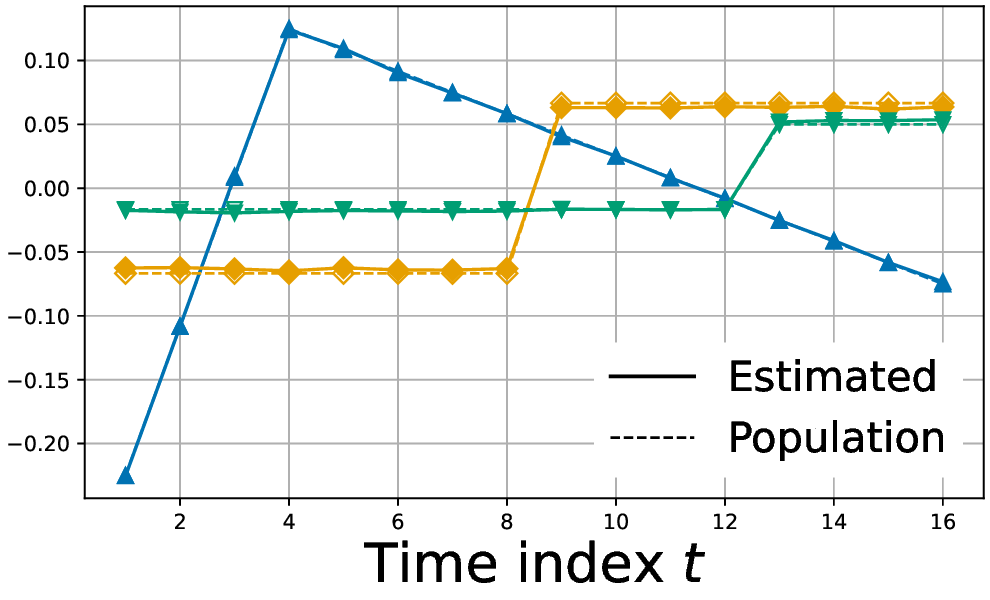}
        \caption{Mode-wise (canonical)}
    \end{subfigure}

    \caption{
    Multiscale trajectories estimated by the original UASE (top row) and the proposed modified UASE (bottom row).
    Columns correspond to MV, trace, and mode-wise distances under a canonical basis.
    The mode-wise trajectories separate distinct temporal components at the mode level, while MV and trace trajectories provide global summaries.
    In contrast, the original UASE exhibits distortions in mode-wise trajectories, including inconsistencies in alignment and ordering across modes. Solid lines denote estimated trajectories, while dashed lines denote the corresponding population trajectories. Note that trajectories are invariant up to sign, and the signs have been adjusted for better visual interpretability.
    }
    \label{fig:multiscale-trajectories}
\end{figure}

\begin{figure}[t]
    \centering
    \includegraphics[width=0.95\textwidth]{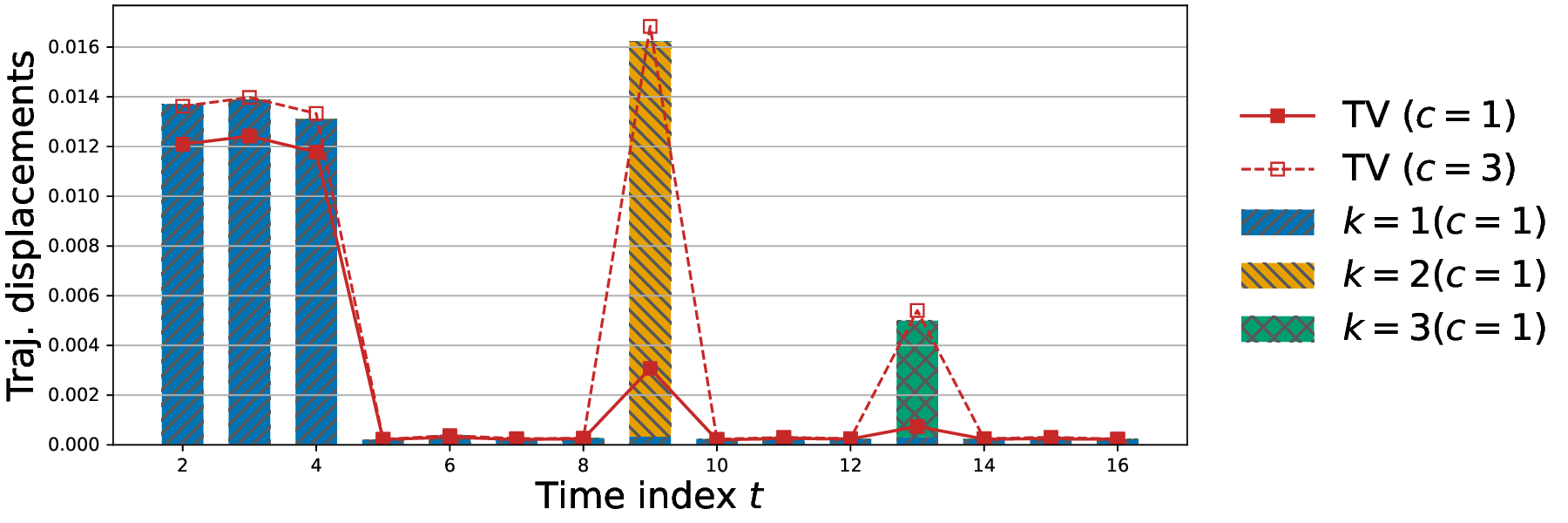}
    \caption{
    Decomposition of temporal variation under the proposed modified UASE.
    Line plots show the squared temporal differences of the trace trajectory for dimensions $c=1$ and $c=3$.
    The stacked plot shows the sum of squared temporal differences of the mode-wise trajectories.
    The 3-dimensional case ($c=3$) closely matches the aggregated mode-wise variation, while the 1D ($c=1$) case underestimates the magnitude but broadly preserves the overall temporal pattern.
    }
    \label{fig:tv-decomposition}
\end{figure}

\subsection{Additional Results for Node Attribution}
\label{app:exp-attribution}

In this subsection, we verify the node attribution results summarized in Theorem~\ref{thm:local-global-main} in the main text. Theorem~\ref{thm:local-global-main} states that the discrepancy between trajectory differences and aggregated node-level attributions is uniformly bounded over all time pairs \((t,s)\) by a term determined by the CMDS residual. To verify this numerically, we aggregate squared attributions across nodes:

\[
\sum_{i=1}^n \hat a_{\rho}(t,s;i)^2.
\]

Figure~\ref{fig:pairwise-attribution} presents scatter plots over all time pairs \((t,s)\) with \(t < s\). The trajectory-level variation and the corresponding aggregated node-level attribution exhibit strong agreement, indicating that variations measured in trajectory space are well explained by local node-level contributions.

\begin{figure}[t]
    \centering

    \begin{subfigure}[t]{0.48\linewidth}
        \centering
        \includegraphics[width=\linewidth]{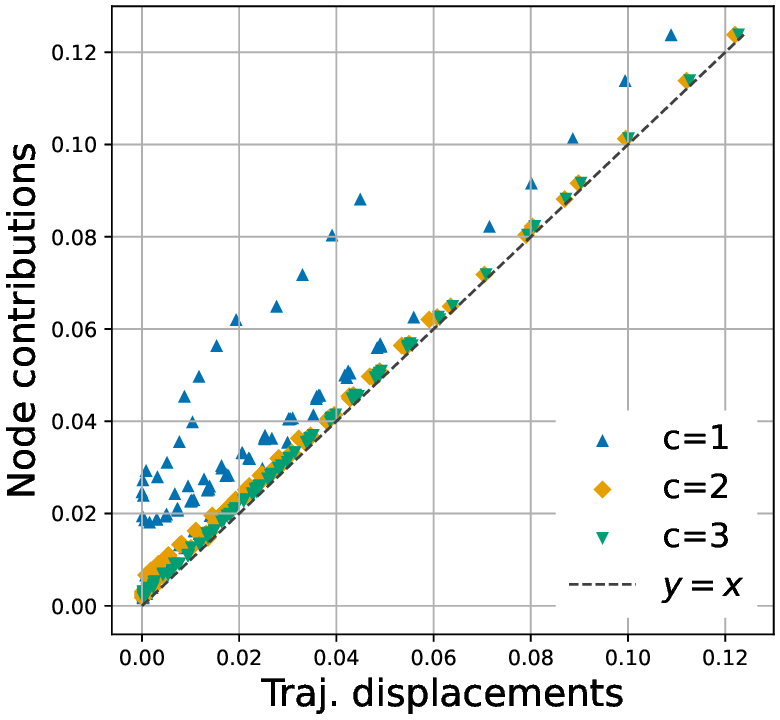}
        \caption{Trace}
    \end{subfigure}
    \hfill
    \begin{subfigure}[t]{0.48\linewidth}
        \centering
        \includegraphics[width=\linewidth]{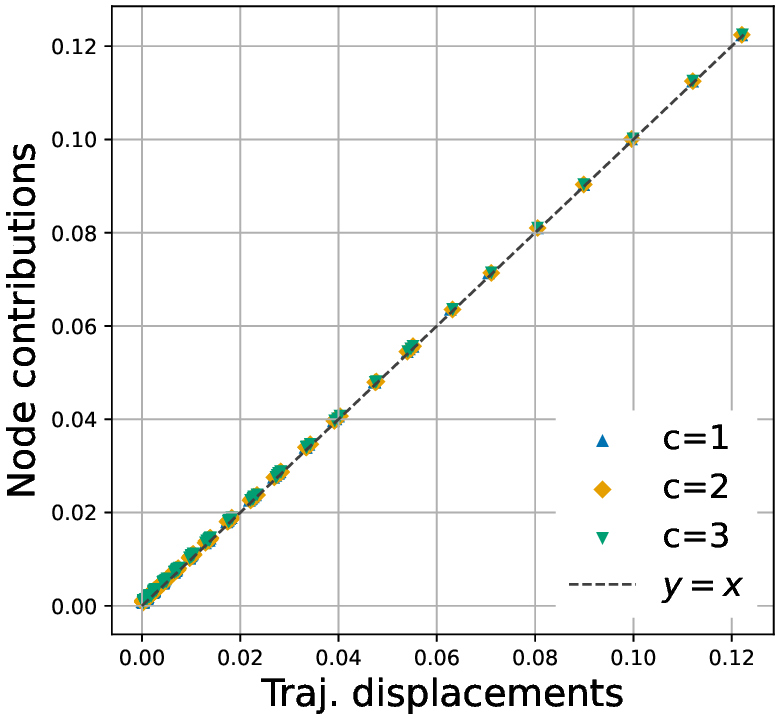}
        \caption{Mode 1}
    \end{subfigure}

    \vspace{0.5em}

    \begin{subfigure}[t]{0.48\linewidth}
        \centering
        \includegraphics[width=\linewidth]{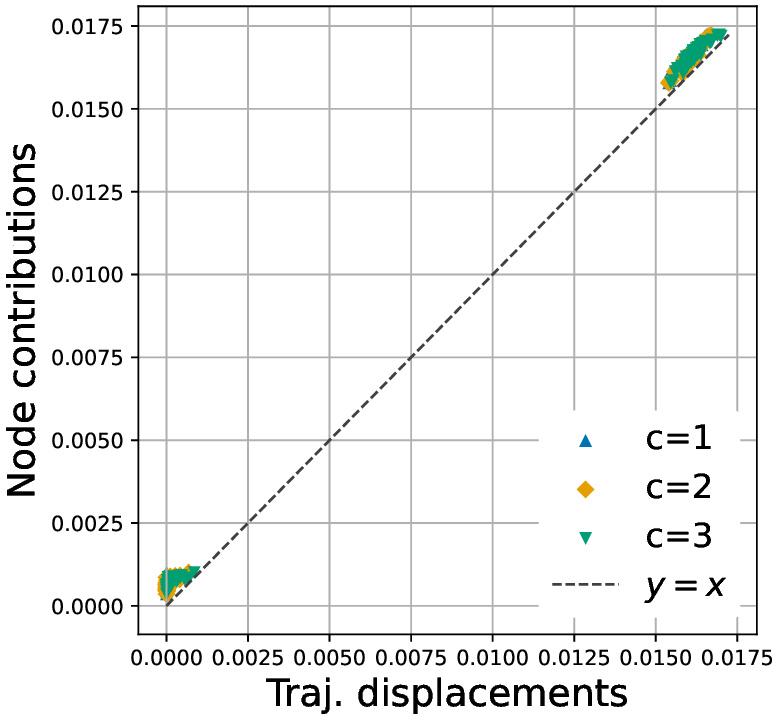}
        \caption{Mode 2}
    \end{subfigure}
    \hfill
    \begin{subfigure}[t]{0.48\linewidth}
        \centering
        \includegraphics[width=\linewidth]{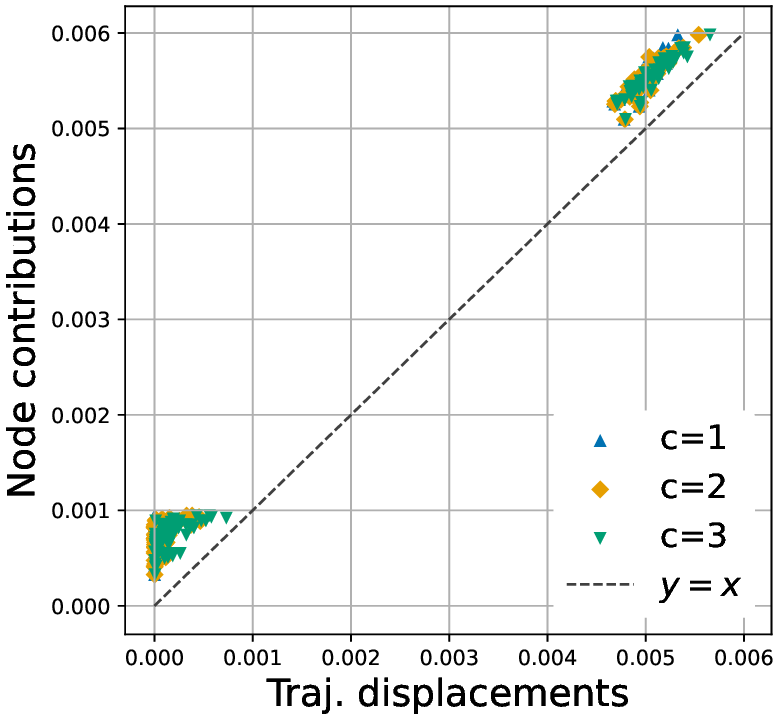}
        \caption{Mode 3}
    \end{subfigure}

    \caption{
    Pairwise comparison between trajectory differences and aggregated node-level attributions over all time pairs $(t,s)$ with $t < s$.
    Each panel corresponds to a different trajectory type (trace or mode-wise), where the mode-wise results are computed under the canonical basis.
    The horizontal axis shows the trajectory-level variation, and the vertical axis shows the corresponding aggregated node-level attribution.
    Results are shown for trajectory dimensions $c=1,2,3$.
    }
    \label{fig:pairwise-attribution}
\end{figure}

We next examine the corresponding aggregate discrepancy across all time pairs (Proposition~\ref{prop:local-global-app}).
Figure~\ref{fig:aggregate-attribution} compares the left-hand side
with this upper bound as a function of \(c\). As \(c\) increases, the
alignment between trajectory differences and aggregated node-level
attributions improves, consistent with the bound.
\begin{figure}[t]
    \centering

    \begin{subfigure}[t]{0.48\linewidth}
        \centering
        \includegraphics[width=\linewidth]{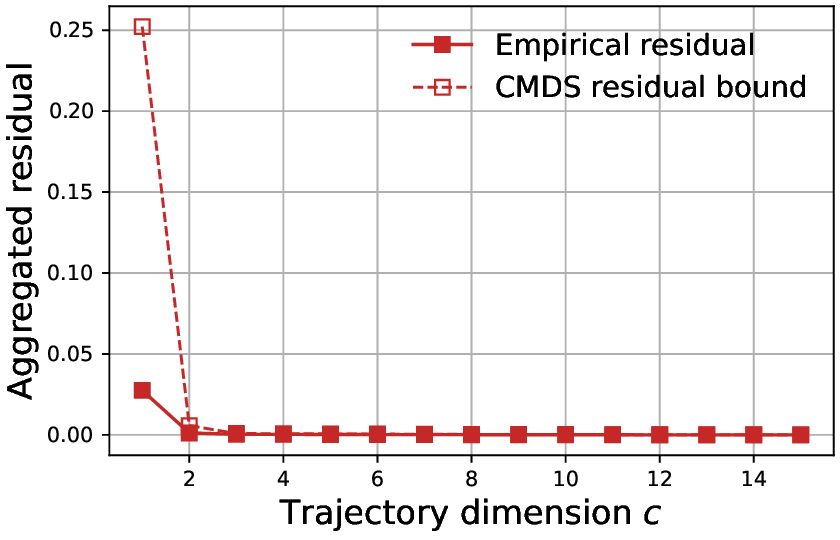}
        \caption{Trace}
    \end{subfigure}
    \hfill
    \begin{subfigure}[t]{0.48\linewidth}
        \centering
        \includegraphics[width=\linewidth]{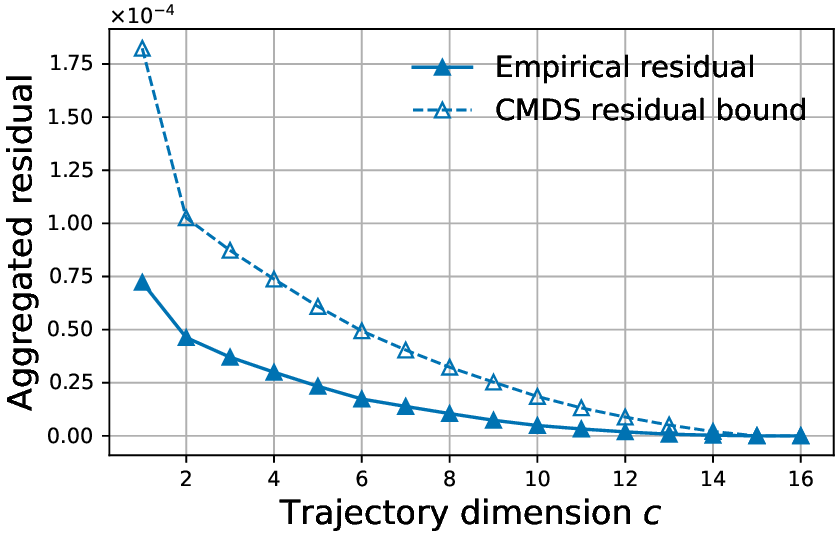}
        \caption{Mode 1}
    \end{subfigure}

    \vspace{0.5em}

    \begin{subfigure}[t]{0.48\linewidth}
        \centering
        \includegraphics[width=\linewidth]{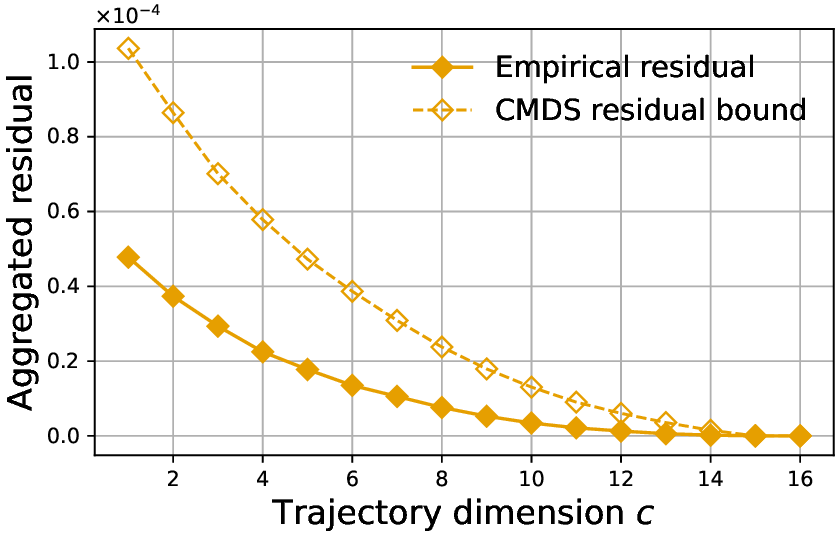}
        \caption{Mode 2}
    \end{subfigure}
    \hfill
    \begin{subfigure}[t]{0.48\linewidth}
        \centering
        \includegraphics[width=\linewidth]{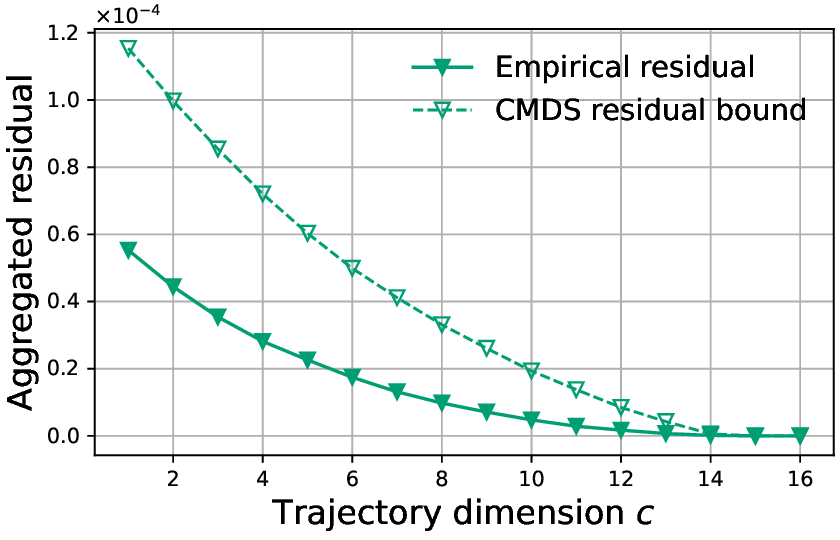}
        \caption{Mode 3}
    \end{subfigure}

    \caption{
    Aggregated discrepancy between trajectory-level variation and node-level attribution as a function of the trajectory dimension $c$.
    The empirical residual (solid line) is compared with the theoretical upper bound (dashed line) given immediately above.
    Each panel corresponds to a different trajectory type (trace or mode-wise), where the mode-wise results are computed under the canonical basis.
    }
    \label{fig:aggregate-attribution}
\end{figure}

Having verified these identities numerically, we next examine whether the recovered node-level attributions also reflect the expected structural patterns. We use the estimated embeddings to visualize node-level contributions to the trace and mode-wise trajectories, as shown in Figure~\ref{fig:combined-attribution}. Each node is positioned using a t-SNE projection of the 3-dimensional modified UASE embedding. At each time point \(t\), the size of a node marker is proportional to its squared attribution:
\[
\hat a_{\rho}(t-1,t;i)^2
\]

\noindent For the mode-wise visualizations, color encodes the sign of the attribution, with red indicating positive values and blue indicating negative values, whereas for trace all nodes are shown in gray.

The recovered patterns closely match the synthetic ground truth. For mode \(1\), all nodes contribute with signs corresponding to increases or decreases in global connectivity strength. For mode \(2\), nodes in communities \(1\) and \(2\) exhibit signs opposite to those in community \(3\). For mode \(3\), communities \(1\) and \(2\) contribute with opposite signs, reflecting their pairwise separation. By contrast, the trace attribution aggregates these effects. Nodes with large contributions in one or more modes appear prominent, indicating that the trace trajectory reflects the combined magnitude of the mode-wise variations.

Taken together, these results demonstrate that node-level attributions provide a faithful decomposition of trajectory variation. Moreover, they capture meaningful structural patterns at the node-level while aggregating consistently to explain the variations observed in trajectory space, in agreement with the theoretical guarantees.

\begin{figure}[t]
    \centering

    \begin{subfigure}[t]{0.48\linewidth}
        \centering
        \includegraphics[width=\linewidth]{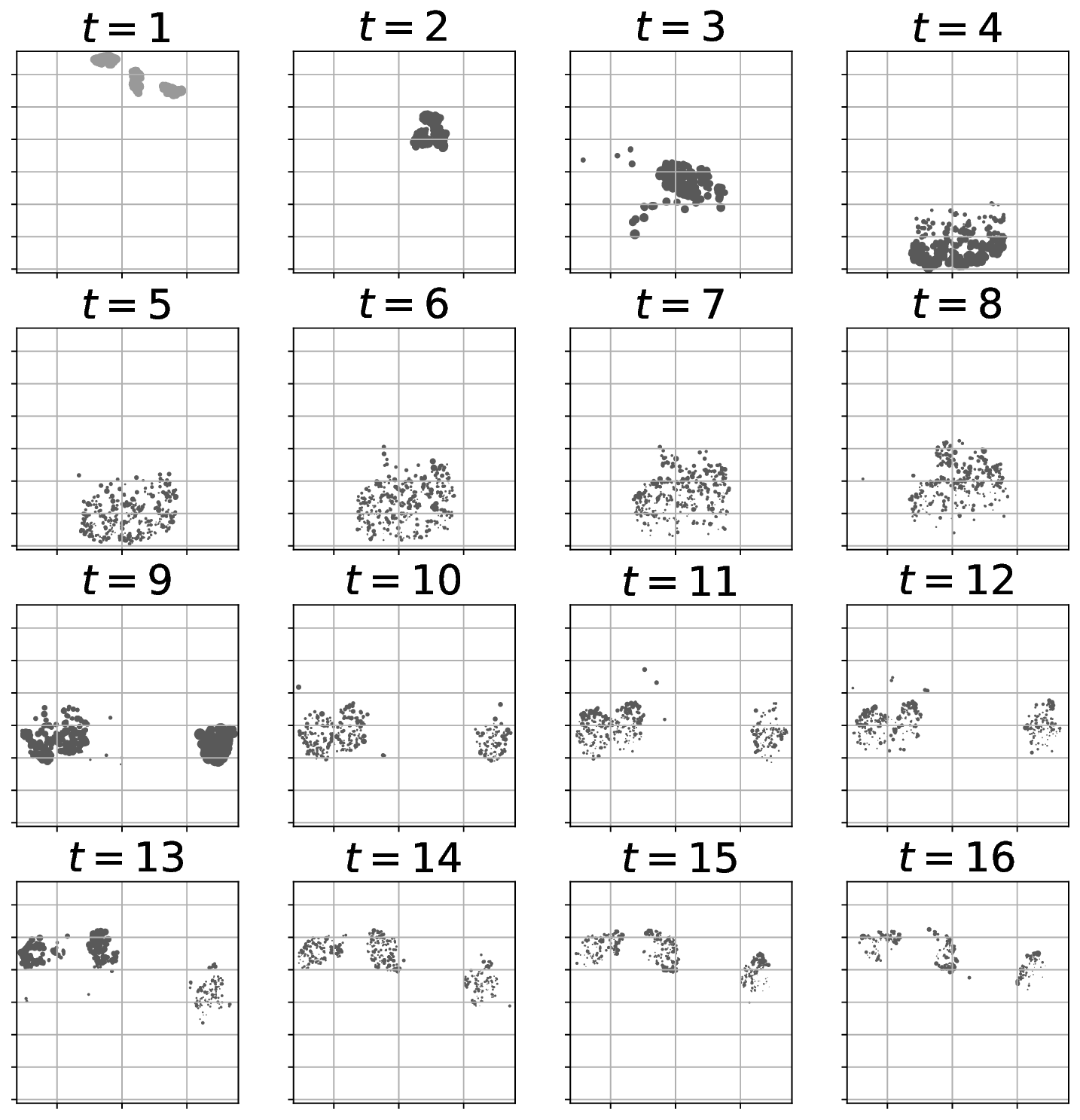}
        \caption{Trace variation}
        \label{fig:tv-attribution}
    \end{subfigure}
    \hfill
    \begin{subfigure}[t]{0.48\linewidth}
        \centering
        \includegraphics[width=\linewidth]{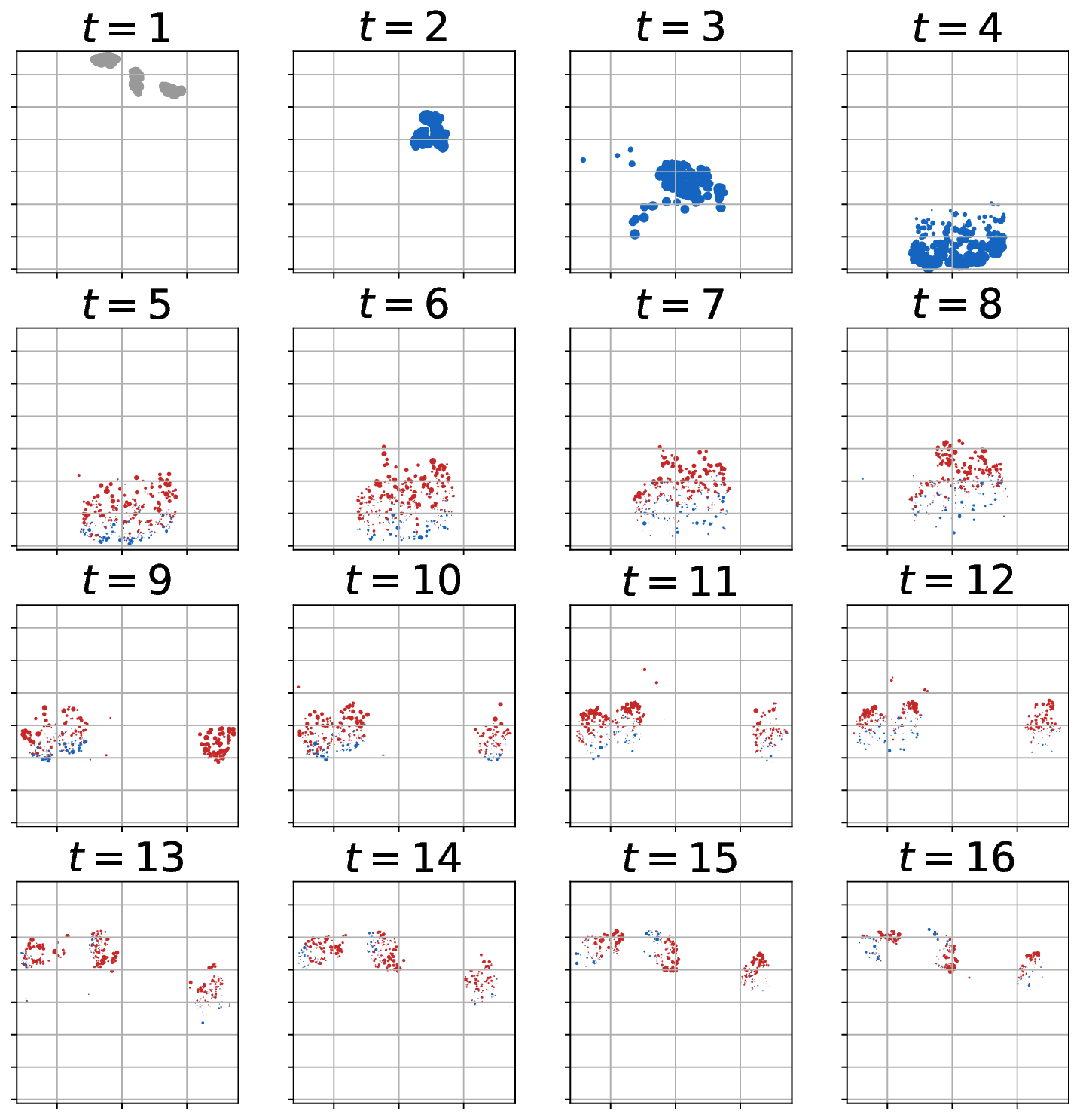}
        \caption{Mode 1}
        \label{fig:mode1-attribution}
    \end{subfigure}

    \vspace{0.5em}

    \begin{subfigure}[t]{0.48\linewidth}
        \centering
        \includegraphics[width=\linewidth]{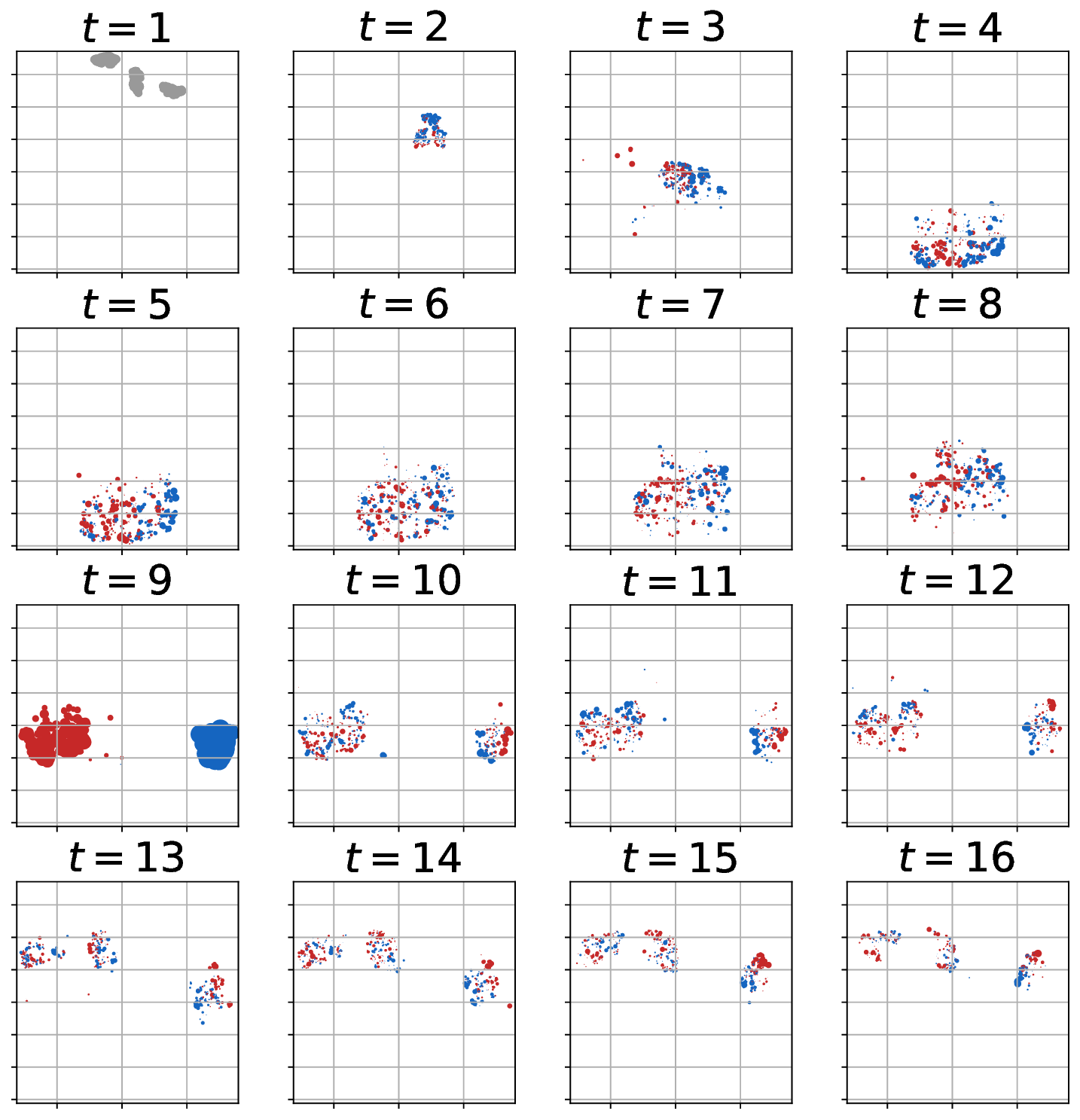}
        \caption{Mode 2}
        \label{fig:mode2-attribution}
    \end{subfigure}
    \hfill
    \begin{subfigure}[t]{0.48\linewidth}
        \centering
        \includegraphics[width=\linewidth]{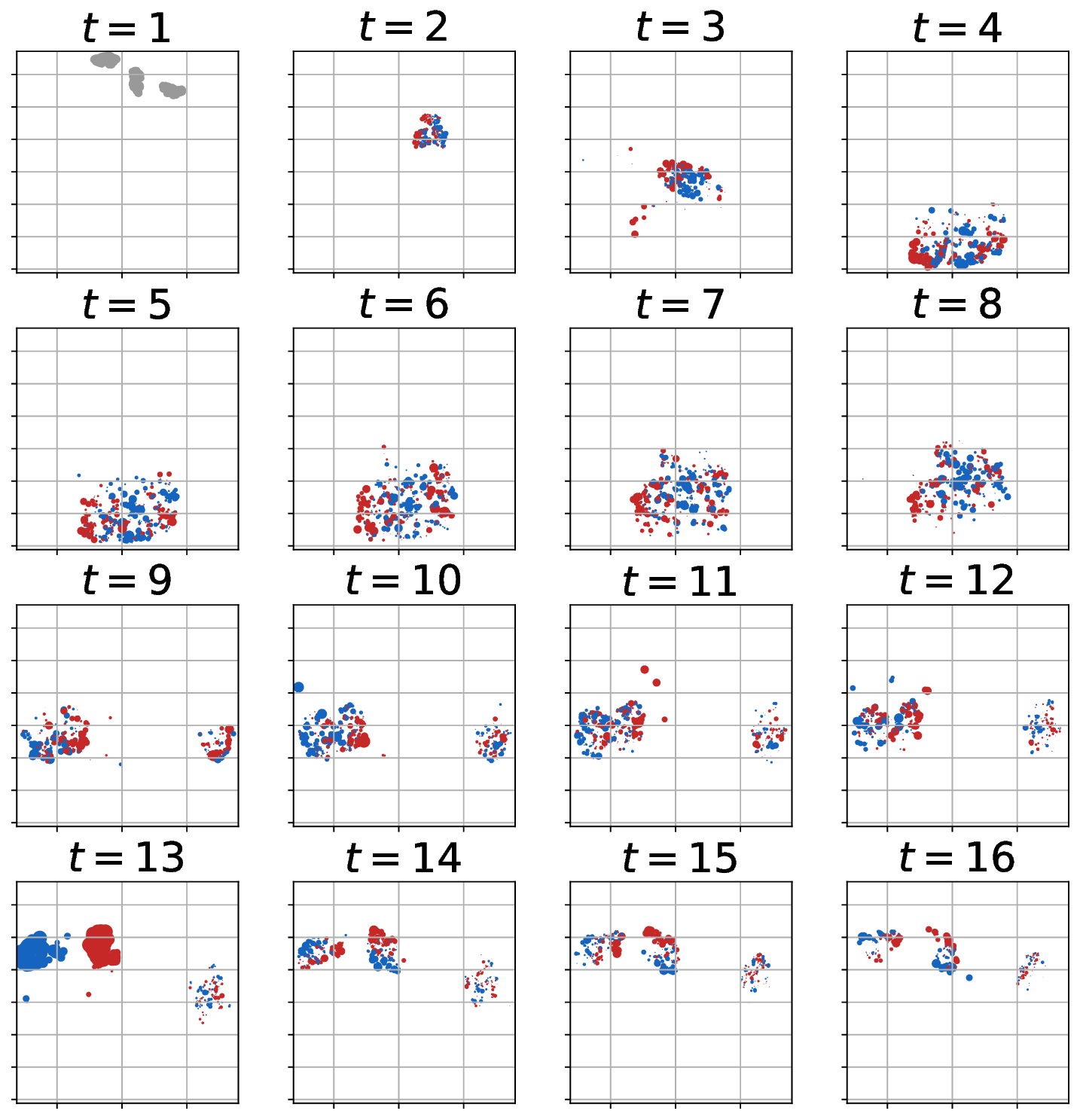}
        \caption{Mode 3}
        \label{fig:mode3-attribution}
    \end{subfigure}

    \caption{Node-level attribution visualizations. 
    Node size is proportional to the corresponding magnitude. 
    For canonical modes, color encodes sign (red: positive, blue: negative). 
    At $t=1$, node sizes are uniform and colors are gray since no value is defined.}
    \label{fig:combined-attribution}
\end{figure}

\subsection{Localization Theorem}
\label{app:exp-cp}

We next empirically verify the localization guarantee in Theorem~\ref{thm:cp-localization}.  Dataset~1 is designed so that each population mode-wise trajectory contains a single change point, either of order \(0\) or order \(1\).  Thus, for each mode, the population target is a well defined 1D change point. We estimate the corresponding change point time by minimizing the piecewise-trajectory fitting criterion introduced in the main text, and compare the proposed modified UASE with the original UASE.

Figure~\ref{fig:cp-result} summarizes the results. The top row shows detection accuracy, while the bottom row shows localization error. The dashed lines indicate that the original UASE fails to detect the ground truth change points for Mode \(1\) and Mode \(3\), as shown in Figure~\ref{fig:cp-result} (a), (c), (d), and (f). Although it appears to achieve higher accuracy and lower localization error for Mode \(2\), strong performance in only a subset of modes is insufficient when the goal is to accurately localize the full set of mode-specific change points. By contrast, the solid lines show that the proposed modified UASE recovers the change points for all modes.

These results demonstrate that accurate change point detection critically depends on faithful recovery of mode-wise geometry across all modes. The original UASE suffers from geometric distortion due to the unresolved \(\mathrm{GL}(d)\) ambiguity, which leads to inconsistent performance across modes. By resolving this ambiguity, the proposed method achieves stable convergence and enables reliable localization of mode specific temporal discontinuities, with detection accuracy converging to \(1\) and localization error converging to \(0\) across all modes.

\begin{figure}[t]
    \centering

    \begin{subfigure}[t]{0.32\textwidth}
        \centering
        \includegraphics[width=\textwidth]{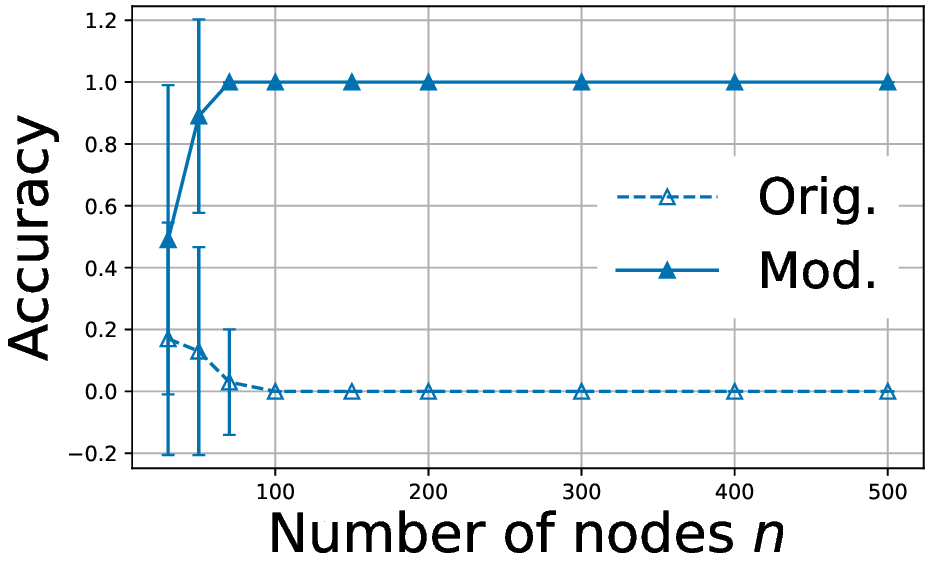}
        \caption{Mode 1 (1st order, $t_1^*=4$)}
    \end{subfigure}
    \hfill
    \begin{subfigure}[t]{0.32\textwidth}
        \centering
        \includegraphics[width=\textwidth]{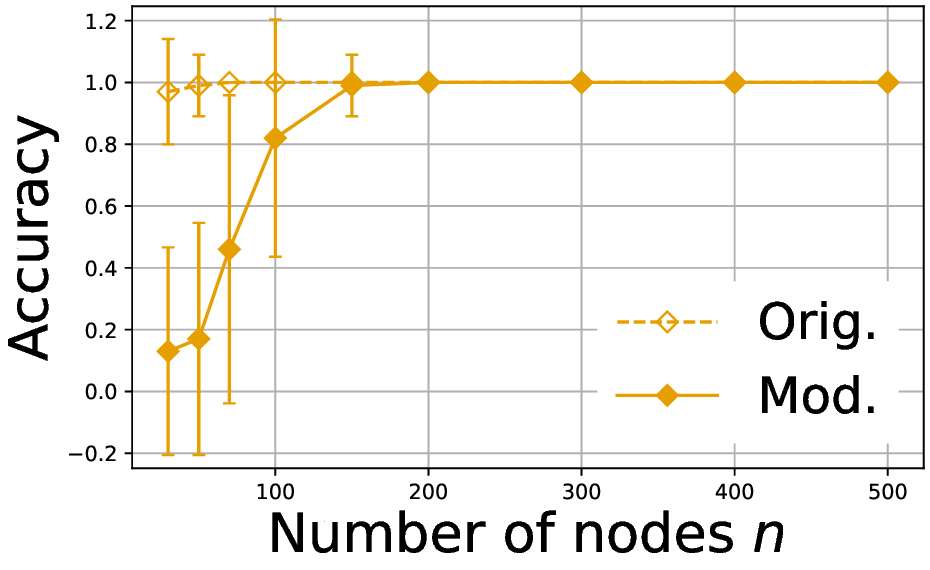}
        \caption{Mode 2 (0th order, $t_2^*=9$)}
    \end{subfigure}
    \hfill
    \begin{subfigure}[t]{0.32\textwidth}
        \centering
        \includegraphics[width=\textwidth]{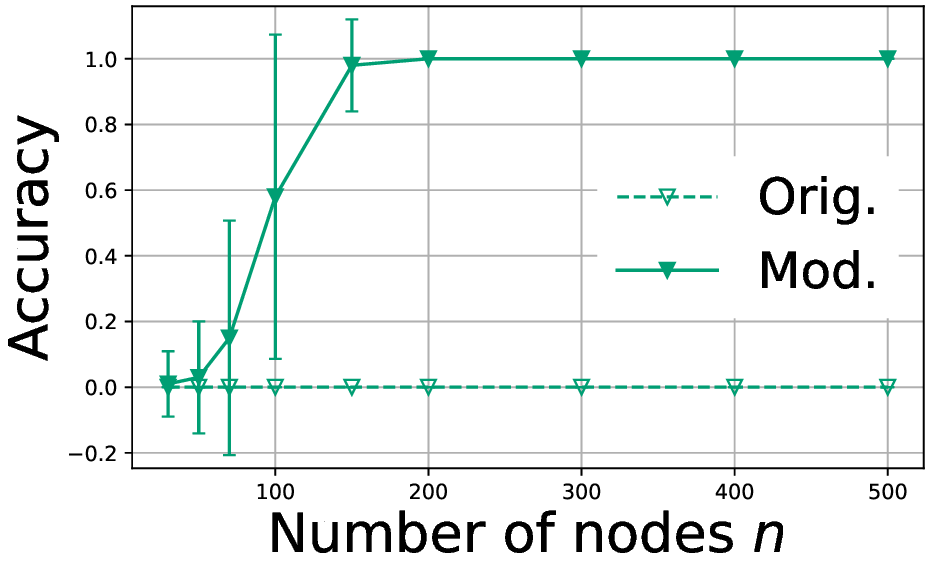}
        \caption{Mode 3 (0th order, $t_3^*=13$)}
    \end{subfigure}

    \vspace{0.8em}

    \begin{subfigure}[t]{0.32\textwidth}
        \centering
        \includegraphics[width=\textwidth]{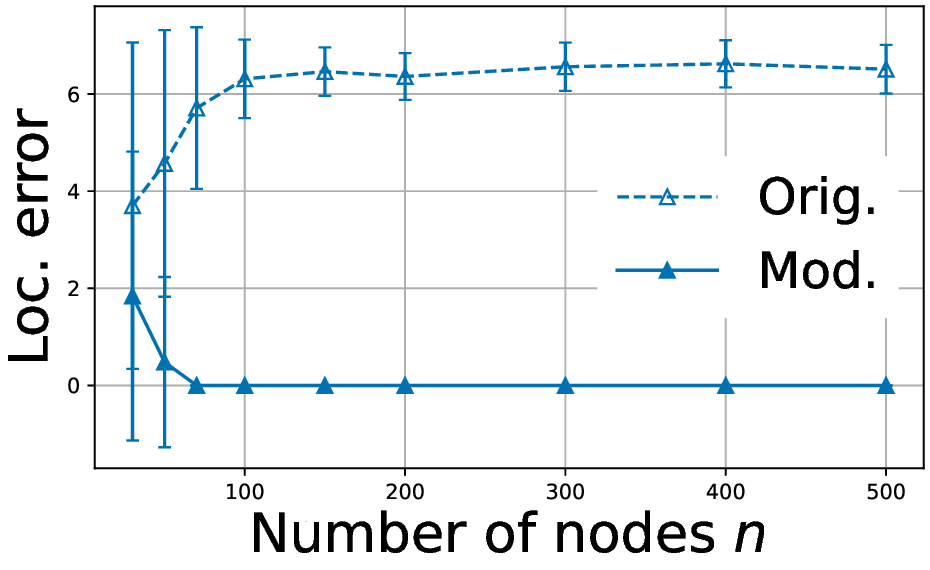}
        \caption{Mode 1 (1st order, $t_1^*=4$)}
    \end{subfigure}
    \hfill
    \begin{subfigure}[t]{0.32\textwidth}
        \centering
        \includegraphics[width=\textwidth]{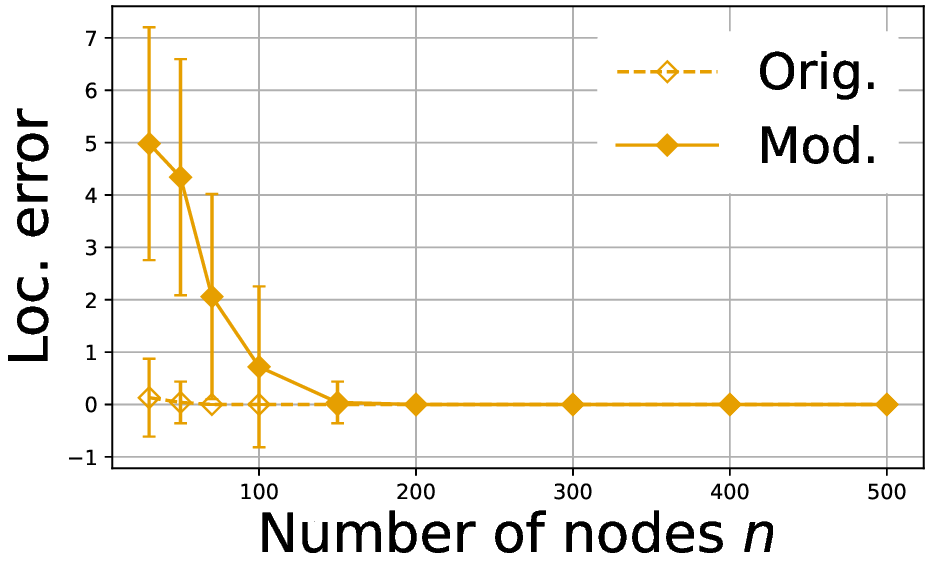}
        \caption{Mode 2 (0th order, $t_2^*=9$)}
    \end{subfigure}
    \hfill
    \begin{subfigure}[t]{0.32\textwidth}
        \centering
        \includegraphics[width=\textwidth]{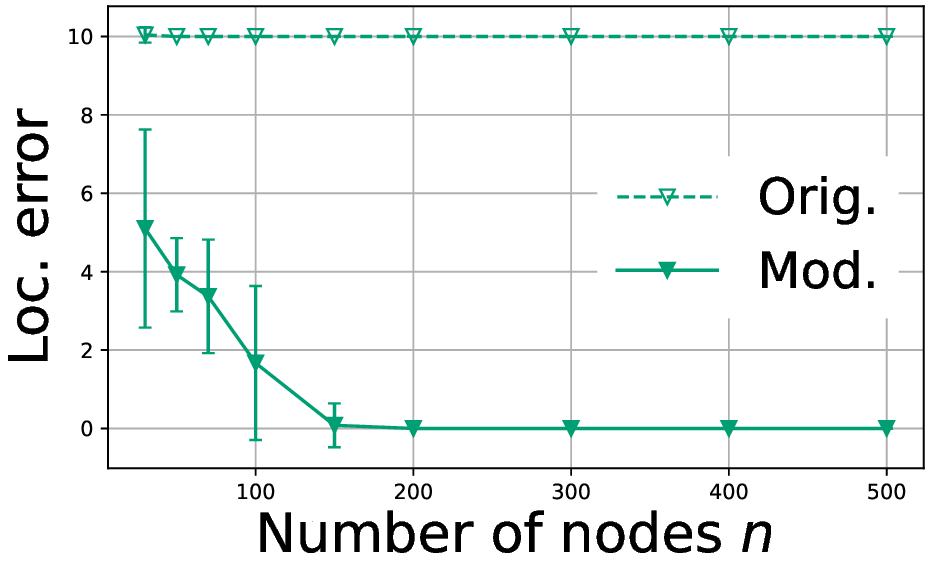}
        \caption{Mode 3 (0th order, $t_3^*=13$)}
    \end{subfigure}

    \caption{
    Change point detection performance for each mode and evaluation metric.
    The top row shows detection accuracy, and the bottom row shows localization error $|\hat{t}_k - t_k^*|$.
    Each column corresponds to a mode-specific ground-truth change point.
    }
    \label{fig:cp-result}
\end{figure}

\subsection{Additional Results for the Effect of Overestimating the Embedding Dimension}
\label{app:overestimating}

In the main text, we assume that the embedding dimension $d_{\mathrm{emb}}$ is equal to the true latent dimension $d_{\mathrm{lat}}$. In practice, however, the latent dimension is typically unknown and must be estimated or selected, potentially resulting in a misspecified embedding dimension. In particular, overestimating the embedding dimension can introduce noisy components that do not capture the underlying latent structure. In this section, we study the effect of choosing $d_{\mathrm{emb}} > d_{\mathrm{lat}}$.

We briefly note that when $d_{\mathrm{emb}} < d_{\mathrm{lat}}$, the embedding induces a spectral truncation of the latent position matrix. If the contribution from singular values beyond the top $d_{\mathrm{emb}}$ components is small, then the resulting trajectories remain well estimated. Moreover, the $d_{\mathrm{emb}}$ mode-wise trajectories correspond to the population trajectories associated with the leading modes. We do not pursue this case further, as underestimation may discard relevant signal components, whereas overestimation only introduces additional error, which we show below to be asymptotically negligible for trajectory estimation.

A misspecified embedding dimension can affect both the estimated pairwise distances and the resulting trajectories, since noise introduced by the additional dimensions propagates through the estimation procedure. In particular, overestimating the embedding dimension gives rise to additional mode-wise trajectories that do not correspond to any population trajectories.

At a high level, overestimating the embedding dimension has only a limited effect, because the embedding admits a signal to noise decomposition under the modified UASE. The leading $d_{\mathrm{lat}}$ coordinates capture the latent structure, whereas the remaining dimensions are dominated by noise. Consequently, the estimated distances and trajectories constructed from the embeddings are effectively governed by the signal component, with the contribution of the additional dimensions vanishing asymptotically.

To formalize this intuition, we analyze the structure of the embedding and its effect on the second-moment of the displacements. We partition $\hat{\mathbf Y}(t) = (\hat{\mathbf Y}_1(t), \hat{\mathbf Y}_2(t))$ into its first $d_{\mathrm{lat}}$ coordinates and its remaining coordinates. Building on the analysis of \citet{jones2020multilayer}, the modified UASE admits a signal to noise decomposition: there exists a possibly random orthogonal matrix $\mathbf W \in \mathbb O(d_{\mathrm{lat}})$ such that
\[
\sup_{t\in\mathcal T}\|\hat{\mathbf Y}_1(t) - \mathbf Y(t)\mathbf W\|_2 = O(\log n)\qquad \text{a.s.},
\]
and
\[
\sup_{t\in\mathcal T}\|\hat{\mathbf Y}_2(t)\|_2 = O((\log n)^{1/2})\qquad \text{a.s.}
\]
Although these bounds diverge in absolute terms, they are negligible compared to the $O(n^{1/2})$ scale of the latent positions, implying that the leading $d_{\mathrm{lat}}$ dimensions consistently recover the latent position matrix up to orthogonal transformation, while the remaining dimensions are asymptotically dominated by noise.

We now examine the second-moment of displacements, which is used in the main text to capture changes in the latent structure. Let $\hat{\Delta}_1 := \hat{\mathbf Y}_1(t) - \hat{\mathbf Y}_1(s)$ and $\hat{\Delta}_2 := \hat{\mathbf Y}_2(t) - \hat{\mathbf Y}_2(s)$. Then, the second-moment of displacements can be expressed as
\[
\frac{1}{n}(\hat{\mathbf Y}(t)-\hat{\mathbf Y}(s))^\top(\hat{\mathbf Y}(t)-\hat{\mathbf Y}(s))
= \frac{1}{n}
\begin{pmatrix}
   \hat{\Delta}_1^\top \hat{\Delta}_1 & \hat{\Delta}_1^\top \hat{\Delta}_2\\
   \hat{\Delta}_2^\top \hat{\Delta}_1 & \hat{\Delta}_2^\top \hat{\Delta}_2
\end{pmatrix}.
\]
Using the bounds above, we obtain
\[
\frac{1}{n}\hat{\Delta}_1^\top \hat{\Delta}_1 = O(1) \quad \text{a.s.}, \quad
\frac{1}{n}\hat{\Delta}_1^\top \hat{\Delta}_2 = O(n^{-1/2}(\log n)^{1/2}) \quad \text{a.s.}, \quad
\frac{1}{n}\hat{\Delta}_2^\top \hat{\Delta}_2 = O(n^{-1}\log n) \quad \text{a.s.}
\]
Thus, only the top left block remains asymptotically non negligible, and it converges to the population second-moment matrix of dimension $d_{\mathrm{lat}}$, preserving the latent structure.

This implies that the effective dimensionality of the second-moment matrix is asymptotically $d_{\mathrm{lat}}$. Consequently, the mode-wise trajectories corresponding to the latent structure can still be consistently estimated. In contrast, the additional trajectories induced by the excess dimensions have vanishing magnitude and therefore appear nearly flat, reflecting their noise dominated nature.

To illustrate this behavior, we conduct an experiment based on Dataset~1. The resulting trajectories are shown in Figure~\ref{fig:multiscale-trajectories-dim6}. The network has the same latent structure as Dataset~1 and contains $n = 800$ nodes. We set the embedding dimension to $d_{\mathrm{emb}} = 6$, while the true latent dimension is $d_{\mathrm{lat}} = 3$.

The trace and MV trajectories, as well as the leading three mode-wise trajectories, are consistent with the population trajectories. In contrast, the remaining three mode-wise trajectories are nearly flat, reflecting their noise dominated nature. These results confirm that overestimating the embedding dimension introduces additional trajectories that do not carry meaningful signal, while leaving the estimation of the true latent structure largely unaffected.

\begin{figure}[t]
    \centering
    \begin{subfigure}[t]{0.3\textwidth}
        \centering
        \includegraphics[width=\textwidth]{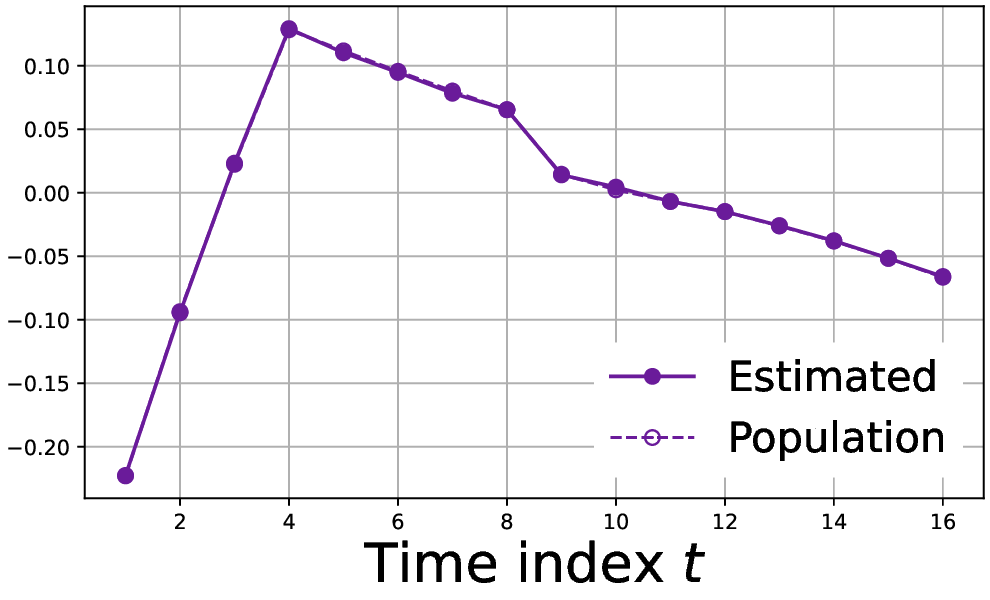}
        \caption{MV}
    \end{subfigure}
    \hfill
    \begin{subfigure}[t]{0.3\textwidth}
        \centering
        \includegraphics[width=\textwidth]{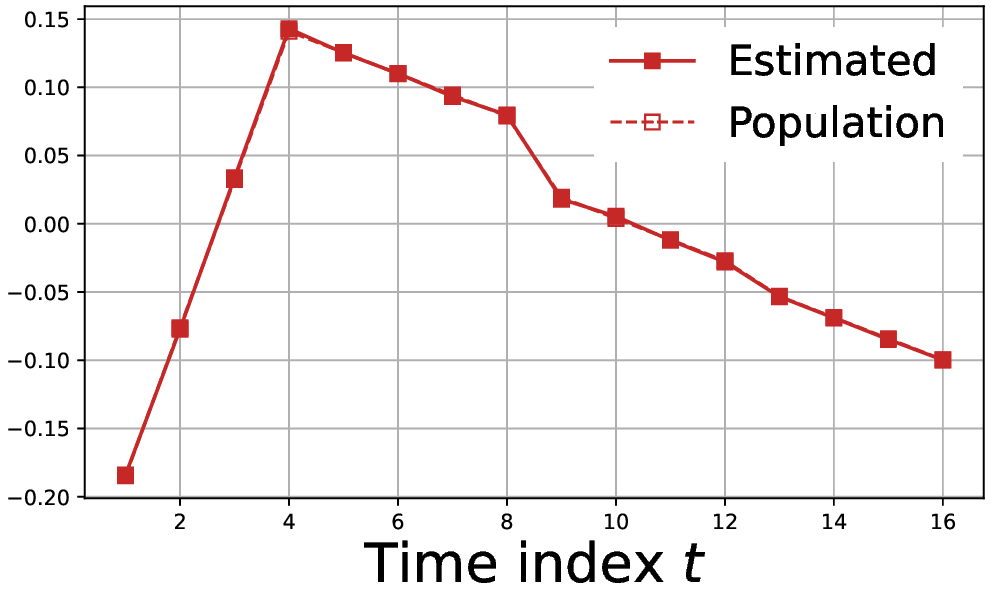}
        \caption{Trace}
    \end{subfigure}
    \hfill
    \begin{subfigure}[t]{0.3\textwidth}
        \centering
        \includegraphics[width=\textwidth]{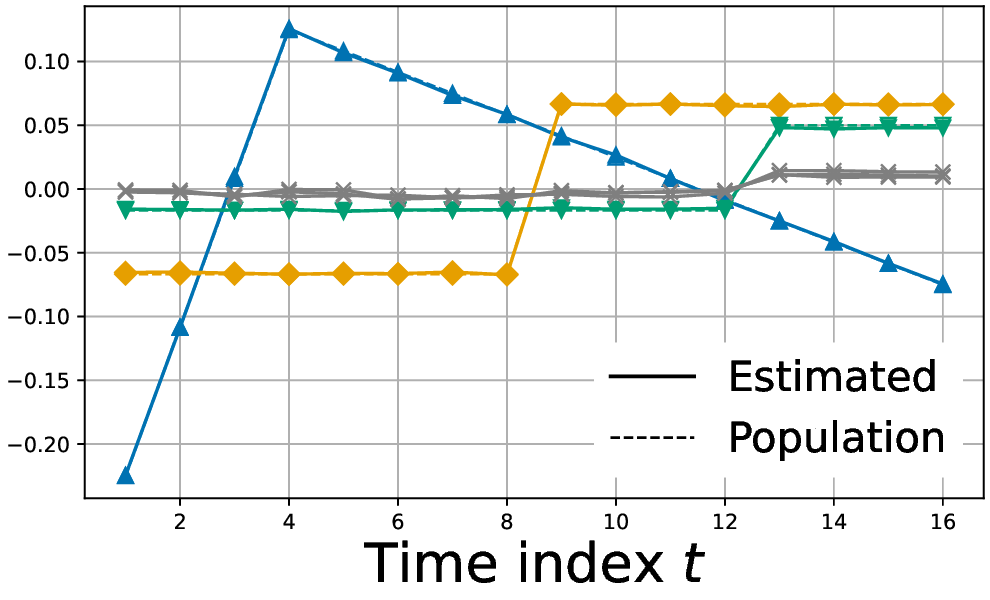}
        \caption{Mode-wise (canonical)}
    \end{subfigure}

    \caption{
    Multiscale trajectories estimated by the modified UASE with embedding dimension $d_{\mathrm{emb}} = 6$ and latent dimension $d_{\mathrm{lat}} = 3$. 
    The columns correspond to MV, trace, and mode-wise trajectories under a canonical basis. 
    The leading three mode-wise trajectories capture the latent structure, while the remaining three are nearly flat, reflecting their noise-dominated nature.
    Note that trajectories are invariant up to sign, and the signs have been adjusted for better visual interpretability.
    }
    \label{fig:multiscale-trajectories-dim6}
\end{figure}

\subsection{Comparison of Multiple Change Point Detection Performance}
\label{app:exp-cp-comparison}

We compare our proposed method with several competing approaches on Dataset~2. Relative to Dataset~1, this setting is substantially more challenging because it contains multiple temporal events with heterogeneous types and signal strengths. In particular, each of the three mode strengths shown in the right panel of Figure~\ref{fig:xi} contains two change points, yielding six ground truth events overall at
\[
t \in \{11,21,31,41,51,61\}.
\]
This design tests not only whether a method can detect the presence of change, but also whether it can localize several distinct events accurately over an extended time horizon.

We compare against five baselines, namely DeltaCon \citep{koutra2016deltacon}, Laplacian Anomaly Detection (LAD) \citep{huang2020laplacian}, the HCDL framework of \citet{fukushima2020detecting}, Euclidean Mirror \citep{athreya2025euclidean}, and the recent two stage localization and inference procedure of \citet{wang2026change}. Taken together, these methods cover a range of approaches to dynamic graph change detection, including graph similarity based statistics, spectral anomaly measures, hierarchical latent diagnostics, geometric trajectory summaries, and dedicated localization procedures.

For each method, we convert the output into a ranked list of candidate change times and evaluate top $K$ detection performance for $K \in \{3,6,9\}$. For methods that produce a single score sequence, we retain only strict local peaks. A time point is kept only when its score is larger than those of its two immediate neighbors, so boundary points and flat plateaus are excluded. The remaining candidates are ranked by score, and the top $K$ are selected, with earlier times used to break ties. We also impose a minimum separation of two time units between selected candidates to avoid counting nearby peaks as distinct detections. Predicted and true change times are then matched using a tolerance window of two time units, so any prediction within $\pm 2$ of a true change is counted as a match. We compute F1 from the tolerance-based matches. Timing MAE averages, over all selected predictions, the absolute distance to the nearest true change point. It is not restricted to matched predictions. All results are averaged over 100 trials.


For methods that produce multiple score sequences, we first fuse stream specific candidates into a single ranked list. For HCDL, and for each value of $K$, we take the top $K$ candidates from each of the three level-specific streams. To make scores comparable across streams, we rescale each nominated score by dividing it by the median score of its own stream. We then pool the nominated times across streams, merge duplicate time points, and assign each merged candidate the largest fused score among its duplicates. The fused candidates are ranked globally, after which the final top $K$ are selected subject to a minimum separation of four time units. The resulting change points are evaluated using the same matching tolerance and evaluation metrics as for the single-score methods.


For the proposed method, we construct change point scores from the mode-wise temporal trajectories and then fuse candidates across modes and change point orders. For each embedding mode \(k\), let \(\hat{\psi}_k(t)\) denote the associated 1D trajectory. We consider two types of changes in this trajectory: a \(0\)th order change, corresponding to a shift in level, and a \(1\)st order change, corresponding to a change in slope. Both scores are obtained from a local linear trend state space model fitted separately to each mode-wise trajectory. After centering \(\hat{\psi}_k(t)\) and scaling it to unit variance, we fit

\begin{align*}
\hat{\psi}_k(t) &= L_{k,t} + \epsilon_{k,t},
& \epsilon_{k,t} &\sim \mathcal N(0,\sigma_{\epsilon,k}^2), \\
L_{k,t} &= L_{k,t-1} + B_{k,t-1} + \eta_{k,t},
& \eta_{k,t} &\sim \mathcal N(0,\sigma_{\eta,k}^2), \\
B_{k,t} &= B_{k,t-1} + \zeta_{k,t},
& \zeta_{k,t} &\sim \mathcal N(0,\sigma_{\zeta,k}^2),
\end{align*}

\noindent where \(L_{k,t}\) is the local level and \(B_{k,t}\) is the local slope of the \(k\)th trajectory. The disturbances \(\eta_{k,t}\) and \(\zeta_{k,t}\) represent innovations in the level and slope, respectively, between times \(t-1\) and \(t\). We estimate the model parameters and initial state by maximum likelihood using the unobserved components implementation of \citet{seabold2010statsmodels}, and then apply Kalman smoothing to obtain the posterior mean disturbances
\[
\hat{\eta}_{k,t}
=
\mathbb E[\eta_{k,t}\mid \hat{\psi}_k(1),\dots,\hat{\psi}_k(T)],
\qquad
\hat{\zeta}_{k,t}
=
\mathbb E[\zeta_{k,t}\mid \hat{\psi}_k(1),\dots,\hat{\psi}_k(T)].
\]

Let \(\sigma_k\) denote the standard deviation used before normalizing \(\hat{\psi}_k(t)\). We define the \(0\)th order and \(1\)st order change point scores for mode \(k\) as
\[
s_{k,t}^{(0)} = \bigl|\sigma_k \hat{\eta}_{k,t}\bigr|,
\qquad
s_{k,t}^{(1)} = \bigl|\sigma_k \hat{\zeta}_{k,t}\bigr|.
\]
Thus, \(s_{k,t}^{(0)}\) measures the magnitude of a level innovation in the \(k\)th trajectory at time \(t\), whereas \(s_{k,t}^{(1)}\) measures the magnitude of a slope innovation. 
We interpret these as the level (i.e., \(0\)th order) and slope (i.e., \(1\)st order) change point scores for mode \(k\).

We then fuse candidates across the trajectory based score sequences. Each sequence contributes candidates given by strict local peaks. For a given \(K\), we take the top \(K\) candidates from each sequence and pool them.  Because the raw score scales may differ between the \(0\)th order and \(1\)st order families, we compute one shared median score within the \(0\)th order family and one shared median score within the \(1\)st order family, and rescale each sequence by the corresponding family median. The remaining fusion and evaluation steps are the same as those used for HCDL.

For each competing method, we tuned method-specific hyperparameters by grid search while keeping the matching tolerance and top-$K$ evaluation values fixed across methods. We report the best-performing configuration. For HCDL~\citep{fukushima2020detecting}, we varied the half window parameter over \(h\in\{2,3,4,5,6,7,8\}\), while fixing the maximum number of clusters to \(K=10\), which is sufficiently large for the three block synthetic networks. The mean $F_1$ across $K$s peaked at $h=5$, which we use in the experiments. For LAD~\citep{huang2020laplacian}, we swept the (short, long) window pair $(w_{\mathrm{short}}, w_{\mathrm{long}}) \in {(3,7), (4,8), (5,9)}$ (burn-in $= w_{\mathrm{long}}$) and fixed the number of singular values to $r = n - 1 = 499$. We cap $w_{\mathrm{long}} < 10$ to avoid straddling the $10$ timestamp inter change point spacing. The best mean
$F_1$ was attained at $(5,9)$. For DeltaCon~\citep{koutra2016deltacon}, we used the exact implementation, since the faster approximation is intended for substantially larger networks and may introduce unnecessary approximation error at the present scale. For the two stage method of \citet{wang2026change}, we report results for \(K\in\{3,6,9\}\). For the proposed method, we fixed the embedding dimension to \(d=3\). This matches the true latent dimension in the synthetic experiments, and Appendix~\ref{app:overestimating} shows that moderate overestimation of the embedding dimension does not substantially affect recovery of the relevant latent structure.

\begin{table}[t]
	\centering
	\caption{Performance comparison across different values of $K$. Values in brackets denote 95\% confidence intervals.}
	\label{tab:performance_ci}
	\scriptsize
	\setlength{\tabcolsep}{2.5pt}
	\resizebox{\textwidth}{!}{%
	\begin{tabular}{lrrrrrr}
		\toprule
		& \multicolumn{2}{c}{$K=3$} & \multicolumn{2}{c}{$K=6$} & \multicolumn{2}{c}{$K=9$} \\
		\cmidrule(lr){2-3}\cmidrule(lr){4-5}\cmidrule(lr){6-7}
		Method & F1 & MAE & F1 & MAE & F1 & MAE \\
		\midrule
		DeltaCon
		& 0.447 [0.444, 0.451] & 2.090 [2.027, 2.147]
		& 0.513 [0.492, 0.535] & 2.400 [2.293, 2.512]
		& 0.579 [0.565, 0.592] & 2.289 [2.227, 2.352] \\

		LAD
		& \textbf{0.667 [0.667, 0.667]} & 0.733 [0.707, 0.760]
		& 0.833 [0.833, 0.833] & 1.257 [1.213, 1.300]
		& 0.688 [0.679, 0.699] & 1.888 [1.834, 1.942] \\

		HCDL
		& 0.322 [0.293, 0.351] & 2.420 [2.280, 2.560]
		& 0.480 [0.455, 0.505] & 2.208 [2.112, 2.303]
		& 0.532 [0.511, 0.553] & 2.149 [2.092, 2.209] \\

		Mirror
		& 0.664 [0.660, 0.667] & 0.007 [0.000, 0.020]
		& 0.573 [0.557, 0.590] & 1.562 [1.490, 1.635]
		& 0.512 [0.500, 0.525] & 2.076 [2.017, 2.136] \\

		Two-stage
		& \textbf{0.667 [0.667, 0.667]} & \textbf{0.000 [0.000, 0.000]}
		& 0.815 [0.798, 0.830] & 1.037 [1.007, 1.070]
		& 0.692 [0.677, 0.707] & 2.073 [2.027, 2.120] \\

		\midrule
		MENT
		& \textbf{0.667 [0.667, 0.667]} & \textbf{0.000 [0.000, 0.000]}
		& \textbf{0.965 [0.952, 0.978]} & \textbf{0.095 [0.058, 0.135]}
		& \textbf{0.800 [0.800, 0.800]} & \textbf{1.002 [0.960, 1.046]} \\
		\bottomrule
	\end{tabular}%
	}
\end{table}
Table~\ref{tab:performance_ci} reports the mean results with 95\% bootstrap confidence intervals. The proposed method achieves the strongest overall performance across all three values of $K$. At $K=3$, it matches LAD in F1 score while achieving perfect localization accuracy among the matched detections. At $K=6$, which coincides with the true number of change points, it attains near-perfect recovery with an F1 score of $0.965$ and a timing MAE of $0.1$, substantially outperforming all competing methods (See also Figure~\ref{app:fig:ment}). Even at $K=9$, where some degree of over selection is unavoidable, it remains the best performing approach under both metrics.

\begin{figure}[htbp]
    \centering
    \includegraphics[width=0.63\textwidth]{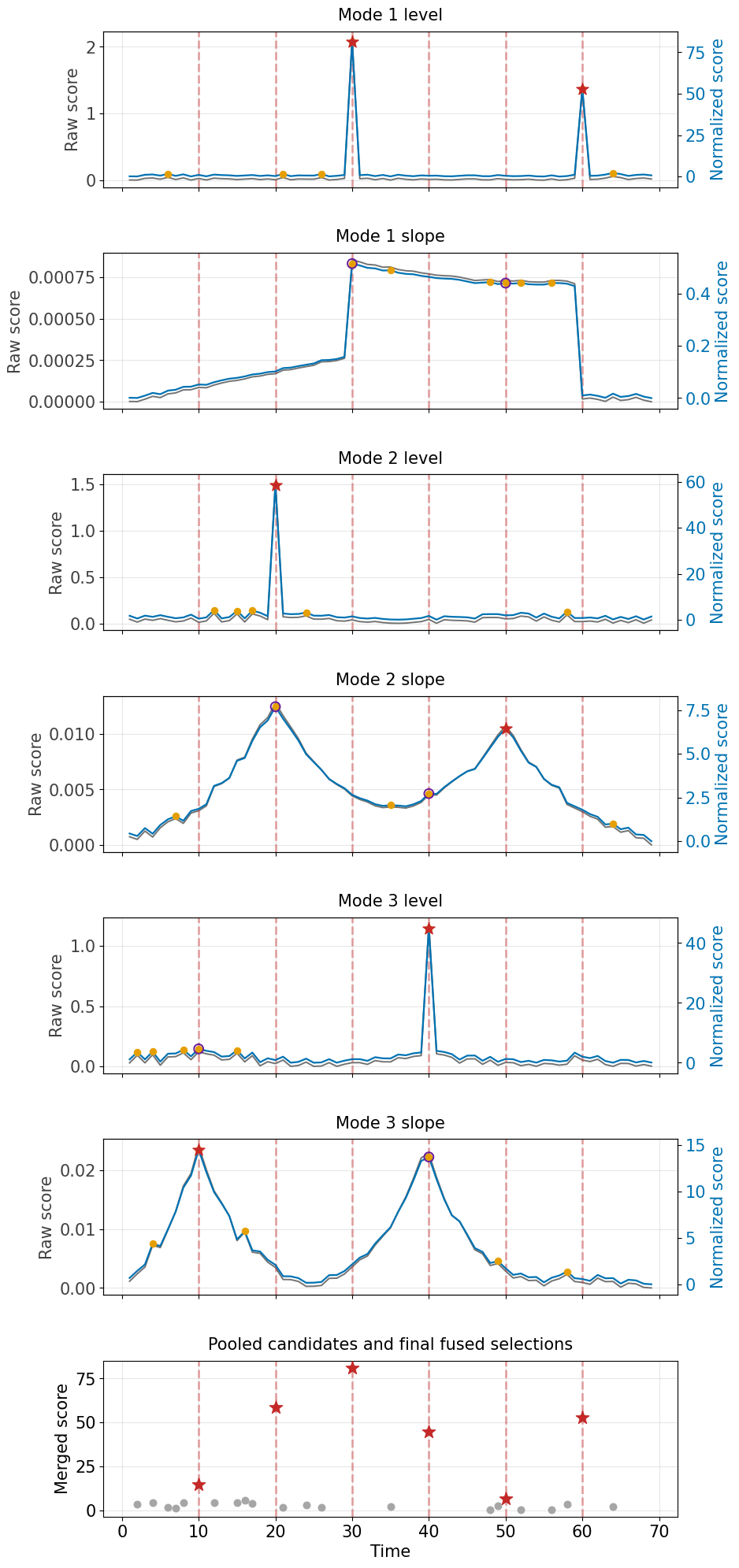}
    \caption{Change points detected by MENT for Dataset 2. In each of the top panels, the gray curve shows the raw score for one stream. When the fusion mode uses normalized stream scores, the blue curve shows the normalized score used for fusion. Orange circles denote that stream’s top-K candidate changepoints. Red stars mark candidate times that are selected in the final fused top-K and for which the current stream provides the winning score. Purple hollow circles mark candidate times that are selected globally but whose winning score comes from a different stream. Vertical red dashed lines indicate the final fused changepoints. The bottom panel shows the pooled candidate set obtained by taking the union of the per-stream top-K candidates and merging them by the fusion rule. Red stars indicate the final selected changepoints.}
    \label{app:fig:ment}
\end{figure}

Among the baselines, LAD is one of the most competitive. This is likely because it leverages information from multiple singular values, which is broadly consistent with the multi component nature of our setting. However, because LAD summarizes these components through an aggregated spectral anomaly score rather than a canonical mode-wise temporal geometry, it is less effective at separating distinct latent directions of change. Consequently, it can miss change points that are weak in the aggregate score but clear along a specific mode. This behavior is evident in Figure~\ref{app:fig:lad}, where LAD fails to identify the change point at time \(50\).

Another strong competitor is the two stage method of \citet{wang2026change}, which was developed for change point localization and inference in dynamic multilayer random dot product graphs. The method first uses seeded binary segmentation with CUSUM statistics to obtain preliminary candidate change points, and then refines these candidates using low rank tensor estimation. This second stage is important. Under the dynamic multilayer RDPG model, the expected CUSUM tensor has a low rank structure induced by shared node latent positions and time-varying layer specific connectivity matrices. As a result, the refinement step can denoise the local change signal by pooling information across nodes, layers, and low rank tensor factors.  This explains why the method is a strong baseline in our experiments, even though our synthetic setting is organized around mode-wise temporal variation rather than multilayer tensor shifts.

In our experiments, the two stage method recovers most of the major events, but it can miss or mislocalize weaker mode specific changes. 
This is consistent with the difference between the two objectives.  The two stage method is designed to localize coherent low-rank shifts in the probability tensor, whereas our method explicitly decomposes temporal variation into canonical mode-wise trajectories. Thus, when multiple changes occur along different latent directions with heterogeneous signal strengths, the proposed multiscale trajectory representation can separate events that may be partially merged or attenuated in a shared low-rank tensor localization procedure.

Because the two stage method is not naturally a score based peak picking method, we do not display a per time score curve for it.  Instead, we evaluate it according to its native output that is the refined candidate change points produced by its two stage procedure.  For this reason, when reporting top-\(K\) performance, we use the ranked candidate locations returned by the method directly, rather than forcing it into the same local-peak selection protocol used for single score baselines.

\begin{figure}[htbp]
    \centering
    \includegraphics[width=0.8\textwidth]{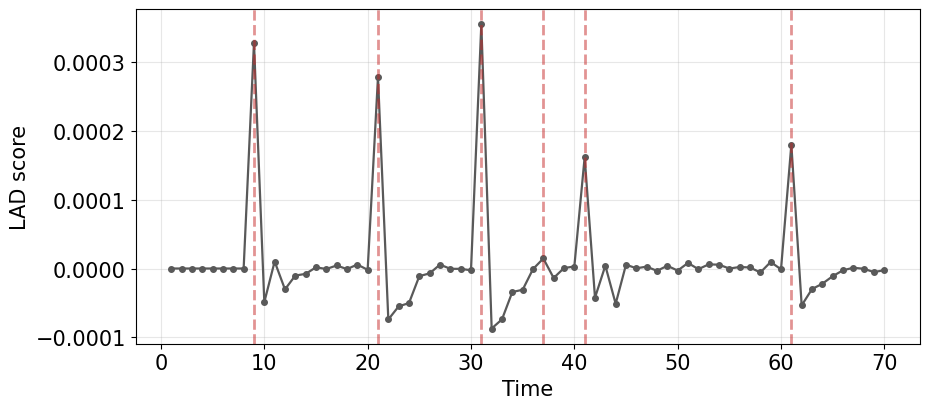}
    \caption{Change points detected by LAD for Dataset 2. The gray curve shows the raw score. Vertical red dashed lines indicate the detected change points.}
    \label{app:fig:lad}
\end{figure}

The weaker performance of Euclidean Mirror is also consistent with its design. Since it primarily tracks the direction of maximal variation, it tends to emphasize only the dominant mode of temporal change. As a result, more subtle or secondary modes are easily overlooked. This can be seen in Figure~\ref{app:fig:euclidean_mirror}, where the method appears to respond mainly to the first mode.

\begin{figure}[htbp]
    \centering
    \includegraphics[width=0.8\textwidth]{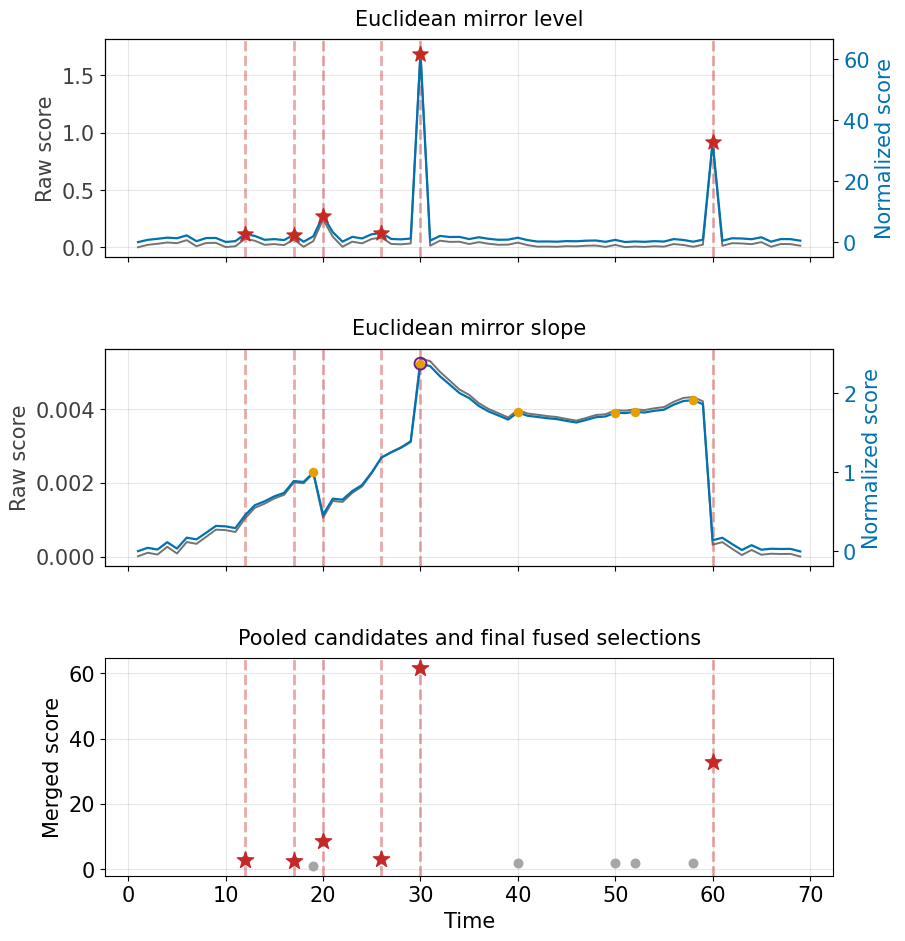}
    \caption{Change points detected by Euclidean Mirror for Dataset 2. The gray curve shows the raw score for one stream. The blue curve shows the normalized score used for fusion. Orange circles denote that stream’s top-K candidate change points. Red stars mark candidate times that are selected in the final fused top $K$ and for which the current stream provides the winning score. Purple hollow circles mark candidate times that are selected globally but whose winning score comes from a different stream. Vertical red dashed lines indicate the final fused change points. The bottom panel shows the pooled candidate set obtained by taking the union of the per stream top $K$ candidates and merging them by the fusion rule. Red stars indicate the final selected change points.}
    \label{app:fig:euclidean_mirror}
\end{figure}

HCDL is more structured than single score methods because it targets changes at different levels of a latent variable model (See Figure~\ref{app:fig:hcdl}). However, these levels do not correspond to the geometric modes of temporal variation that drive our setting. As a result, when multiple heterogeneous events occur across different population trajectories, HCDL is less naturally aligned with precise eventwise localization than the proposed multiscale trajectory method.

\begin{figure}[htbp]
    \centering
    \includegraphics[width=0.8\textwidth]{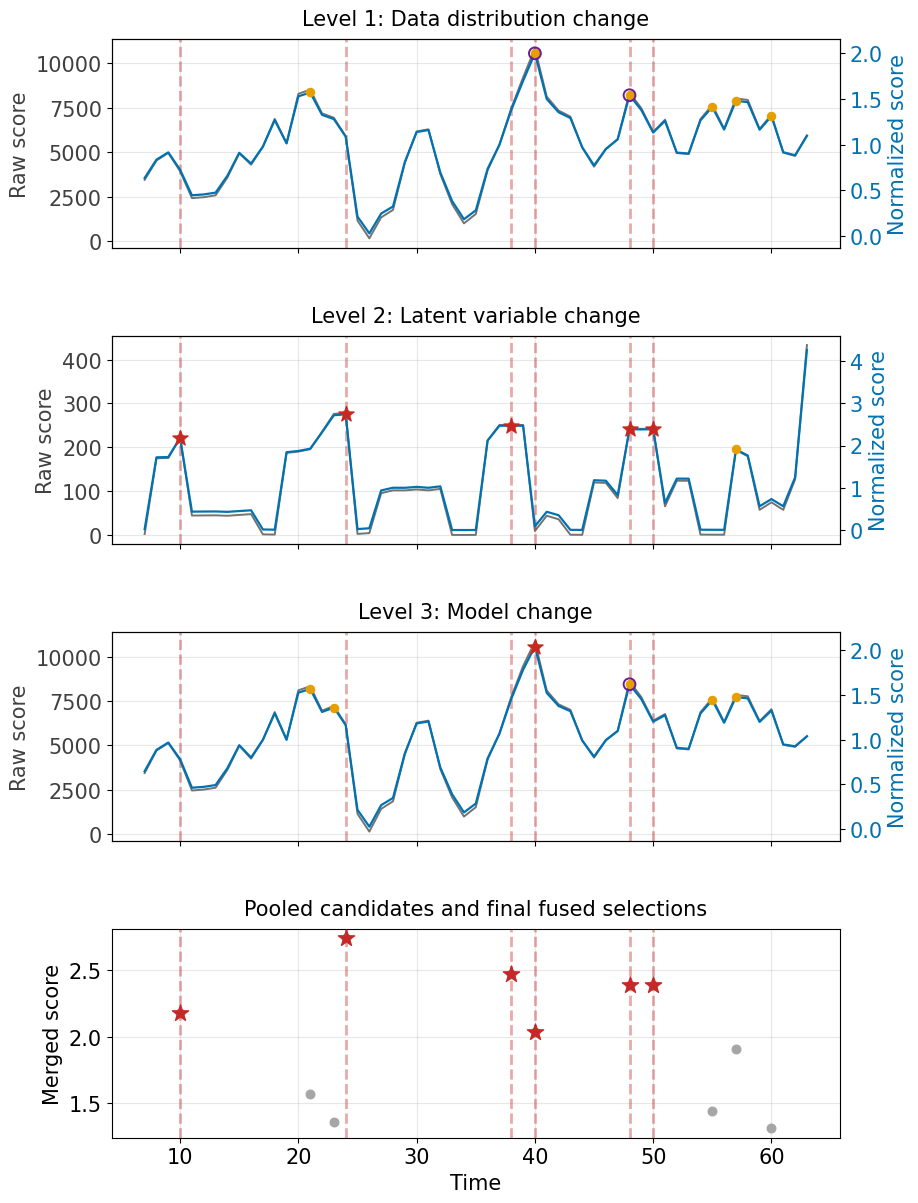}
    \caption{Change points detected by HCDL for Dataset 2. Gray lines show raw scores, blue lines show normalized fusion scores, orange circles show per level top $K$ candidates, red stars show final fused selections attributed to the current level, and purple hollow circles show globally selected candidates whose winning score comes from another level. Red dashed vertical lines mark the final fused changepoints. The bottom panel summarizes the merged pooled candidates and final selections.}
    \label{app:fig:hcdl}
\end{figure}

DeltaCon performs worse in this setting because it is designed as a pairwise graph similarity measure rather than a multievent localization method. It summarizes the difference between consecutive graphs through a single global affinity based score, which is effective for detecting whether a change occurred but is less well suited for separating several distinct structural events over time. As a result, changes arising from different latent modes can be blended into one aggregate signal, making weaker events harder to localize accurately (See Figure~\ref{app:fig:deltacon}).

\begin{figure}[htbp]
    \centering
    \includegraphics[width=0.7\textwidth]{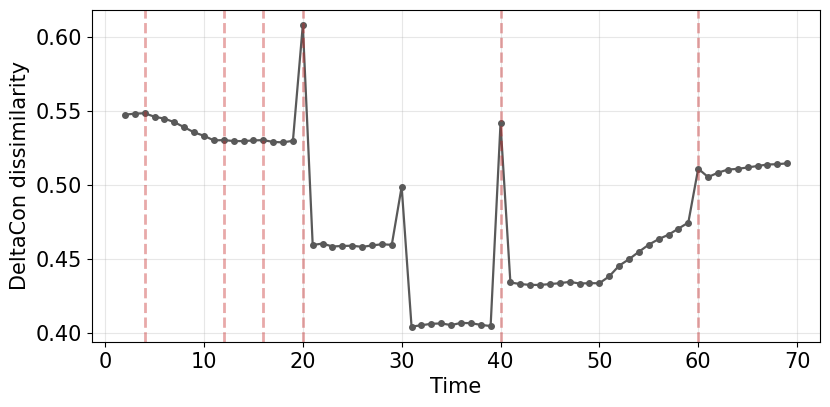}
    \caption{Change points detected by DeltaCon for Dataset 2. The gray curve shows the raw score. Vertical red dashed lines indicate the detected change points.}
    \label{app:fig:deltacon}
\end{figure}

Overall, these results show that the proposed multiscale trajectory construction not only detects the presence of multiple changes reliably, but also localizes them more accurately than competing approaches in this challenging heterogeneous setting.

We also report the computational complexity of the compared methods in Table~\ref{tab:complexity}. The proposed MENT procedure has the same leading order time complexity as the spectral embedding based baselines, such as LAD and Euclidean Mirror, namely \(O(n^2 dT)\). Its space complexity is \(O(ndT)\), reflecting the storage of time-indexed embeddings used to construct the multiscale trajectories. Thus, MENT remains computationally comparable to the strongest embedding based competitors while providing additional mode-wise decomposition, attribution, and change point localization capabilities.

All reported runs were executed on a single CPU workstation with two Intel Xeon Gold 6338 CPUs, 64 physical cores in total, 2.0 TiB of RAM, and no GPU. The machine ran Ubuntu 20.04 with Linux 5.15. The code was run with Python 3.13.1, NumPy/SciPy, and the OpenBLAS 0.3.31 LP64 backend. BLAS thread counts were left at their default settings. The largest run, the U.S. court opinions experiment with \(T=80\), \(n=9399\), and 16,716,425 edge instances, took 27.5 minutes end to end. Its largest persisted artifact was 67.7 GB. The Enron experiment with \(T=165\), \(n=150\), and 14,500 edge instances took 83 seconds end to end, with a largest persisted artifact of 1.3 GB. The synthetic Dataset 2 run with \(T=70\), \(n=500\), and 2,488,628 edge instances took 17 seconds end to end, with a largest persisted artifact of 0.8 GB. Peak resident memory was not directly recorded.

\begin{figure}[t]
\centering
\setlength{\tabcolsep}{7pt}
\begin{tabular}{lrr}
\toprule
Method & Time Complexity & Space Complexity \\
\midrule
DeltaCon         &  $O(mT)$       &  $O(m)$     \\
LAD              &  $O(n^2dT)$    &  $O(nd)$    \\
HCDL             &  $O(n^2K^2hT)$ &  $O(nK)$   \\
Mirror           &  $O(n^2dT)$    &  $O(ndT)$    \\
Two stage        &  $O(n^3)$      &  $O(n^2)$   \\
\midrule
MENT             &  $O(n^2dT)$    &  $O(ndT)$   \\
\bottomrule
\end{tabular}
\caption{Computational complexity comparison of the compared methods. Here, $T$, $n$ and $m$ represent the number of time points, the number of nodes, and the maximum number of edges per snapshot. $d$, $K$ and $h$ represent the embedding dimension, the maximum number of clusters and the window size. We assume that $n \gg T, d, K, h$.}
\label{tab:complexity}
\end{figure}

\subsection{Construction of the Real-World Legal Citation Network}
\label{app:exp-real}

We construct the legal citation network from the CourtListener bulk case law snapshots, which are released as quarterly CSV exports of the underlying database rather than as incremental updates~\citep{freelawproject2026bulk}. We use the opinion clusters and opinions tables from the May 6, 2024 bulk release. The opinion clusters table provides metadata for groups of related opinions, including filing dates, while the opinions table provides the corresponding opinion texts.

The preprocessing pipeline has three main stages. First, we extract the mapping from each opinion to its parent opinion cluster and the filing date associated with that cluster, linking each opinion level text record to filing time metadata. Second, we process the opinion texts to extract legal references. For each opinion, we select the first available substantive text field among the available plain text, HTML, and XML fields, after cleaning the text and requiring a minimum length of $1{,}000$ characters. We then apply \texttt{eyecite}~\citep{cushman2021eyecite} to identify legal citations and retain only normalized references passing strict filtering rules. Case citations are required to match a reporter style pattern, statutory citations are required to contain both a section symbol and a digit, and non-substantive citation strings such as ``Id.'' and ``supra'' are removed. We also separately extract references to U.S.\ constitutional amendments and normalize them to canonical forms, such as ``U.S.\ Const.\ amend.\ XIV''. Opinions are retained only if they contain at least two extracted references in total, counting both legal citations and amendment references.

The resulting intermediate object is an opinion level reference dataset. Each source opinion is associated with its filing date, a deduplicated set of normalized cited authorities, and any normalized constitutional amendment references.

The full set of authorities cited at least once across the entire corpus is heavy tailed and contains a long tail of essentially singleton references that carry little signal for a temporal embedding. We therefore impose two complementary restrictions on the data before constructing the network. First, we restrict the set of admissible authorities (i.e.\ candidate nodes). We partition the corpus into ten year filing windows and, within each window, retain only the top $1{,}000$ most frequently cited authorities by raw within-window citation count. The union of these per decade top $1{,}000$ sets defines the candidate authority pool, which is fixed across all years. Each opinion's cited authority set is then intersected with this pool, and any opinion left with fewer than two surviving references is discarded. 

Following \citet{matsumoto2025hypergraph}, we then treat the set of authorities cited by each opinion (after the restrictions above) as a hyperedge. Aggregating these hyperedges by filing year yields a yearly legal citation hypergraph. Finally, we apply clique expansion within each yearly hypergraph to obtain an undirected network for each year.

We focus on the period 1939--2018, yielding 80 yearly snapshots. In aggregate, the network contains roughly $9.4\mathrm{k}$ distinct authorities, on the order of $3.2\mathrm{M}$ distinct undirected co-citation pairs, and about $1.7 \times 10^{7}$ edge instances when summed with multiplicity across years. Individual yearly snapshots are large, locally clustered graphs, each typically containing a few thousand active authorities and on the order of $10^{5}$ edges. Both the active node set and the edge count grow over time as more case law is filed.

All code needed to reproduce the dataset construction is provided in the supplementary material. Because the resulting network files are large, they are not included in the supplementary archive. We will make them available online upon acceptance.

\begin{figure}[htbp]
    \centering
    \includegraphics[width=0.6\textwidth]{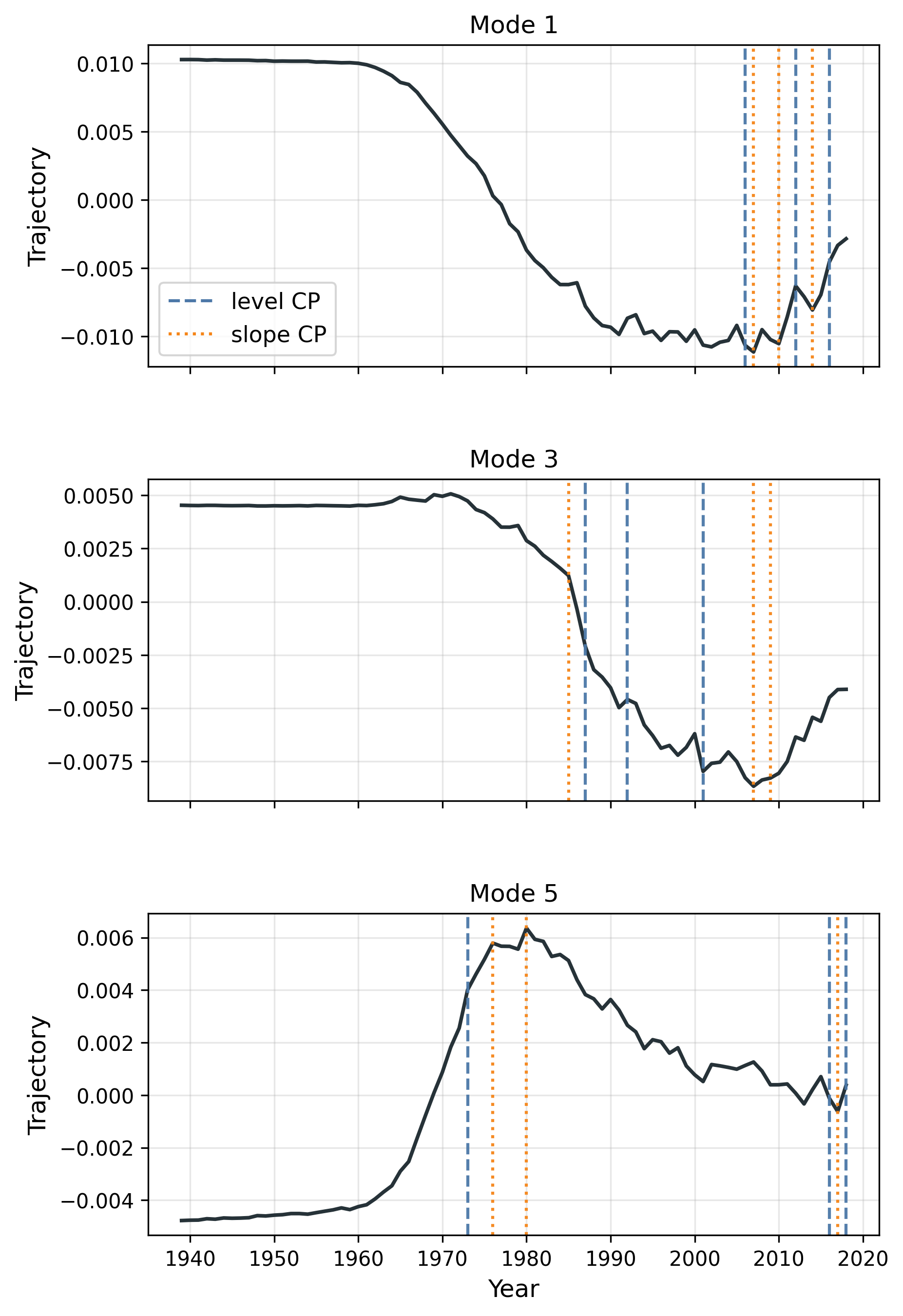}
    \caption{Trajectories of the selected modes for the U.S. Court opinions data.}
    \label{app:fig:us:trajectory}
\end{figure}

\begin{figure}[htbp]
    \centering
    \includegraphics[width=1.0\textwidth]{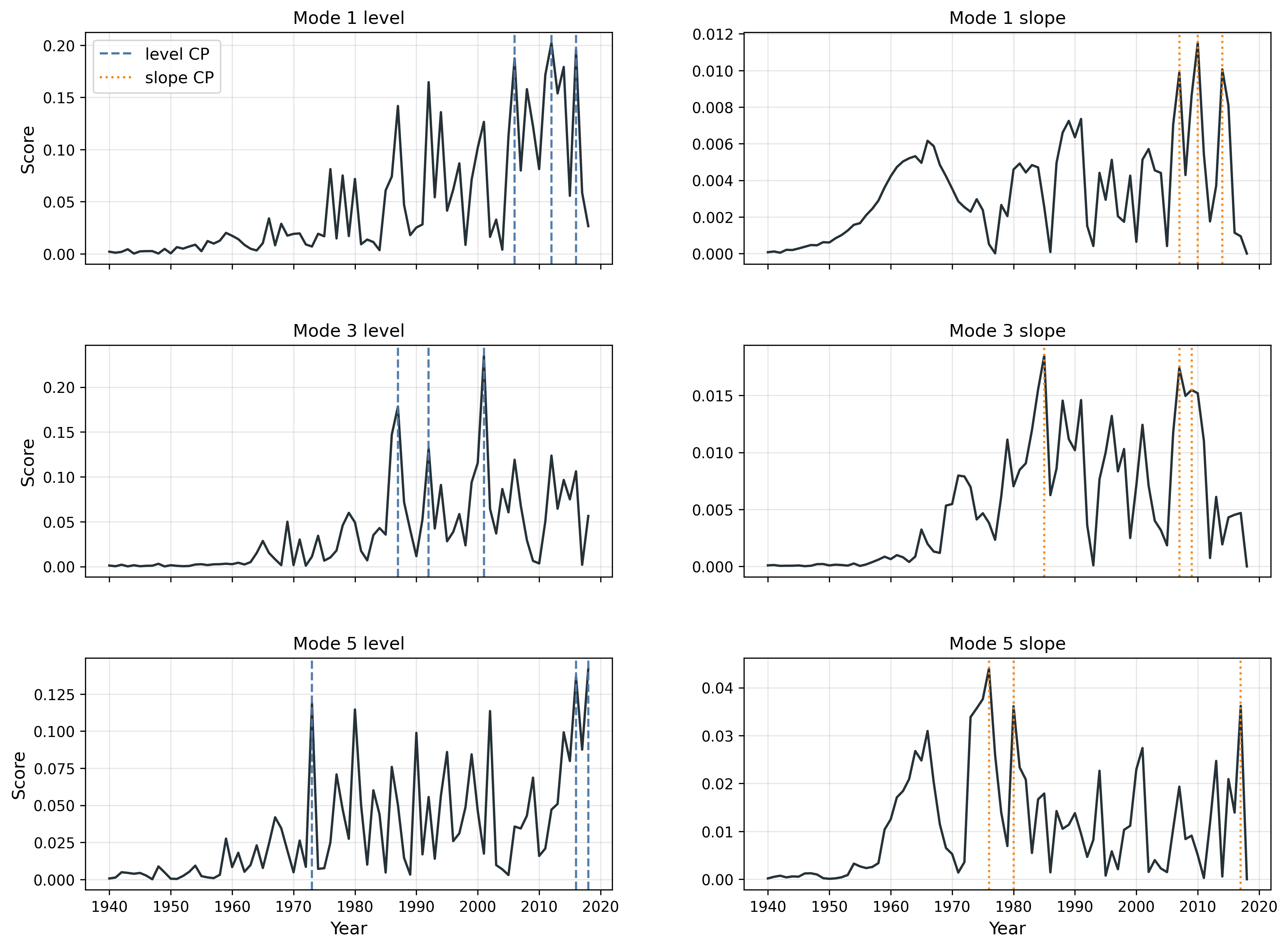}
    \caption{Mode-wise level (i.e., $0$th) and slope (i.e., $1$st) order change point scores of the selected modes for the U.S. Court opinions data.}
    \label{app:fig:us:score}
\end{figure}

\subsection{Detailed Analysis of the Real World Legal Citation Network}
\label{app:exp-real-additional}

This section explains how we interpret the legal citation modes shown in the main text. Each learned mode has a positive side and a negative side. For mode $k$ and a time comparison $(t,s)$, we consider the signed node attribution of reporter node $i$ given by
\[
	\hat a_k(t,s;i)
	=
	\frac{1}{n^{1/2}}
	\left\langle
	\hat{\mathbf Y}_{i:}(t)-\hat{\mathbf Y}_{i:}(s),
	\hat{\bm u}_k
	\right\rangle.
\]
\noindent Nodes with large positive values of $\hat a_k(t,s;i)$ form the positive side of the mode at that time comparison, while nodes with large negative values form the negative side.

At each detected level change point, we identify the reporter citations with the largest positive and negative node attributions. These citations tell us which cases contribute most strongly to the movement of that mode at that time. We then ask whether the positive and negative sides correspond to recognizable legal clusters.

The analysis below focuses on level change attributions, corresponding to the $0$th order scores plotted in the main text. The same attribution procedure can also be applied to slope change scores by applying it to changes in the trajectory increments rather than to level shifts. We do not interpret those slope attributions here because the main text case study is organized around level shifts. For both real-world dataset we set $d=32$.

\subsubsection{Mode 1}
\begingroup
\setlength{\tabcolsep}{3pt}
\renewcommand{\arraystretch}{1.08}
\begin{footnotesize}

\begin{longtable}{c c l p{2.0cm} r p{5.8cm}}
\caption{Mode 1 attribution table. For each change point, TOP rows report the most positively attributed reporter nodes and BOTTOM rows report the most negatively attributed reporter nodes.}
\label{tab:mode0_attributions}\\
\toprule
CP rank & Year & Side & Reporter cite & Score & Case \\
\midrule
\endfirsthead

\toprule
CP rank & Year & Side & Reporter cite & Score & Case \\
\midrule
\endhead

\midrule
\multicolumn{6}{r}{\emph{Continued on next page}}\\
\midrule
\endfoot

\bottomrule
\endlastfoot

1 & 2012 & TOP & 47 L.Ed.2d 128 & 0.08947 & \textit{Imbler v. Pachtman}, 424 U.S. 409 (1976) \\
1 & 2012 & TOP & 104 L.Ed.2d 443 & 0.08912 & \textit{Graham v. Connor}, 490 U.S. 386 (1989) \\
1 & 2012 & TOP & 96 S.Ct. 984 & 0.08782 & \textit{Imbler v. Pachtman}, 424 U.S. 409 (1976) \\
1 & 2012 & TOP & 3 L.Ed.2d 1217 & 0.08277 & \textit{Napue v. Illinois}, 360 U.S. 264 (1959) \\
1 & 2012 & TOP & 79 S.Ct. 1173 & 0.08237 & \textit{Napue v. Illinois}, 360 U.S. 264 (1959) \\
1 & 2012 & TOP & 112 S.Ct. 475 & 0.07287 & \textit{Estelle v. McGuire}, 502 U.S. 62 (1991) \\
1 & 2012 & TOP & 467 U.S. 431 & 0.07248 & \textit{Nix v. Williams}, 467 U.S. 431 (1984) \\
1 & 2012 & TOP & 64 L.Ed.2d 497 & 0.07098 & \textit{United States v. Mendenhall}, 446 U.S. 544 (1980) \\
1 & 2012 & TOP & 100 S.Ct. 1870 & 0.07098 & \textit{United States v. Mendenhall}, 446 U.S. 544 (1980) \\
1 & 2012 & TOP & 116 L.Ed.2d 385 & 0.06818 & \textit{Estelle v. McGuire}, 502 U.S. 62 (1991) \\
\addlinespace
1 & 2012 & BOTTOM & 45 S.Ct. 280 & -0.04782 & \textit{Carroll v. United States}, 267 U.S. 132 (1925) \\
1 & 2012 & BOTTOM & 50 L.Ed.2d 450 & -0.04905 & \textit{Village of Arlington Heights v. Metropolitan Housing Development Corp.}, 429 U.S. 252 (1977) \\
1 & 2012 & BOTTOM & 97 S.Ct. 555 & -0.04924 & \textit{Village of Arlington Heights v. Metropolitan Housing Development Corp.}, 429 U.S. 252 (1977) \\
1 & 2012 & BOTTOM & 155 L.Ed.2d 144 & -0.05175 & \textit{Lockyer v. Andrade}, 538 U.S. 63 (2003) \\
1 & 2012 & BOTTOM & 111 S.Ct. 1246 & -0.05285 & \textit{Arizona v. Fulminante}, 499 U.S. 279 (1991) \\
1 & 2012 & BOTTOM & 113 L.Ed.2d 302 & -0.05285 & \textit{Arizona v. Fulminante}, 499 U.S. 279 (1991) \\
1 & 2012 & BOTTOM & 499 U.S. 279 & -0.05689 & \textit{Arizona v. Fulminante}, 499 U.S. 279 (1991) \\
1 & 2012 & BOTTOM & 412 U.S. 218 & -0.05839 & \textit{Schneckloth v. Bustamonte}, 412 U.S. 218 (1973) \\
1 & 2012 & BOTTOM & 105 S.Ct. 2633 & -0.06366 & \textit{Caldwell v. Mississippi}, 472 U.S. 320 (1985) \\
1 & 2012 & BOTTOM & 86 L.Ed.2d 231 & -0.06366 & \textit{Caldwell v. Mississippi}, 472 U.S. 320 (1985) \\

\midrule

2 & 2016 & TOP & 105 S.Ct. 3249 & 0.09473 & \textit{City of Cleburne v. Cleburne Living Center, Inc.}, 473 U.S. 432 (1985) \\
2 & 2016 & TOP & 6 L.Ed.2d 1081 & 0.08627 & \textit{Mapp v. Ohio}, 367 U.S. 643 (1961) \\
2 & 2016 & TOP & 104 L.Ed.2d 443 & 0.08290 & \textit{Graham v. Connor}, 490 U.S. 386 (1989) \\
2 & 2016 & TOP & 109 S.Ct. 1865 & 0.08283 & \textit{Graham v. Connor}, 490 U.S. 386 (1989) \\
2 & 2016 & TOP & 161 L.Ed.2d 1 & 0.07755 & \textit{Roper v. Simmons}, 543 U.S. 551 (2005) \\
2 & 2016 & TOP & 95 L.Ed.2d 697 & 0.07237 & \textit{United States v. Salerno}, 481 U.S. 739 (1987) \\
2 & 2016 & TOP & 364 U.S. 479 & 0.06613 & \textit{Shelton v. Tucker}, 364 U.S. 479 (1960) \\
2 & 2016 & TOP & 125 S.Ct. 1183 & 0.06515 & \textit{Roper v. Simmons}, 543 U.S. 551 (2005) \\
2 & 2016 & TOP & 48 L.Ed.2d 597 & 0.06411 & \textit{Washington v. Davis}, 426 U.S. 229 (1976) \\
2 & 2016 & TOP & 81 S.Ct. 1684 & 0.06394 & \textit{Mapp v. Ohio}, 367 U.S. 643 (1961) \\
\addlinespace
2 & 2016 & BOTTOM & 103 S.Ct. 3469 & -0.04945 & \textit{Michigan v. Long}, 463 U.S. 1032 (1983) \\
2 & 2016 & BOTTOM & 101 S.Ct. 690 & -0.04961 & \textit{United States v. Cortez}, 449 U.S. 411 (1981) \\
2 & 2016 & BOTTOM & 32 L.Ed.2d 612 & -0.05059 & \textit{Adams v. Williams}, 407 U.S. 143 (1972) \\
2 & 2016 & BOTTOM & 88 L.Ed.2d 662 & -0.05081 & \textit{Daniels v. Williams}, 474 U.S. 327 (1986) \\
2 & 2016 & BOTTOM & 77 L.Ed.2d 1201 & -0.05118 & \textit{Michigan v. Long}, 463 U.S. 1032 (1983) \\
2 & 2016 & BOTTOM & 106 S.Ct. 662 & -0.05378 & \textit{Daniels v. Williams}, 474 U.S. 327 (1986) \\
2 & 2016 & BOTTOM & 374 U.S. 23 & -0.05509 & \textit{Ker v. California}, 374 U.S. 23 (1963) \\
2 & 2016 & BOTTOM & 83 S.Ct. 1623 & -0.05962 & \textit{Ker v. California}, 374 U.S. 23 (1963) \\
2 & 2016 & BOTTOM & 9 L.Ed.2d 799 & -0.06539 & \textit{Gideon v. Wainwright}, 372 U.S. 335 (1963) \\
2 & 2016 & BOTTOM & 83 S.Ct. 792 & -0.06780 & \textit{Gideon v. Wainwright}, 372 U.S. 335 (1963) \\

\midrule

3 & 2006 & TOP & 473 U.S. 432 & 0.10705 & \textit{City of Cleburne v. Cleburne Living Center, Inc.}, 473 U.S. 432 (1985) \\
3 & 2006 & TOP & 105 S.Ct. 3249 & 0.10318 & \textit{City of Cleburne v. Cleburne Living Center, Inc.}, 473 U.S. 432 (1985) \\
3 & 2006 & TOP & 87 L.Ed.2d 313 & 0.10301 & \textit{City of Cleburne v. Cleburne Living Center, Inc.}, 473 U.S. 432 (1985) \\
3 & 2006 & TOP & 397 U.S. 337 & 0.06634 & \textit{Illinois v. Allen}, 397 U.S. 337 (1970) \\
3 & 2006 & TOP & 338 U.S. 25 & 0.06556 & \textit{Wolf v. Colorado}, 338 U.S. 25 (1949) \\
3 & 2006 & TOP & 69 S.Ct. 1359 & 0.06495 & \textit{Wolf v. Colorado}, 338 U.S. 25 (1949) \\
3 & 2006 & TOP & 93 L.Ed. 1782 & 0.06495 & \textit{Wolf v. Colorado}, 338 U.S. 25 (1949) \\
3 & 2006 & TOP & 523 U.S. 1 & 0.06156 & \textit{Spencer v. Kemna}, 523 U.S. 1 (1998) \\
3 & 2006 & TOP & 519 U.S. 33 & 0.05660 & \textit{Ohio v. Robinette}, 519 U.S. 33 (1996) \\
3 & 2006 & TOP & 501 U.S. 429 & 0.05525 & \textit{Florida v. Bostick}, 501 U.S. 429 (1991) \\
\addlinespace
3 & 2006 & BOTTOM & 37 Cal.4th 428 & -0.06259 & \textit{People v. Partida}, 37 Cal. 4th 428 (2005) \\
3 & 2006 & BOTTOM & 36 L.Ed.2d 439 & -0.06302 & \textit{Preiser v. Rodriguez}, 411 U.S. 475 (1973) \\
3 & 2006 & BOTTOM & 388 U.S. 293 & -0.06440 & \textit{Stovall v. Denno}, 388 U.S. 293 (1967) \\
3 & 2006 & BOTTOM & 80 L.Ed.2d 657 & -0.08186 & \textit{United States v. Cronic}, 466 U.S. 648 (1984) \\
3 & 2006 & BOTTOM & 104 S.Ct. 2039 & -0.08258 & \textit{United States v. Cronic}, 466 U.S. 648 (1984) \\
3 & 2006 & BOTTOM & 130 L.Ed.2d 808 & -0.08436 & \textit{Schlup v. Delo}, 513 U.S. 298 (1995) \\
3 & 2006 & BOTTOM & 115 S.Ct. 851 & -0.08436 & \textit{Schlup v. Delo}, 513 U.S. 298 (1995) \\
3 & 2006 & BOTTOM & 513 U.S. 298 & -0.08450 & \textit{Schlup v. Delo}, 513 U.S. 298 (1995) \\
3 & 2006 & BOTTOM & 87 S.Ct. 1967 & -0.09221 & \textit{Stovall v. Denno}, 388 U.S. 293 (1967) \\
3 & 2006 & BOTTOM & 18 L.Ed.2d 1199 & -0.09426 & \textit{Stovall v. Denno}, 388 U.S. 293 (1967) \\

\end{longtable}
\end{footnotesize}
\endgroup

Mode 1 has three leading level change points, in 2006, 2012, and 2016. The dates are relatively close to one another, and the same types of authorities recur across the three attribution blocks. We therefore read this mode as a medium term reweighting of citation neighborhoods during the mid 2000s to mid 2010s, rather than as a single isolated legal event.

The main contrast is between two groups of authorities. The positive side is broad, but it is dominated by cases that courts use to state general constitutional standards, rights frameworks, or limits on government action. For example, the 2006 positive block is led by \textit{City of Cleburne v. Cleburne Living Center}, with additional weight on \textit{Wolf v. Colorado}, \textit{Ohio v. Robinette}, and \textit{Florida v. Bostick}. The 2012 positive block shifts toward \textit{Imbler v. Pachtman}, \textit{Graham v. Connor}, \textit{Napue v. Illinois}, \textit{Estelle v. McGuire}, \textit{Nix v. Williams}, and \textit{United States v. Mendenhall}. The 2016 positive block again contains \textit{Cleburne} and \textit{Graham}, together with \textit{Mapp v. Ohio}, \textit{Roper v. Simmons}, \textit{United States v. Salerno}, and \textit{Washington v. Davis}. These cases span equal protection, Fourth Amendment seizure doctrine, excessive force, exclusionary rule incorporation, habeas review, prosecutorial misconduct, punishment, and preventive detention. Thus, the positive side should not be described as one narrow doctrine. What the table supports is a broader characterization: this side collects authorities that articulate constitutional standards and remedial limits across several areas of public-law and criminal-adjudication doctrine.

The negative side is more concentrated in criminal procedure and post-conviction review. In 2006, the largest negative attributions are \textit{Stovall v. Denno}, \textit{Schlup v. Delo}, \textit{United States v. Cronic}, \textit{Preiser v. Rodriguez}, and \textit{People v. Partida}. These cases concern eyewitness identification, habeas gateways, right to counsel, the habeas/Section 1983 boundary, and evidentiary error preservation. In 2012, the negative side contains \textit{Caldwell v. Mississippi}, \textit{Schneckloth v. Bustamonte}, \textit{Arizona v. Fulminante}, \textit{Lockyer v. Andrade}, \textit{Village of Arlington Heights}, and \textit{Carroll v. United States}. In 2016, it contains \textit{Gideon v. Wainwright}, \textit{Ker v. California}, \textit{Daniels v. Williams}, \textit{Michigan v. Long}, \textit{Adams v. Williams}, and \textit{United States v. Cortez}. This side is therefore more visibly tied to criminal adjudication: right to counsel, search and seizure, consent, reasonable suspicion, coerced confessions, habeas review, and post-conviction standards.

The change point is understandable because the same broad opposition appears at all three peaks. Mode 1 does not simply pick up one famous case or one doctrinal category. Instead, it repeatedly places a broad constitutional standards group on the positive side and a more procedure specific criminal adjudication group on the negative side. The 2006 peak is driven especially by \textit{Cleburne} on the positive side and \textit{Stovall}/\textit{Schlup}/\textit{Cronic} on the negative side. The 2012 peak adds a different set of positive authorities involving prosecutorial immunity, excessive force, false testimony, habeas limits, and seizure doctrine, while the negative side emphasizes capital sentencing, consent searches, coerced confessions, and habeas proportionality review. The 2016 peak returns to the same larger contrast, with \textit{Cleburne}, \textit{Mapp}, \textit{Graham}, and \textit{Roper} on the positive side and \textit{Gideon}, \textit{Ker}, \textit{Daniels}, \textit{Long}, and \textit{Cortez} on the negative side.

We therefore interpret Mode 1 as a broad reweighting between two citation neighborhoods: a positive neighborhood of general constitutional standards and authorities concerning rights and remedies, and a negative neighborhood more anchored in criminal procedure and post conviction doctrine. The timing of these changes is also consistent with the early and middle years of the Roberts Court, during which the Court repeatedly addressed the remedial and procedural consequences of constitutional criminal adjudication claims, including habeas review, exclusionary rule doctrine, ineffective assistance, sentencing, and collateral review. This context is relevant because many of these decisions did not simply ask whether a constitutional right existed in the abstract. They also asked whether an alleged violation justified suppression, reversal, collateral relief, or reopening of a criminal judgment. In that sense, the period placed greater emphasis on finality, deference to prior adjudication, harmless error, procedural default, and the limits of available remedies. We do not treat the mode as identifying the Roberts Court as a single causal source of the shift. Rather, the Roberts Court context helps explain why the criminal procedure and post-conviction side of this citation contrast may have become more salient during this period.

\subsubsection{Mode 3}

\begingroup
\setlength{\tabcolsep}{3pt}
\renewcommand{\arraystretch}{1.08}
\begin{footnotesize}

\begin{longtable}{c c l p{2.0cm} r p{5.8cm}}
\caption{Mode 3 attribution table. For each change point, TOP rows report the most positively attributed reporter nodes and BOTTOM rows report the most negatively attributed reporter nodes.}
\label{tab:mode2_attributions}\\
\toprule
CP rank & Year & Side & Reporter cite & Score & Case \\
\midrule
\endfirsthead

\toprule
CP rank & Year & Side & Reporter cite & Score & Case \\
\midrule
\endhead

\midrule
\multicolumn{6}{r}{\emph{Continued on next page}}\\
\midrule
\endfoot

\bottomrule
\endlastfoot

1 & 2001 & TOP & 102 S.Ct. 2727 & 0.06780 & \textit{Harlow v. Fitzgerald}, 457 U.S. 800 (1982) \\
1 & 2001 & TOP & 73 L.Ed.2d 396 & 0.06778 & \textit{Harlow v. Fitzgerald}, 457 U.S. 800 (1982) \\
1 & 2001 & TOP & 448 U.S. 98 & 0.05388 & \textit{Rawlings v. Kentucky}, 448 U.S. 98 (1980) \\
1 & 2001 & TOP & 86 L.Ed.2d 411 & 0.05083 & \textit{Mitchell v. Forsyth}, 472 U.S. 511 (1985) \\
1 & 2001 & TOP & 105 S.Ct. 2806 & 0.05064 & \textit{Mitchell v. Forsyth}, 472 U.S. 511 (1985) \\
1 & 2001 & TOP & 428 U.S. 242 & 0.04152 & \textit{Proffitt v. Florida}, 428 U.S. 242 (1976) \\
1 & 2001 & TOP & 49 L.Ed.2d 913 & 0.04139 & \textit{Proffitt v. Florida}, 428 U.S. 242 (1976) \\
1 & 2001 & TOP & 96 S.Ct. 2960 & 0.04055 & \textit{Proffitt v. Florida}, 428 U.S. 242 (1976) \\
1 & 2001 & TOP & 59 L.Ed.2d 660 & 0.03933 & \textit{Delaware v. Prouse}, 440 U.S. 648 (1979) \\
1 & 2001 & TOP & 150 L.Ed.2d 272 & 0.03826 & \textit{Saucier v. Katz}, 533 U.S. 194 (2001) \\
\addlinespace
1 & 2001 & BOTTOM & 462 U.S. 696 & -0.03869 & \textit{United States v. Place}, 462 U.S. 696 (1983) \\
1 & 2001 & BOTTOM & 88 S.Ct. 1556 & -0.03912 & \textit{Carafas v. LaVallee}, 391 U.S. 234 (1968) \\
1 & 2001 & BOTTOM & 13 L.Ed.2d 684 & -0.03917 & \textit{United States v. Ventresca}, 380 U.S. 102 (1965) \\
1 & 2001 & BOTTOM & 20 L.Ed.2d 554 & -0.03977 & \textit{Carafas v. LaVallee}, 391 U.S. 234 (1968) \\
1 & 2001 & BOTTOM & 85 S.Ct. 741 & -0.03989 & \textit{United States v. Ventresca}, 380 U.S. 102 (1965) \\
1 & 2001 & BOTTOM & 92 L.Ed. 436 & -0.04120 & \textit{Johnson v. United States}, 333 U.S. 10 (1948) \\
1 & 2001 & BOTTOM & 68 S.Ct. 367 & -0.04148 & \textit{Johnson v. United States}, 333 U.S. 10 (1948) \\
1 & 2001 & BOTTOM & 442 U.S. 200 & -0.04529 & \textit{Dunaway v. New York}, 442 U.S. 200 (1979) \\
1 & 2001 & BOTTOM & 102 S.Ct. 2157 & -0.05316 & \textit{United States v. Ross}, 456 U.S. 798 (1982) \\
1 & 2001 & BOTTOM & 72 L.Ed.2d 572 & -0.05335 & \textit{United States v. Ross}, 456 U.S. 798 (1982) \\

\midrule

2 & 1987 & TOP & 100 S.Ct. 1870 & 0.05454 & \textit{United States v. Mendenhall}, 446 U.S. 544 (1980) \\
2 & 1987 & TOP & 64 L.Ed.2d 497 & 0.05372 & \textit{United States v. Mendenhall}, 446 U.S. 544 (1980) \\
2 & 1987 & TOP & 446 U.S. 544 & 0.05360 & \textit{United States v. Mendenhall}, 446 U.S. 544 (1980) \\
2 & 1987 & TOP & 415 U.S. 164 & 0.04358 & \textit{United States v. Matlock}, 415 U.S. 164 (1974) \\
2 & 1987 & TOP & 63 L.Ed.2d 639 & 0.04294 & \textit{Payton v. New York}, 445 U.S. 573 (1980) \\
2 & 1987 & TOP & 100 S.Ct. 1371 & 0.04294 & \textit{Payton v. New York}, 445 U.S. 573 (1980) \\
2 & 1987 & TOP & 461 U.S. 352 & 0.04262 & \textit{Kolender v. Lawson}, 461 U.S. 352 (1983) \\
2 & 1987 & TOP & 103 S.Ct. 1855 & 0.04186 & \textit{Kolender v. Lawson}, 461 U.S. 352 (1983) \\
2 & 1987 & TOP & 445 U.S. 573 & 0.04041 & \textit{Payton v. New York}, 445 U.S. 573 (1980) \\
2 & 1987 & TOP & 29 L.Ed.2d 564 & 0.04002 & \textit{Coolidge v. New Hampshire}, 403 U.S. 443 (1971) \\
\addlinespace
2 & 1987 & BOTTOM & 94 S.Ct. 613 & -0.04170 & \textit{United States v. Calandra}, 414 U.S. 338 (1974) \\
2 & 1987 & BOTTOM & 414 U.S. 338 & -0.04187 & \textit{United States v. Calandra}, 414 U.S. 338 (1974) \\
2 & 1987 & BOTTOM & 38 L.Ed.2d 561 & -0.04248 & \textit{United States v. Calandra}, 414 U.S. 338 (1974) \\
2 & 1987 & BOTTOM & 10 L.Ed.2d 726 & -0.04630 & \textit{Ker v. California}, 374 U.S. 23 (1963) \\
2 & 1987 & BOTTOM & 6 S.Ct. 524 & -0.04670 & \textit{Boyd v. United States}, 116 U.S. 616 (1886) \\
2 & 1987 & BOTTOM & 374 U.S. 23 & -0.04972 & \textit{Ker v. California}, 374 U.S. 23 (1963) \\
2 & 1987 & BOTTOM & 83 S.Ct. 1623 & -0.04972 & \textit{Ker v. California}, 374 U.S. 23 (1963) \\
2 & 1987 & BOTTOM & 100 S.Ct. 2556 & -0.05570 & \textit{Rawlings v. Kentucky}, 448 U.S. 98 (1980) \\
2 & 1987 & BOTTOM & 65 L.Ed.2d 633 & -0.05574 & \textit{Rawlings v. Kentucky}, 448 U.S. 98 (1980) \\
2 & 1987 & BOTTOM & 448 U.S. 98 & -0.05624 & \textit{Rawlings v. Kentucky}, 448 U.S. 98 (1980) \\

\midrule

3 & 1992 & TOP & 102 S.Ct. 2157 & 0.05455 & \textit{United States v. Ross}, 456 U.S. 798 (1982) \\
3 & 1992 & TOP & 72 L.Ed.2d 572 & 0.05443 & \textit{United States v. Ross}, 456 U.S. 798 (1982) \\
3 & 1992 & TOP & 456 U.S. 798 & 0.05280 & \textit{United States v. Ross}, 456 U.S. 798 (1982) \\
3 & 1992 & TOP & 89 S.Ct. 1394 & 0.04452 & \textit{Davis v. Mississippi}, 394 U.S. 721 (1969) \\
3 & 1992 & TOP & 22 L.Ed.2d 676 & 0.04446 & \textit{Davis v. Mississippi}, 394 U.S. 721 (1969) \\
3 & 1992 & TOP & 394 U.S. 721 & 0.04378 & \textit{Davis v. Mississippi}, 394 U.S. 721 (1969) \\
3 & 1992 & TOP & 83 S.Ct. 1381 & 0.04149 & \textit{Lopez v. United States}, 373 U.S. 427 (1963) \\
3 & 1992 & TOP & 10 L.Ed.2d 462 & 0.04149 & \textit{Lopez v. United States}, 373 U.S. 427 (1963) \\
3 & 1992 & TOP & 373 U.S. 427 & 0.04149 & \textit{Lopez v. United States}, 373 U.S. 427 (1963) \\
3 & 1992 & TOP & 115 L.Ed.2d 389 & 0.03994 & \textit{Florida v. Bostick}, 501 U.S. 429 (1991) \\
\addlinespace
3 & 1992 & BOTTOM & 85 S.Ct. 741 & -0.04382 & \textit{United States v. Ventresca}, 380 U.S. 102 (1965) \\
3 & 1992 & BOTTOM & 362 U.S. 257 & -0.04417 & \textit{Jones v. United States}, 362 U.S. 257 (1960) \\
3 & 1992 & BOTTOM & 13 L.Ed.2d 684 & -0.04477 & \textit{United States v. Ventresca}, 380 U.S. 102 (1965) \\
3 & 1992 & BOTTOM & 80 S.Ct. 725 & -0.04590 & \textit{Jones v. United States}, 362 U.S. 257 (1960) \\
3 & 1992 & BOTTOM & 4 L.Ed.2d 697 & -0.04725 & \textit{Jones v. United States}, 362 U.S. 257 (1960) \\
3 & 1992 & BOTTOM & 94 S.Ct. 467 & -0.05239 & \textit{United States v. Robinson}, 414 U.S. 218 (1973) \\
3 & 1992 & BOTTOM & 38 L.Ed.2d 427 & -0.05239 & \textit{United States v. Robinson}, 414 U.S. 218 (1973) \\
3 & 1992 & BOTTOM & 414 U.S. 218 & -0.05266 & \textit{United States v. Robinson}, 414 U.S. 218 (1973) \\
3 & 1992 & BOTTOM & 66 L.Ed.2d 621 & -0.05578 & \textit{United States v. Cortez}, 449 U.S. 411 (1981) \\
3 & 1992 & BOTTOM & 101 S.Ct. 690 & -0.05688 & \textit{United States v. Cortez}, 449 U.S. 411 (1981) \\

\end{longtable}
\end{footnotesize}
\endgroup

Mode~3 is a focused criminal procedure mode. Its three leading level change points occur in 1987, 1992, and 2001, and the high attribution nodes are concentrated in Fourth Amendment search and seizure doctrine and related police litigation doctrine. Unlike Mode~1, this mode does not span broad constitutional fields. The attribution table instead shows a more local structure inside police conduct law.

The positive side of Mode~3 generally contains cases about police citizen encounters, third party consent, home entry, automobile or field searches, bus sweep encounters, and qualified immunity. The negative side generally contains cases about Fourth Amendment standing, warrant review, custodial or investigative detention, exclusionary rule limits, search incident to arrest, and older search and seizure doctrine. The mode therefore separates different groups of cases used to define the legal boundary of police conduct.

The 1987 change point gives the first version of this separation. The positive side is led by \textit{United States v. Mendenhall}, \textit{United States v. Matlock}, \textit{Payton v. New York}, \textit{Kolender v. Lawson}, and \textit{Coolidge v. New Hampshire}. These cases concern police citizen encounters, third party consent, home entry for arrest, stop and identify doctrine, and search and seizure limits. The negative side is led by \textit{Rawlings v. Kentucky}, \textit{Ker v. California}, \textit{Boyd v. United States}, and \textit{United States v. Calandra}. These cases emphasize Fourth Amendment standing, state search and seizure standards, older search and self incrimination doctrine, and limits on the exclusionary rule. Thus, the 1987 peak contrasts encounter, consent, and home entry cases with standing, older search doctrine, and exclusionary rule cases.

The 1992 change point remains within Fourth Amendment doctrine but shifts to a different set of police conduct cases. The positive side is dominated by \textit{United States v. Ross}, \textit{Davis v. Mississippi}, \textit{Lopez v. United States}, and \textit{Florida v. Bostick}. These cases involve automobile searches, investigative detention for fingerprinting, participant monitoring, and bus sweep seizure analysis. The negative side is dominated by \textit{United States v. Cortez}, \textit{United States v. Robinson}, \textit{Jones v. United States}, and \textit{United States v. Ventresca}, which concern reasonable suspicion, search incident to arrest, Fourth Amendment standing, and warrant affidavit review. The 1992 peak therefore contrasts automobile search and mobile encounter cases with reasonable suspicion, search incident, standing, and warrant review cases.

The 2001 change point adds an explicit police litigation component. The positive side is led by \textit{Harlow v. Fitzgerald}, \textit{Mitchell v. Forsyth}, \textit{Rawlings v. Kentucky}, \textit{Delaware v. Prouse}, and \textit{Saucier v. Katz}, with \textit{Proffitt v. Florida} appearing as a smaller capital sentencing component. This side combines qualified immunity doctrine with Fourth Amendment standing and automobile stop doctrine. The negative side is led by \textit{United States v. Ross}, \textit{Dunaway v. New York}, \textit{Johnson v. United States}, \textit{United States v. Ventresca}, \textit{Carafas v. LaVallee}, and \textit{United States v. Place}. These cases concern automobile searches, custodial detention, warrant requirements, warrant affidavit review, habeas mootness, and investigative detention of luggage. The 2001 peak therefore differs from the 1987 and 1992 peaks because qualified immunity cases become central to the positive side of the mode.

The timing of these peaks is legally plausible. The 1987 peak occurs in the same year as \textit{Anderson v. Creighton}, where the Court applied qualified immunity to a Fourth Amendment warrantless search claim and held that the allegedly violated right must be clearly established at a factually specific level~\citep{anderson1987creighton}. The 1992 peak is consistent with the reorganization of police encounter and automobile search doctrine around this period. Contemporary commentary described the Court as reshaping the boundary between regulated seizures and unregulated police citizen encounters~\citep{maclin1989locomotion}, and the Court had recently clarified automobile search and bus encounter doctrine in \textit{California v. Acevedo} and \textit{Florida v. Bostick}~\citep{acevedo1991,bostick1991}. The 2001 peak aligns with \textit{Saucier v. Katz}, where the Court required courts to ask first whether the facts show a constitutional violation and then whether the right was clearly established~\citep{saucier2001katz}. Later scholarship describes this Wilson--Saucier sequencing rule as shaping how lower courts articulated constitutional rights in qualified immunity cases~\citep{hughes2009sequencing}. More broadly, Kinports characterizes the Supreme Court's qualified immunity jurisprudence as a ``quiet expansion,'' which supports reading the 1987 and 2001 peaks as part of an incremental doctrinal development rather than as a response to a single triggering event~\citep{kinports2016quiet}.

This provides a network interpretation of the mode. Mode~3 is not simply detecting Fourth Amendment doctrine as a single block. Instead, it shows that different Fourth Amendment case groups are reorganized with different neighboring cases over time. In 1987, encounter, consent, and entry cases separate from standing, older search doctrine, and exclusionary rule cases. In 1992, automobile search and mobile encounter cases separate from reasonable suspicion, search incident, standing, and warrant review cases. By 2001, qualified immunity cases such as \textit{Harlow}, \textit{Mitchell}, and \textit{Saucier} appear on the same side as Fourth Amendment police conduct cases such as \textit{Rawlings} and \textit{Prouse}. The resulting network change is that Fourth Amendment cases are no longer organized only around merits questions about searches and seizures. They also become linked to the civil litigation question whether the relevant right was clearly established enough to support damages liability against officers.

Overall, Mode~3 is best interpreted as a Fourth Amendment and police litigation reweighting. The positive side generally emphasizes encounter, consent, home entry, automobile search, field search, and qualified immunity precedents. The negative side generally emphasizes standing, warrant review, detention, exclusionary rule limits, search incident to arrest, and older search and seizure precedents. The main finding is that Mode~3 separates citation neighborhoods inside criminal procedure and shows how the network organization of police conduct doctrine changes. Encounter and consent cases cluster apart from standing and warrant cases. Automobile and mobile encounter cases cluster apart from reasonable suspicion and search incident cases. By 2001, Fourth Amendment police conduct cases are linked to qualified immunity cases governing the clearly established law requirement.

\subsubsection{Mode 5}

\begingroup
\setlength{\tabcolsep}{3pt}
\renewcommand{\arraystretch}{1.08}
\begin{footnotesize}

\begin{longtable}{c c l p{2.0cm} r p{5.8cm}}
\caption{Mode 5 attribution table. For each change point, TOP rows report the most positively attributed reporter nodes and BOTTOM rows report the most negatively attributed reporter nodes.}
\label{tab:mode4_attributions}\\
\toprule
CP rank & Year & Side & Reporter cite & Score & Case \\
\midrule
\endfirsthead

\toprule
CP rank & Year & Side & Reporter cite & Score & Case \\
\midrule
\endhead

\midrule
\multicolumn{6}{r}{\emph{Continued on next page}}\\
\midrule
\endfoot

\bottomrule
\endlastfoot

1 & 2018 & TOP & 112 S.Ct. 475 & 0.04112 & \textit{Estelle v. McGuire}, 502 U.S. 62 (1991) \\
1 & 2018 & TOP & 103 L.Ed.2d 249 & 0.03544 & \textit{DeShaney v. Winnebago County Department of Social Services}, 489 U.S. 189 (1989) \\
1 & 2018 & TOP & 109 S.Ct. 998 & 0.03497 & \textit{DeShaney v. Winnebago County Department of Social Services}, 489 U.S. 189 (1989) \\
1 & 2018 & TOP & 116 L.Ed.2d 385 & 0.03422 & \textit{Estelle v. McGuire}, 502 U.S. 62 (1991) \\
1 & 2018 & TOP & 140 L.Ed.2d 1043 & 0.03008 & \textit{County of Sacramento v. Lewis}, 523 U.S. 833 (1998) \\
1 & 2018 & TOP & 3 L.Ed.2d 1217 & 0.02998 & \textit{Napue v. Illinois}, 360 U.S. 264 (1959) \\
1 & 2018 & TOP & 121 S.Ct. 2151 & 0.02985 & \textit{Saucier v. Katz}, 533 U.S. 194 (2001) \\
1 & 2018 & TOP & 150 L.Ed.2d 272 & 0.02980 & \textit{Saucier v. Katz}, 533 U.S. 194 (2001) \\
1 & 2018 & TOP & 107 S.Ct. 2095 & 0.02977 & \textit{United States v. Salerno}, 481 U.S. 739 (1987) \\
1 & 2018 & TOP & 95 L.Ed.2d 697 & 0.02895 & \textit{United States v. Salerno}, 481 U.S. 739 (1987) \\
\addlinespace
1 & 2018 & BOTTOM & 8 L.Ed.2d 758 & -0.02110 & \textit{Robinson v. California}, 370 U.S. 660 (1962) \\
1 & 2018 & BOTTOM & 99 S.Ct. 970 & -0.02121 & \textit{Montana v. United States}, 440 U.S. 147 (1979) \\
1 & 2018 & BOTTOM & 96 L.Ed.2d 440 & -0.02386 & \textit{Booth v. Maryland}, 482 U.S. 496 (1987) \\
1 & 2018 & BOTTOM & 71 L.Ed.2d 1 & -0.02493 & \textit{Eddings v. Oklahoma}, 455 U.S. 104 (1982) \\
1 & 2018 & BOTTOM & 415 U.S. 452 & -0.02547 & \textit{Steffel v. Thompson}, 415 U.S. 452 (1974) \\
1 & 2018 & BOTTOM & 99 S.Ct. 645 & -0.02660 & \textit{Parklane Hosiery Co. v. Shore}, 439 U.S. 322 (1979) \\
1 & 2018 & BOTTOM & 58 L.Ed.2d 552 & -0.02660 & \textit{Parklane Hosiery Co. v. Shore}, 439 U.S. 322 (1979) \\
1 & 2018 & BOTTOM & 102 S.Ct. 869 & -0.02919 & \textit{Eddings v. Oklahoma}, 455 U.S. 104 (1982) \\
1 & 2018 & BOTTOM & 56 L.Ed.2d 611 & -0.03155 & \textit{Monell v. Department of Social Services}, 436 U.S. 658 (1978) \\
1 & 2018 & BOTTOM & 98 S.Ct. 2018 & -0.03156 & \textit{Monell v. Department of Social Services}, 436 U.S. 658 (1978) \\

\midrule

2 & 2016 & TOP & 161 L.Ed.2d 1 & 0.04108 & \textit{Roper v. Simmons}, 543 U.S. 551 (2005) \\
2 & 2016 & TOP & 96 S.Ct. 2978 & 0.03862 & \textit{Woodson v. North Carolina}, 428 U.S. 280 (1976) \\
2 & 2016 & TOP & 71 L.Ed.2d 1 & 0.03833 & \textit{Eddings v. Oklahoma}, 455 U.S. 104 (1982) \\
2 & 2016 & TOP & 428 U.S. 153 & 0.03504 & \textit{Gregg v. Georgia}, 428 U.S. 153 (1976) \\
2 & 2016 & TOP & 49 L.Ed.2d 944 & 0.03498 & \textit{Woodson v. North Carolina}, 428 U.S. 280 (1976) \\
2 & 2016 & TOP & 125 S.Ct. 1183 & 0.03289 & \textit{Roper v. Simmons}, 543 U.S. 551 (2005) \\
2 & 2016 & TOP & 485 U.S. 112 & 0.03227 & \textit{City of St. Louis v. Praprotnik}, 485 U.S. 112 (1988) \\
2 & 2016 & TOP & 478 U.S. 570 & 0.03139 & \textit{Rose v. Clark}, 478 U.S. 570 (1986) \\
2 & 2016 & TOP & 96 S.Ct. 2909 & 0.03015 & \textit{Gregg v. Georgia}, 428 U.S. 153 (1976) \\
2 & 2016 & TOP & 49 L.Ed.2d 859 & 0.03000 & \textit{Gregg v. Georgia}, 428 U.S. 153 (1976) \\
\addlinespace
2 & 2016 & BOTTOM & 971 S.W.2d 402 & -0.03203 & \textit{Maritime Overseas Corp. v. Ellis}, 971 S.W.2d 402 (Tex. 1998) \\
2 & 2016 & BOTTOM & 879 S.W.2d 10 & -0.03219 & \textit{Transportation Insurance Co. v. Moriel}, 879 S.W.2d 10 (Tex. 1994) \\
2 & 2016 & BOTTOM & 92 S.Ct. 2294 & -0.03223 & \textit{Grayned v. City of Rockford}, 408 U.S. 104 (1972) \\
2 & 2016 & BOTTOM & 50 L.Ed.2d 471 & -0.03231 & \textit{Mt. Healthy City School District Board of Education v. Doyle}, 429 U.S. 274 (1977) \\
2 & 2016 & BOTTOM & 17 L.Ed.2d 374 & -0.03241 & \textit{Hoffa v. United States}, 385 U.S. 293 (1966) \\
2 & 2016 & BOTTOM & 87 S.Ct. 408 & -0.03241 & \textit{Hoffa v. United States}, 385 U.S. 293 (1966) \\
2 & 2016 & BOTTOM & 244 S.W.2d 660 & -0.03367 & \textit{In re King's Estate}, 244 S.W.2d 660 (Tex. 1951) \\
2 & 2016 & BOTTOM & 650 S.W.2d 61 & -0.03376 & \textit{Kindred v. Con/Chem, Inc.}, 650 S.W.2d 61 (Tex. 1983) \\
2 & 2016 & BOTTOM & 715 S.W.2d 629 & -0.03630 & \textit{Pool v. Ford Motor Co.}, 715 S.W.2d 629 (Tex. 1986) \\
2 & 2016 & BOTTOM & 922 S.W.2d 126 & -0.03719 & \textit{Clewis v. State}, 922 S.W.2d 126 (Tex. Crim. App. 1996) \\

\midrule

3 & 1973 & TOP & 93 S.Ct. 1278 & 0.04852 & \textit{San Antonio Independent School District v. Rodriguez}, 411 U.S. 1 (1973) \\
3 & 1973 & TOP & 36 L.Ed.2d 16 & 0.04798 & \textit{San Antonio Independent School District v. Rodriguez}, 411 U.S. 1 (1973) \\
3 & 1973 & TOP & 411 U.S. 1 & 0.04597 & \textit{San Antonio Independent School District v. Rodriguez}, 411 U.S. 1 (1973) \\
3 & 1973 & TOP & 93 S.Ct. 1764 & 0.04308 & \textit{Frontiero v. Richardson}, 411 U.S. 677 (1973) \\
3 & 1973 & TOP & 411 U.S. 677 & 0.03881 & \textit{Frontiero v. Richardson}, 411 U.S. 677 (1973) \\
3 & 1973 & TOP & 36 L.Ed.2d 583 & 0.03850 & \textit{Frontiero v. Richardson}, 411 U.S. 677 (1973) \\
3 & 1973 & TOP & 411 U.S. 223 & 0.03843 & \textit{Brown v. United States}, 411 U.S. 223 (1973) \\
3 & 1973 & TOP & 93 S.Ct. 1565 & 0.03843 & \textit{Brown v. United States}, 411 U.S. 223 (1973) \\
3 & 1973 & TOP & 93 S.Ct. 2041 & 0.03679 & \textit{Schneckloth v. Bustamonte}, 412 U.S. 218 (1973) \\
3 & 1973 & TOP & 36 L.Ed.2d 854 & 0.03675 & \textit{Schneckloth v. Bustamonte}, 412 U.S. 218 (1973) \\
\addlinespace
3 & 1973 & BOTTOM & 4 L.Ed.2d 1688 & -0.02338 & \textit{Rios v. United States}, 364 U.S. 253 (1960) \\
3 & 1973 & BOTTOM & 355 U.S. 339 & -0.02387 & \textit{Lawn v. United States}, 355 U.S. 339 (1958) \\
3 & 1973 & BOTTOM & 84 S.Ct. 1774 & -0.02430 & \textit{Jackson v. Denno}, 378 U.S. 368 (1964) \\
3 & 1973 & BOTTOM & 308 U.S. 371 & -0.02501 & \textit{Chicot County Drainage District v. Baxter State Bank}, 308 U.S. 371 (1940) \\
3 & 1973 & BOTTOM & 60 S.Ct. 317 & -0.02501 & \textit{Chicot County Drainage District v. Baxter State Bank}, 308 U.S. 371 (1940) \\
3 & 1973 & BOTTOM & 2 L.Ed.2d 321 & -0.02547 & \textit{Lawn v. United States}, 355 U.S. 339 (1958) (1958) \\
3 & 1973 & BOTTOM & 84 L.Ed. 329 & -0.02666 & \textit{Chicot County Drainage District v. Baxter State Bank}, 308 U.S. 371 (1940) \\
3 & 1973 & BOTTOM & 67 S.Ct. 1672 & -0.02807 & \textit{Adamson v. California}, 332 U.S. 46 (1947) \\
3 & 1973 & BOTTOM & 91 L.Ed. 1903 & -0.02808 & \textit{Adamson v. California}, 332 U.S. 46 (1947) \\
3 & 1973 & BOTTOM & 332 U.S. 46 & -0.02981 & \textit{Adamson v. California}, 332 U.S. 46 (1947) \\

\end{longtable}
\end{footnotesize}
\endgroup

Mode~5 is broader than Modes~1 and~3, but the attribution table still shows a coherent pattern. Its three leading level change points occur in 1973, 2016, and 2018. Unlike Mode~3, this mode is not centered on Fourth Amendment police conduct. It instead groups cases about federal constitutional claims against government action: when such claims are reviewable, who can be held liable, and what limits apply to punishment or remedies.

The positive side of Mode~5 generally contains cases about habeas limits, substantive due process, state action, preventive detention, qualified immunity, capital punishment, equal protection, and selected criminal procedure rules. The negative side generally contains cases about municipal liability, preclusion, declaratory relief, punishment limits, incorporation, evidentiary sufficiency, and older constitutional criminal procedure. The mode therefore tracks how constitutional claims against government actors are routed through federal review, liability, remedy, and punishment doctrines.

The 1973 change point gives the earliest version of this structure. The positive side is led by \textit{San Antonio Independent School District v. Rodriguez}, \textit{Frontiero v. Richardson}, \textit{Brown v. United States}, and \textit{Schneckloth v. Bustamonte}. These cases combine equal protection, sex discrimination, Fourth Amendment standing, and consent searches. The negative side is led by \textit{Adamson v. California}, \textit{Chicot County Drainage District v. Baxter State Bank}, \textit{Jackson v. Denno}, \textit{Lawn v. United States}, and \textit{Rios v. United States}. These cases involve incorporation, finality of judgments, Eighth Amendment punishment, confession voluntariness, criminal procedure, and search and seizure timing. The 1973 peak therefore contrasts then current equal protection and criminal procedure cases with older constitutional criminal and federal courts cases.

The 2016 change point makes the punishment component of Mode~5 much more visible. The positive side is dominated by \textit{Roper v. Simmons}, \textit{Woodson v. North Carolina}, \textit{Eddings v. Oklahoma}, and \textit{Gregg v. Georgia}. These are central Eighth Amendment and capital sentencing cases. The same side also includes \textit{City of St. Louis v. Praprotnik}, a municipal liability case, and \textit{Rose v. Clark}, a harmless error case. Thus, the positive side links capital punishment doctrine with government liability and criminal review. The negative side is more heterogeneous. It contains several Texas state law sufficiency and evidence review cases, together with \textit{Hoffa v. United States}, \textit{Mt. Healthy City School District Board of Education v. Doyle}, and \textit{Grayned v. City of Rockford}. These cases involve factual sufficiency, informant evidence, retaliation, and First Amendment limits. The 2016 peak should therefore be read as a strong capital punishment and government liability signal on the positive side, against a more mixed procedural and evidentiary sufficiency neighborhood on the negative side.

The 2018 change point is the clearest recent expression of the mode. The positive side is led by \textit{Estelle v. McGuire}, \textit{DeShaney v. Winnebago County Department of Social Services}, \textit{County of Sacramento v. Lewis}, \textit{Napue v. Illinois}, \textit{Saucier v. Katz}, and \textit{United States v. Salerno}. These cases concern limits on federal habeas review, the absence of a general affirmative constitutional duty to protect, substantive due process, false testimony, qualified immunity, and preventive detention. This side is best read as a constitutional claim threshold neighborhood: it concerns when government conduct becomes constitutionally actionable and when federal courts will review or remedy it. The negative side is led by \textit{Monell v. Department of Social Services}, \textit{Eddings v. Oklahoma}, \textit{Parklane Hosiery Co. v. Shore}, \textit{Steffel v. Thompson}, \textit{Booth v. Maryland}, \textit{Montana v. United States}, and \textit{Robinson v. California}. These cases concern municipal liability, capital sentencing, preclusion, declaratory relief, and punishment limits. The 2018 peak therefore separates one group of cases about federal review and constitutional claim thresholds from another group about liability, preclusion, remedies, and punishment doctrine.

This gives Mode~5 a network interpretation. The mode does not detect a single doctrinal field in the way Mode~3 detects Fourth Amendment and police litigation structure. Instead, it shows a changing citation neighborhood around federal constitutional claims against government actors. In 1973, newer equal protection and criminal procedure cases separate from older incorporation, finality, punishment, and confession cases. In 2016, capital punishment cases cluster with municipal liability and criminal review cases, while state law sufficiency and evidentiary review cases sit on the opposite side. By 2018, habeas, substantive due process, qualified immunity, and detention cases are grouped together, while municipal liability, preclusion, declaratory relief, and punishment cases form the opposing neighborhood.

Overall, Mode~5 is best interpreted as a mixed constitutional criminal and remedies mode. The positive side generally emphasizes federal review, constitutional claim thresholds, due process, detention, qualified immunity, capital punishment, and equal protection. The negative side generally emphasizes municipal liability, preclusion, federal court remedies, evidentiary sufficiency, incorporation, and punishment limits. The main finding is that Mode~5 captures a broader reorganization than Mode~3: it links constitutional claims against government action to the doctrines that determine reviewability, liability, remedies, and punishment.

\subsection{Mode-wise change point analysis of the Enron email network} 
\label{app:enron}

\begin{figure}[htbp] 
\centering \includegraphics[width=1.0\textwidth]{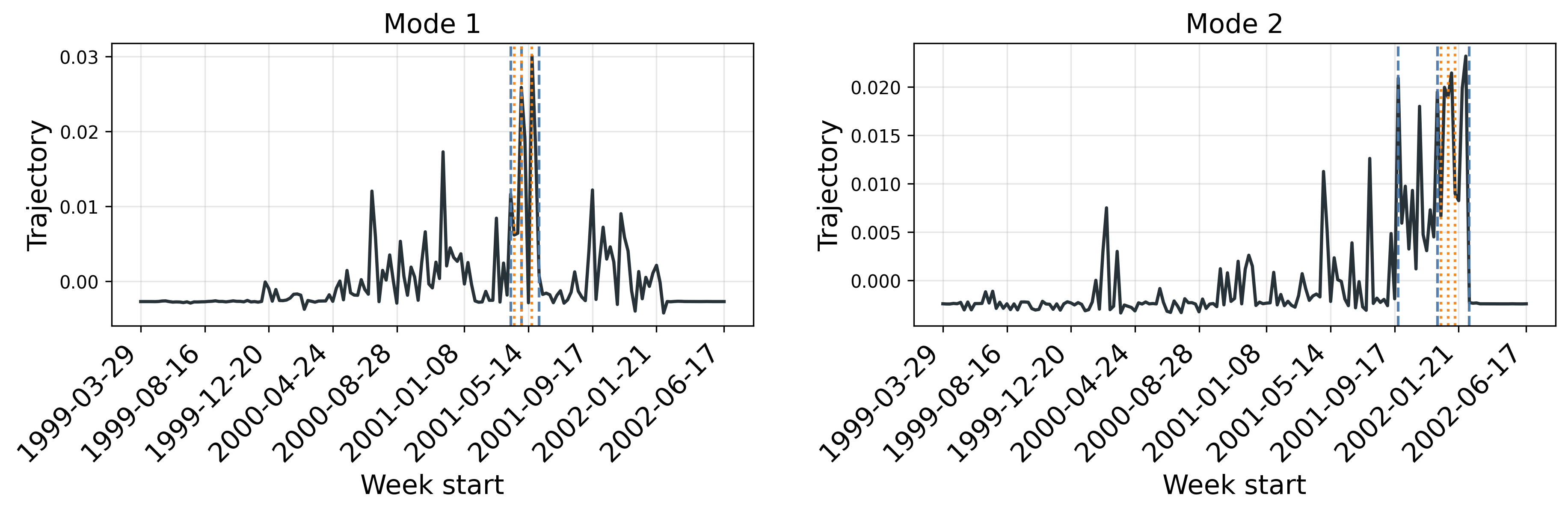} 
\caption{Trajectories of Enron.} 
\label{app:enron:trajectory} 
\end{figure} 

\begin{figure}[htbp] 
\centering 
\includegraphics[width=1.0\textwidth]{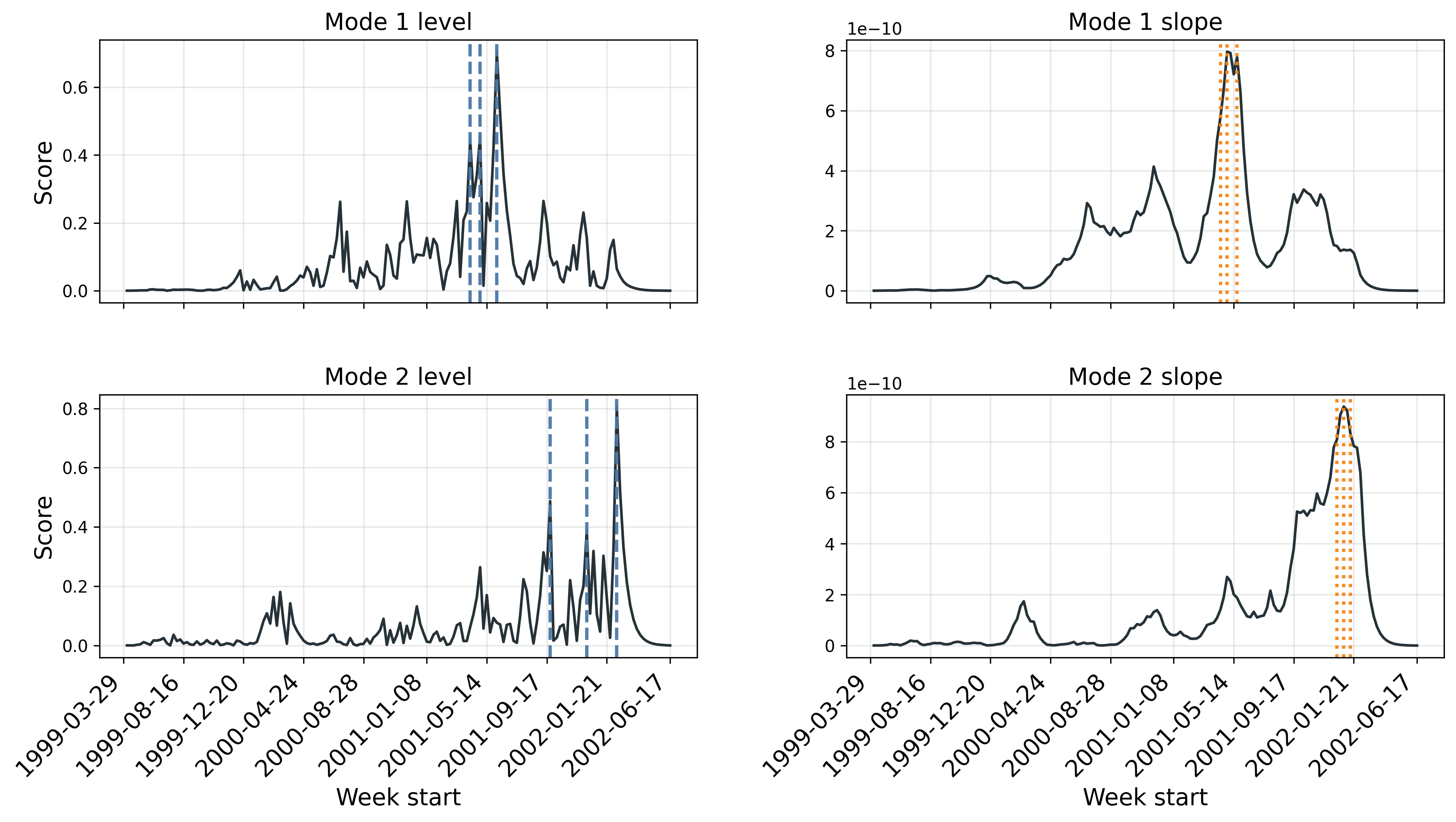} 
\caption{Mode-wise level (i.e., $0$th) and slope (i.e., $1$st) order change point scores for the first two Enron e-mail network modes.} 
\label{app:enron:score} 
\end{figure}

We also apply the same mode-wise change point analysis to the Enron email network. The network is constructed from the CMU version of the Enron email corpus~\citep{klimt2004enron}. Employee role labels are matched from the auxiliary employee metadata file \texttt{E-Mails-updated.csv} in Schulz's Enron-mails repository~\citep{schulz2012roles}. We use these role labels only as auxiliary metadata, and we cite external sources when making more specific claims about individuals' positions or functions. Because the dataset concerns real individuals, we report individual names only in the attribution tables needed to reproduce the analysis, and interpret the results as descriptive network patterns. Figure~\ref{app:enron:trajectory} shows the first two estimated mode-wise trajectories, and Figure~\ref{app:enron:score} reports the corresponding $0$th and $1$st order change point scores. The interpretation below focuses on $0$th order level changes, since Tables~\ref{app:enron:mode0_level} and~\ref{app:enron:mode1_level} report level change attributions.

Mode~1 captures an early shift in communication network structure during spring and early summer 2001. Its three largest level changes occur in the weeks of 2001-04-09, 2001-04-30, and 2001-06-04. These dates fall during Jeffrey Skilling's short tenure as Enron CEO and precede the firm's more visible public collapse period later in 2001~\citep{sec2004skillingcomplaint,time2001enronchronology}. We therefore interpret Mode~1 as a descriptive signal of early internal communication reorganization, rather than as a direct response to bankruptcy or later public investigations.

The April 30 peak provides the clearest example for Mode~1. On the positive side, the table includes employees with roles listed as Director, CEO, President, Vice President, Manager, Employee, and Managing Director, together with one employee whose role is not listed. On the negative side, the table includes employees with roles listed as President, Manager, Director, CEO, Employee, and Trader, together with several employees whose roles are not listed. Among the names with external support, Lavorato is described in contemporary reporting as running Enron's energy trading business~\citep{guardian2002lavoratoKitchen}, while Arnold is described in later reporting as a former Enron natural gas trader and Shankman as a prominent trader~\citep{wired2017arnold}. Kuykendall is listed as a Trader in the role metadata~\citep{schulz2012roles}. Thus, we interpret the April 30 peak conservatively as a separation between two sets of highly attributed employee nodes, with some evidence that the negative side contains several trading or commercially connected figures.

The April 9 and June 4 peaks support the same cautious interpretation while also showing that the signs of the learned mode should not be treated as fixed employee labels. Several employees appear on different sides across the three peaks. For example, Arnold appears on the positive side on April 9 but on the negative side on April 30 and June 4. Shankman appears on the positive side on April 9 and June 4 but on the negative side on April 30. Holst appears on the negative side on April 9 but on the positive side on April 30. We therefore interpret Mode~1 as a spring 2001 rotation in relative communication prominence among overlapping employee groups, rather than as a stable partition of employees. Its main value is that it identifies an internal communication shift that precedes the more visible crisis period events later in 2001.

Mode~2 captures a later crisis period direction. Its largest level change occurs in the week of 2002-02-11, with additional large changes in the weeks of 2001-09-24 and 2001-12-10. These dates align with Enron's public crisis timeline. Skilling resigned in August 2001, Enron's disclosures and investigations intensified in fall 2001, the firm filed for bankruptcy in December 2001, and congressional and regulatory scrutiny continued into early 2002~\citep{time2001enronchronology,sec2001enronrecentevents,sec2002enronbankruptcy}.

The September 24 peak provides the clearest prebankruptcy Mode~2 separation. On the negative side, the table includes employees with roles listed as Director, Vice President, Manager, and Employee. On the positive side, the table includes employees with roles listed as Vice President, Trader, President, and Employee, together with several employees whose roles are not listed. Among the names with external support, Grigsby is described in contemporary reporting as Enron's former Vice President for natural gas trading, and Arnold is described in later reporting as a former Enron natural gas trader~\citep{thestreet2002grigsby,wired2017arnold}. Thus, the September peak appears to separate two highly attributed employee groups, with the negative side containing several employees whose listed roles are senior or managerial and at least two externally supported trading or commercial figures. We interpret this separation descriptively as a communication network pattern, not as evidence of individual responsibility or misconduct.

The December 10 peak occurs immediately after Enron's bankruptcy filing. On the negative side, the table includes employees with roles listed as Manager, Director, Trader, Vice President, and Employee. On the positive side, the table includes employees with roles listed as Vice President, Employee, Manager, and Director, together with several employees whose roles are not listed. Steffes is the largest positive attribution at this peak and is listed as a Vice President in the role metadata~\citep{schulz2012roles}. We therefore interpret the December 10 peak as a crisis period shift between two highly attributed groups of employee nodes, rather than assigning a more specific organizational interpretation beyond what the table and cited metadata support.

The dominant Mode~2 event occurs on February 11, 2002. At this point, the positive side contains employees with roles listed as Employee, President, Vice President, CEO, and Manager. The bottom side has much smaller absolute scores and includes roles listed as Employee, President, Director, Trader, and Vice President, together with several employees whose roles are not listed. Because the positive side has substantially larger attribution magnitudes, the February peak is best interpreted as a post collapse concentration of communication change around the positively attributed employee group.

Taken together, the Enron results illustrate the usefulness of mode-wise attribution. The score curves identify when each learned direction changes, while the attribution tables identify which employee nodes define those changes. Mode~1 isolates an earlier spring 2001 communication reallocation during Skilling's CEO period, whereas Mode~2 isolates a later crisis period reorganization around Enron's public collapse, bankruptcy, and post collapse investigations. The supported claim is not that the method identifies the cause of Enron's collapse, that each mode corresponds to a single department, or that highly attributed employees bear responsibility for events at Enron. Rather, the method separates distinct temporal directions of network change and attaches each direction to employee nodes whose listed roles can be checked in the attribution tables and auxiliary metadata.

\begingroup
\setlength{\tabcolsep}{4pt}
\renewcommand{\arraystretch}{1.10}
\begin{footnotesize}

\begin{longtable}{c c l p{3.0cm} p{2.7cm} r}
\caption{Mode 1 level change-point attributions for the Enron email network. For each change point, TOP rows report the most positively attributed employee/mailbox nodes and BOTTOM rows report the most negatively attributed nodes. Employee/mailbox labels are the normalized node labels used in the Enron email network. Roles are matched from auxiliary employee metadata when available.}
\label{app:enron:mode0_level}\\
\toprule
CP rank & Week & Side & Employee/mailbox & Role & Score \\
\midrule
\endfirsthead

\toprule
CP rank & Week & Side & Employee/mailbox & Role & Score \\
\midrule
\endhead

\midrule
\multicolumn{6}{r}{\emph{Continued on next page}}\\
\midrule
\endfoot

\bottomrule
\endlastfoot

1 & 2001-06-04 & Top & zipper-a & Vice President & +0.02285 \\
1 & 2001-06-04 & Top & taylor-m & Employee & +0.01607 \\
1 & 2001-06-04 & Top & donoho-l & Employee & +0.01497 \\
1 & 2001-06-04 & Top & shapiro-r & Vice President & +0.00981 \\
1 & 2001-06-04 & Top & williams-j & Vice President & +0.00924 \\
1 & 2001-06-04 & Top & watson-k &  & +0.00857 \\
1 & 2001-06-04 & Top & campbell-l & Employee & +0.00774 \\
1 & 2001-06-04 & Top & shankman-j & President & +0.00716 \\
1 & 2001-06-04 & Top & lokay-m & Employee & +0.00706 \\
1 & 2001-06-04 & Top & hayslett-r & Vice President & +0.00582 \\
\addlinespace
1 & 2001-06-04 & Bottom & maggi-m & Director & -0.01849 \\
1 & 2001-06-04 & Bottom & tycholiz-b & Vice President & -0.01930 \\
1 & 2001-06-04 & Bottom & presto-k & Vice President & -0.01971 \\
1 & 2001-06-04 & Bottom & arnold-j & Manager & -0.02007 \\
1 & 2001-06-04 & Bottom & whalley-l &  & -0.02035 \\
1 & 2001-06-04 & Bottom & jones-t &  & -0.02110 \\
1 & 2001-06-04 & Bottom & quigley-d & Trader & -0.02132 \\
1 & 2001-06-04 & Bottom & lavorato-j & CEO & -0.02835 \\
1 & 2001-06-04 & Bottom & storey-g & Director & -0.03984 \\
1 & 2001-06-04 & Bottom & lenhart-m & Employee & -0.04147 \\

\midrule

2 & 2001-04-30 & Top & holst-k & Director & +0.04364 \\
2 & 2001-04-30 & Top & sanchez-m &  & +0.03514 \\
2 & 2001-04-30 & Top & skilling-j & CEO & +0.02747 \\
2 & 2001-04-30 & Top & whalley-g & President & +0.02298 \\
2 & 2001-04-30 & Top & corman-s & Vice President & +0.02290 \\
2 & 2001-04-30 & Top & buy-r & Manager & +0.02138 \\
2 & 2001-04-30 & Top & grigsby-m & Manager & +0.02118 \\
2 & 2001-04-30 & Top & tycholiz-b & Vice President & +0.02000 \\
2 & 2001-04-30 & Top & beck-s & Employee & +0.01778 \\
2 & 2001-04-30 & Top & haedicke-m & Managing Director & +0.01702 \\
\addlinespace
2 & 2001-04-30 & Bottom & scott-s &  & -0.02503 \\
2 & 2001-04-30 & Bottom & shankman-j & President & -0.02759 \\
2 & 2001-04-30 & Bottom & arnold-j & Manager & -0.03399 \\
2 & 2001-04-30 & Bottom & motley-m & Director & -0.03675 \\
2 & 2001-04-30 & Bottom & badeer-r & Director & -0.03991 \\
2 & 2001-04-30 & Bottom & lavorato-j & CEO & -0.04071 \\
2 & 2001-04-30 & Bottom & south-s &  & -0.04296 \\
2 & 2001-04-30 & Bottom & gay-r & Employee & -0.04679 \\
2 & 2001-04-30 & Bottom & lenhart-m & Employee & -0.05000 \\
2 & 2001-04-30 & Bottom & kuykendall-t & Trader & -0.05205 \\

\midrule

3 & 2001-04-09 & Top & taylor-m & Employee & +0.04103 \\
3 & 2001-04-09 & Top & shankman-j & President & +0.04080 \\
3 & 2001-04-09 & Top & mccarty-d & Vice President & +0.04004 \\
3 & 2001-04-09 & Top & shively-h & Vice President & +0.02783 \\
3 & 2001-04-09 & Top & shapiro-r & Vice President & +0.02733 \\
3 & 2001-04-09 & Top & haedicke-m & Managing Director & +0.02675 \\
3 & 2001-04-09 & Top & arnold-j & Manager & +0.02365 \\
3 & 2001-04-09 & Top & donoho-l & Employee & +0.02183 \\
3 & 2001-04-09 & Top & fossum-d & Vice President & +0.02183 \\
3 & 2001-04-09 & Top & lokay-m & Employee & +0.02135 \\
\addlinespace
3 & 2001-04-09 & Bottom & mclaughlin-e & Employee & -0.01292 \\
3 & 2001-04-09 & Bottom & tholt-j & Vice President & -0.01292 \\
3 & 2001-04-09 & Bottom & wolfe-j &  & -0.01437 \\
3 & 2001-04-09 & Bottom & dasovich-j & Employee & -0.01590 \\
3 & 2001-04-09 & Bottom & reitmeyer-j & Employee & -0.01590 \\
3 & 2001-04-09 & Bottom & allen-p & Manager & -0.01634 \\
3 & 2001-04-09 & Bottom & staab-t & Employee & -0.01984 \\
3 & 2001-04-09 & Bottom & kitchen-l & President & -0.02156 \\
3 & 2001-04-09 & Bottom & lavorato-j & CEO & -0.02243 \\
3 & 2001-04-09 & Bottom & holst-k & Director & -0.04147 \\

\end{longtable}
\end{footnotesize}
\endgroup

\begingroup
\setlength{\tabcolsep}{4pt}
\renewcommand{\arraystretch}{1.10}
\begin{footnotesize}

\begin{longtable}{c c l p{3.0cm} p{2.7cm} r}
\caption{Mode 2 level change-point attributions for the Enron email network. For each change point, TOP rows report the most positively attributed employee/mailbox nodes and BOTTOM rows report the most negatively attributed nodes. Employee/mailbox labels are the normalized node labels used in the Enron email network. Roles are matched from auxiliary employee metadata when available.}
\label{app:enron:mode1_level}\\
\toprule
CP rank & Week & Side & Employee/mailbox & Role & Score \\
\midrule
\endfirsthead

\toprule
CP rank & Week & Side & Employee/mailbox & Role & Score \\
\midrule
\endhead

\midrule
\multicolumn{6}{r}{\emph{Continued on next page}}\\
\midrule
\endfoot

\bottomrule
\endlastfoot

1 & 2002-02-11 & Top & sager-e & Employee & +0.06423 \\
1 & 2002-02-11 & Top & gilbertsmith-d & Employee & +0.06109 \\
1 & 2002-02-11 & Top & whalley-g & President & +0.06085 \\
1 & 2002-02-11 & Top & presto-k & Vice President & +0.05944 \\
1 & 2002-02-11 & Top & davis-d & Vice President & +0.05907 \\
1 & 2002-02-11 & Top & steffes-j & Vice President & +0.05835 \\
1 & 2002-02-11 & Top & zufferli-j & Vice President & +0.05771 \\
1 & 2002-02-11 & Top & lavorato-j & CEO & +0.05514 \\
1 & 2002-02-11 & Top & grigsby-m & Manager & +0.05480 \\
1 & 2002-02-11 & Top & tycholiz-b & Vice President & +0.05480 \\
\addlinespace
1 & 2002-02-11 & Bottom & dickson-s & Employee & +0.00000 \\
1 & 2002-02-11 & Bottom & williams-w3 &  & -0.00010 \\
1 & 2002-02-11 & Bottom & semperger-c & Employee & -0.00058 \\
1 & 2002-02-11 & Bottom & watson-k &  & -0.00131 \\
1 & 2002-02-11 & Bottom & horton-s & President & -0.00178 \\
1 & 2002-02-11 & Bottom & hyatt-k & Director & -0.00567 \\
1 & 2002-02-11 & Bottom & lokay-m & Employee & -0.00599 \\
1 & 2002-02-11 & Bottom & keiser-k & Employee & -0.01048 \\
1 & 2002-02-11 & Bottom & ybarbo-p & Trader & -0.02262 \\
1 & 2002-02-11 & Bottom & corman-s & Vice President & -0.02479 \\

\midrule

2 & 2001-09-24 & Top & harris-s & Vice President & +0.01623 \\
2 & 2001-09-24 & Top & hayslett-r & Vice President & +0.01269 \\
2 & 2001-09-24 & Top & ybarbo-p & Trader & +0.01096 \\
2 & 2001-09-24 & Top & lokay-m & Employee & +0.00905 \\
2 & 2001-09-24 & Top & kitchen-l & President & +0.00886 \\
2 & 2001-09-24 & Top & ring-a &  & +0.00771 \\
2 & 2001-09-24 & Top & mclaughlin-e & Employee & +0.00727 \\
2 & 2001-09-24 & Top & dasovich-j & Employee & +0.00451 \\
2 & 2001-09-24 & Top & fossum-d & Vice President & +0.00412 \\
2 & 2001-09-24 & Top & gang-l &  & +0.00321 \\
\addlinespace
2 & 2001-09-24 & Bottom & maggi-m & Director & -0.05631 \\
2 & 2001-09-24 & Bottom & zipper-a & Vice President & -0.05795 \\
2 & 2001-09-24 & Bottom & tycholiz-b & Vice President & -0.06100 \\
2 & 2001-09-24 & Bottom & arnold-j & Manager & -0.07963 \\
2 & 2001-09-24 & Bottom & davis-d & Vice President & -0.08071 \\
2 & 2001-09-24 & Bottom & gilbertsmith-d & Employee & -0.08082 \\
2 & 2001-09-24 & Bottom & neal-s & Vice President & -0.08083 \\
2 & 2001-09-24 & Bottom & presto-k & Vice President & -0.08090 \\
2 & 2001-09-24 & Bottom & grigsby-m & Manager & -0.09069 \\
2 & 2001-09-24 & Bottom & martin-t & Vice President & -0.09145 \\

\midrule

3 & 2001-12-10 & Top & steffes-j & Vice President & +0.05076 \\
3 & 2001-12-10 & Top & rapp-b &  & +0.01941 \\
3 & 2001-12-10 & Top & geaccone-t & Employee & +0.00931 \\
3 & 2001-12-10 & Top & schoolcraft-d &  & +0.00599 \\
3 & 2001-12-10 & Top & griffith-j & Manager & +0.00562 \\
3 & 2001-12-10 & Top & maggi-m & Director & +0.00528 \\
3 & 2001-12-10 & Top & mckay-j & Director & +0.00285 \\
3 & 2001-12-10 & Top & bailey-s &  & +0.00264 \\
3 & 2001-12-10 & Top & lucci-p & Employee & +0.00142 \\
3 & 2001-12-10 & Top & forney-j & Manager & +0.00142 \\
\addlinespace
3 & 2001-12-10 & Bottom & arnold-j & Manager & -0.04975 \\
3 & 2001-12-10 & Bottom & badeer-r & Director & -0.04984 \\
3 & 2001-12-10 & Bottom & motley-m & Director & -0.05034 \\
3 & 2001-12-10 & Bottom & swerzbin-m & Trader & -0.05034 \\
3 & 2001-12-10 & Bottom & davis-d & Vice President & -0.05034 \\
3 & 2001-12-10 & Bottom & buy-r & Manager & -0.05034 \\
3 & 2001-12-10 & Bottom & storey-g & Director & -0.05077 \\
3 & 2001-12-10 & Bottom & arora-h & Vice President & -0.05082 \\
3 & 2001-12-10 & Bottom & zufferli-j & Vice President & -0.05173 \\
3 & 2001-12-10 & Bottom & gilbertsmith-d & Employee & -0.05284 \\

\end{longtable}
\end{footnotesize}
\endgroup

\end{document}